\newif\iffull \fulltrue
\documentclass{article}
\usepackage{microtype}
\usepackage{graphicx}
\usepackage{subfigure}
\usepackage{booktabs} % for professional tables
\usepackage{nicefrac}

\usepackage{hyperref}

\usepackage[accepted]{icml2019}

\icmltitlerunning{Orthogonal Random Forest}

 \usepackage{booktabs}       % professional-quality tables
 \usepackage{amsfonts}       % blackboard math symbols
 \usepackage{nicefrac}       % compact symbols for 1/2, etc.
 \usepackage{microtype}      % microtypography
 \usepackage{graphicx}
 \usepackage{amsmath, amssymb, amsthm}
\usepackage{mathtools}
\usepackage{color}
\usepackage{xcolor}
\usepackage{multicol}
\usepackage{caption}
\usepackage{graphicx}
\usepackage{enumitem}
\usepackage{hyperref}
\usepackage[suppress]{color-edits}
\addauthor{vs}{red}
\addauthor{sw}{blue}
\addauthor{mo}{purple}

\definecolor{DarkGreen}{rgb}{0.1,0.5,0.1}
\definecolor{DarkRed}{rgb}{0.5,0.1,0.1}
\definecolor{DarkBlue}{rgb}{0.1,0.1,0.5}
\usepackage[]{hyperref}
\hypersetup{
    unicode=false,          % non-Latin characters in Acrobat's bookmarks
    pdftoolbar=true,        % show Acrobat toolbar?
    pdfmenubar=true,        % show Acrobat menu?
    pdffitwindow=false,      % page fit to window when opened
    pdftitle={},    % title
    pdfauthor={}
    pdfsubject={},   % subject of the document
    pdfnewwindow=true,      % links in new window
    pdfkeywords={keywords}, % list of keywords
    colorlinks=true,       % false: boxed links; true: colored links
    linkcolor=DarkRed,          % color of internal links
    citecolor=DarkGreen,        % color of links to bibliography
    filecolor=DarkRed,      % color of file links
    urlcolor=DarkBlue,          % color of external links
}
\usepackage{xspace}
\usepackage{cleveref}
\usepackage{thmtools, thm-restate}
\usepackage{natbib}

\newcommand{\kibitz}[2]{\ifnum\Comments=1{\textcolor{#1}{\textsf{\footnotesize #2}}}\fi}

\newcommand\NN{\mathbb{N}}
\newcommand\RR{\mathbb{R}}

\newcommand\cL{\mathcal{L}}

\newcommand\cE{\mathcal{E}}

\newcommand\cX{\mathcal{X}}
\newcommand\cD{\mathcal{D}}

\newcommand\cN{\mathcal{N}}

\newcommand\cV{\mathcal{V}}

\newcommand\cZ{\mathcal{Z}}
\newcommand\cO{\mathcal{O}}

\newcommand{\mcH}{\mathcal{H}}

\newcommand\lres{\ell_{\textrm{res}}}

\newcommand{\bW}{\mathbf{W}}
\newcommand{\psimax}{\psi_{\max}}

\newcommand{\ldot}[2]{\left\langle #1, #2 \right\rangle}

\DeclareMathOperator{\polylog}{polylog}

\DeclareMathOperator*{\Expectation}{\mathbb{E}}
\newcommand{\Ex}[2]{\Expectation_{#1}\left[#2\right]}

\newcommand{\eps}{\varepsilon}
\def\epsilon{\varepsilon}

\DeclareMathOperator*{\argmin}{\mathrm{argmin}}

\newcommand{\INDSTATE}[1][1]{\STATE\hspace{#1\algorithmicindent}}

\newtheorem{theorem}{Theorem}[section]
\newtheorem{lemma}[theorem]{Lemma}

\newtheorem{claim}[theorem]{Claim}

\newtheorem{corollary}[theorem]{Corollary}

\newtheorem{assumpt}[theorem]{Assumption}
\newtheorem{example}[theorem]{Example}
\newtheorem{specification}{Specification}

\theoremstyle{definition}
\newtheorem{definition}[theorem]{Definition}

\newcommand{\reals}{\mathbb{R}}
\newcommand{\bb}[1]{\left[ #1 \right]}
\newcommand{\bp}[1]{\left( #1 \right)}

\newcommand{\ba}[1]{\left| #1 \right|}
\newcommand{\bn}[1]{\left\| #1 \right\|}
\newcommand{\Var}{\text{Var}}
\newcommand{\E}{\mathbb{E}}
\newcommand{\normal}{{\cal N}}
\newcommand{\diam}{\textnormal{diam}}
\newcommand{\mG}{\mathbb{G}}

\newcommand*{\defeq}{\mathrel{\vcenter{\baselineskip0.5ex \lineskiplimit0pt
                     \hbox{\scriptsize.}\hbox{\scriptsize.}}}%
                     =}

\newcommand{\figurescale}{.28}
\newcommand{\omitcm}[1]{}
\newcommand{\omitarxiv}[1]{#1}

\begin{document}
\twocolumn[ \icmltitle{Orthogonal Random Forest for Causal Inference}

\begin{icmlauthorlist}
\icmlauthor{Miruna Oprescu}{msr}
\icmlauthor{Vasilis Syrgkanis}{msr}
\icmlauthor{Zhiwei Steven Wu}{umn}
\end{icmlauthorlist}

\icmlaffiliation{msr}{Microsoft Research--New England}
\icmlaffiliation{umn}{University of Minnesota--Twin Cities}

\icmlcorrespondingauthor{Miruna Oprescu}{moprescu@microsoft.com }
\icmlcorrespondingauthor{Vasilis Syrgkanis}{vasy@microsoft.com}
\icmlcorrespondingauthor{Zhiwei Steven Wu}{zsw@umn.edu}


\vskip 0.3in
]



\printAffiliationsAndNotice{}  % leave blank if no need to mention equal contribution

\begin{abstract}
 
We propose the \emph{orthogonal random forest}, an algorithm that
combines %incorporates double machine learning---a method of using
\emph{Neyman-orthogonality} to reduce sensitivity with respect to
estimation error of nuisance parameters %to estimate the target parameter---
with generalized random forests \citep{ATW17}---a flexible non-parametric
method for statistical estimation of conditional moment models using random forests. We provide a
consistency rate and establish asymptotic normality for our
estimator. We show that under mild assumptions on the consistency rate
of the nuisance estimator, we can achieve the same error rate as an
oracle with a priori knowledge of these nuisance parameters. We show that when the nuisance functions have a locally sparse parametrization, 
then a local $\ell_1$-penalized regression achieves the required rate.
We apply our method
to estimate heterogeneous treatment effects from observational
data with discrete treatments or continuous treatments, and 
we show that, unlike prior work, our method provably allows to control for a high-dimensional set
of variables under standard sparsity conditions. 
We also provide a comprehensive empirical evaluation of our
algorithm on both synthetic and real data. 
%, and show that it consistently outperforms baseline approaches.
\looseness=-1

\end{abstract}

\section{Introduction}

 Many problems that arise in causal inference
 can be formulated in the language of conditional moment models: given a target
 feature $x$ find a solution $\theta_0(x)$ to a system of conditional moment equations
 \begin{equation}
 \textstyle{\Ex{}{\psi(Z;
    \theta, h_0(x, W)) \mid X=x} = 0,}
\end{equation}
given access to $n$ i.i.d. samples from the data generating distribution, where $\psi$ is a known score function and $h_0$ is an unknown nuisance function that also needs to be estimated from data. Examples include non-parametric regression, heterogeneous treatment effect estimation, instrumental variable regression, local maximum likelihood estimation and estimation of structural econometric models.\footnote{See e.g. \citet{Reiss2007} and examples in \citet{Chernozhukov2016locally,Chernozhukov2018plugin}} The study of such conditional moment restriction problems has a long history in econometrics (see e.g. \citet{Newey1993,Ai2003,Chen2009,Chernozhukov2015}). 

\swedit{In this general estimation problem, the main} goal is to
estimate the target parameter at a rate that is robust to the
estimation error of the nuisance component. This allows \swedit{the
  use of} flexible models to fit the nuisance functions and enables
asymptotically valid inference. Almost all prior work on the topic has
focused on two settings: i) they either assume the target function
$\theta_0(x)$ takes a parametric form and allow for a potentially
high-dimensional parametric nuisance function,
e.g. \cite{Chernozhukov2016locally,Chernozhukov2017,Chernozhukov2018plugin},
ii) \swedit{or} take a non-parametric stance at estimating
$\theta_0(x)$ but do not allow for high-dimensional \swedit{nuisance functions}~\cite{WA,ATW17}.

We propose \emph{Orthogonal Random Forest} (ORF), a random forest-based
estimation algorithm, which performs non-parametric
estimation of the target parameter while permitting
more complex nuisance functions with high-dimensional
parameterizations. Our estimator is also asymptotically normal and hence
allows for the construction of asymptotically valid confidence
intervals via plug-in or bootstrap approaches. Our approach combines
the notion of \emph{Neyman orthogonality} of the moment equations with
a two-stage random forest based algorithm, which generalizes prior
work on \emph{Generalized Random Forests}~\cite{ATW17} and the double
machine learning (double ML) approach proposed in
\cite{Chernozhukov2017}. \swedit{To support our general algorithm, we
  also provide a novel nuisance estimation algorithm---\emph{Forest
    Lasso}---that effectively recovers high-dimensional nuisance
  parameters provided they have locally sparse structure.} This result combines techniques from Lasso theory~\cite{Lasso-Book}
with concentration inequalities for $U$-statistics~\cite{Hoeffding}.

{As a concrete example and as a main application of our approach, we
  consider the problem of \emph{heterogeneous treatment effect}
  estimation. This problem is at the heart of many decision-making
  processes, including clinical trial assignment to patients, price
  adjustments of products, and ad placement by a search engine.  In
  many situations, we would like to take the heterogeneity of the
  population into account and estimate the \emph{heterogeneous
    treatment effect} (\emph{HTE})---the effect of a treatment $T$
  (e.g. drug treatment, price discount, and ad position), on the
  outcome $Y$ of interest (e.g. clinical response, demand, and
  click-through-rate), as a function of observable characteristics
  $x$ of the treated subject (e.g. individual patient, product, and
  ad). HTE estimation is a fundamental problem in causal inference
  from observational data \citep{imbens_rubin_2015,WA,ATW17}, and is
  intimately related to many areas of machine learning, including
  contextual bandits, off-policy evaluation and optimization
  \citep{Swaminathan16,Wang17,NW17}, and
  counterfactual prediction \citep{SwaminathanJ15,Hartford}. }

{The key challenge in HTE estimation is that the observations are
  typically collected by a policy that depends on confounders or
  control variables $W$, which also directly influence the outcome.
  Performing a direct regression of the outcome $Y$ on the treatment
  $T$ and features $x$, without controlling for a multitude of other
  potential confounders, will produce biased estimation. This leads to
  a regression problem that in the 
  language of conditional moments takes the form:
  \begin{equation}
  \textstyle{\E\bb{Y - \theta_0(x) T - f_0(x, W) \mid X=x} = 0 }
  \end{equation}  
  where $\theta_0(x)$ is the heterogeneous effect of the treatment $T$ (discrete or continuous) on the 
  outcome $Y$ as a function of the features $x$ and $f_0(x,W)$ is an unknown nuisance function that captures the
  direct effect of the control
  variables on the outcome.
  Moreover, 
  unlike active experimentation settings such as contextual bandits, when
  dealing with observational data, the actual treatment or logging policy
  $\E\bb{T | x, W} = g_0(x, W)$ that could potentially be used to de-bias
  the estimation of $\theta_0(x)$ is also unknown.
%  In machine
%  learning, similar problems are studied under the literature of
%  off-policy evaluation and contextual bandits, where the learner
%  observes the outcome only for the selected treatment.
%  The standard
%  approach in these lines of work is to impose active experimentation
%  or to assume that the data collection or logging policy is known to
%  the analyst. However, in many real-world settings, active
%  experimentation may be infeasible, and there can be a vast amount of
%  observational data for which the logging policy is
%  unknown. Leveraging such observational data can also benefit the
%  optimization of online learning algorithms (e.g. by identifying a
%  suitable policy class for the contextual bandit algorithm). }

There is a surge of recent work at the interplay of machine learning
and causal inference that studies efficient estimation and inference of treatment
effects. \citet{doubleML} propose a two-stage estimation method called
\emph{double machine learning} that first orthogonalizes
out the effect of high-dimensional confounding factors using
sophisticated machine learning algorithms, including Lasso, deep
neural nets and random forests, and then estimates the effect of the
lower dimensional treatment variables, by running a low-dimensional
linear regression between the residualized treatments and residualized
outcomes. They show that \swedit{even}
if the estimation error of the first stage is \swedit{not particularly accurate}, the second-stage estimate can still be
${n^{-1/2}}$-asymptotically normal. However, their approach requires a parametric specification of $\theta_0(x)$. In contrast, another line of work that brings machine learning to causal inference
provides fully flexible non-parametric HTE estimation based
on random forest techniques \citep{WA,ATW17,Powers2017}. However,
these methods heavily rely on low-dimensional assumptions. 
% and tend to
%perform poorly in the presence of high-dimensional controls.

% provide a flexible method for non-parametric HTE estimation from
% observational data, as a function of a small number of features $x$,
% while controlling for a large number of confounding variables $W$
% that influence the logging policy.
Our algorithm ORF, when applied to the HTE problem (see
Section~\ref{sec:hte}) allows for the non-parametric estimation of
$\theta_0(x)$ via forest based approaches while simultaneously
allowing for a high-dimensional set of control variables $W$. 
%Our
%approach applies to settings with multiple continuous or discrete
%treatments. 
This estimation problem is of practical importance when
a decision maker (DM) wants to optimize a policy that depends only on a small
set of variables, e.g. due to data collection or regulatory constraints or
due to interpretability of the resulting policy, while at the same time
controlling for many potential confounders in the existing data that could lead to biased estimates. Such settings naturally arise in contextual pricing or personalized medicine. In such settings the DM is faced with the problem of estimating a conditional average treatment effect
conditional on a small set of variables while controlling for a much larger set.
Our estimator provably offers a significant statistical advantage for this task over prior approaches.

In the HTE setting, the ORF algorithm follows the
residual-on-residual regression approach analyzed by
\cite{Chernozhukov2016locally} to formulate a locally Neyman
orthogonal moment and then applies our orthogonal forest algorithm to
this orthogonal moment.
Notably, \cite{ATW17} also recommend such a residual on residual regression approach
in their empirical evaluation, which they refer to as ``local centering'', albeit with no
theoretical analysis. Our results provide a theoretical foundation of the local centering
approach through the lens of Neyman orthogonality. Moreover, our theoretical
results give rise to a slightly different overall estimation approach than the one 
in \cite{ATW17}: namely we residualize locally around the target estimation point $x$, as opposed
to performing an overall residualization step and then calling the Generalized Random Forest algorithm on the residuals. The latter stems from the fact that our results require that 
the nuisance estimator achieve a good estimation rate \emph{only} around the target point $x$. Hence, residualizing locally seems more appropriate than running a global nuisance estimation, which would typically minimize a non-local mean squared error. Our experimental findings reinforce this intuition (see e.g. comparison between ORF and the GRF-Res benchmark).
Another notable work that combines the residualization idea with flexible heterogeneous effect estimation is that of \cite{NW17}, who formulate the problem as an appropriate residual-based square loss minimization over an arbitrary hypothesis space for the heterogeneous effect function $\theta(x)$. Formally, they show robustness, with respect to nuisance estimation errors, of the mean squared error (MSE) of the resulting estimate in expectation over the distribution $X$ and for the case where the hypothesis space is a reproducing kernel Hilbert space (RKHS). Our work differs primarily by: i) focusing on sup-norm estimation error at any target point $x$ as opposed to MSE, ii) using forest based estimation as opposed to finding a function in an RKHS, iii) working with the general orthogonal conditional moment problems, and iv) providing asymptotic normality results and hence valid inference.

We provide a comprehensive empirical comparison of ORF with several benchmarks, including three
variants of GRF. We show that by setting the parameters according to
what our theory suggests, ORF consistently outperforms all of the
benchmarks. Moreover, we show that bootstrap based confidence intervals
provide good finite sample coverage.

Finally, to motivate the usage of the ORF, we applied our technique to Dominick's dataset, a popular historical dataset of store-level orange juice prices and sales provided by University of Chicago Booth School of Business. The dataset is comprised of a large number of covariates $W$, but economics researchers might only be interested in learning the elasticity of demand as a function of a few variables $x$ such as income or education. We applied our method (see Appendix~\ref{sec:oj} for details) to estimate orange juice price elasticity as a function of income, and our results, depicted in Figure~\ref{fig:oj}, unveil the natural phenomenon that lower income consumers are more price-sensitive.

\begin{figure}
\centering
\includegraphics[scale=0.4]{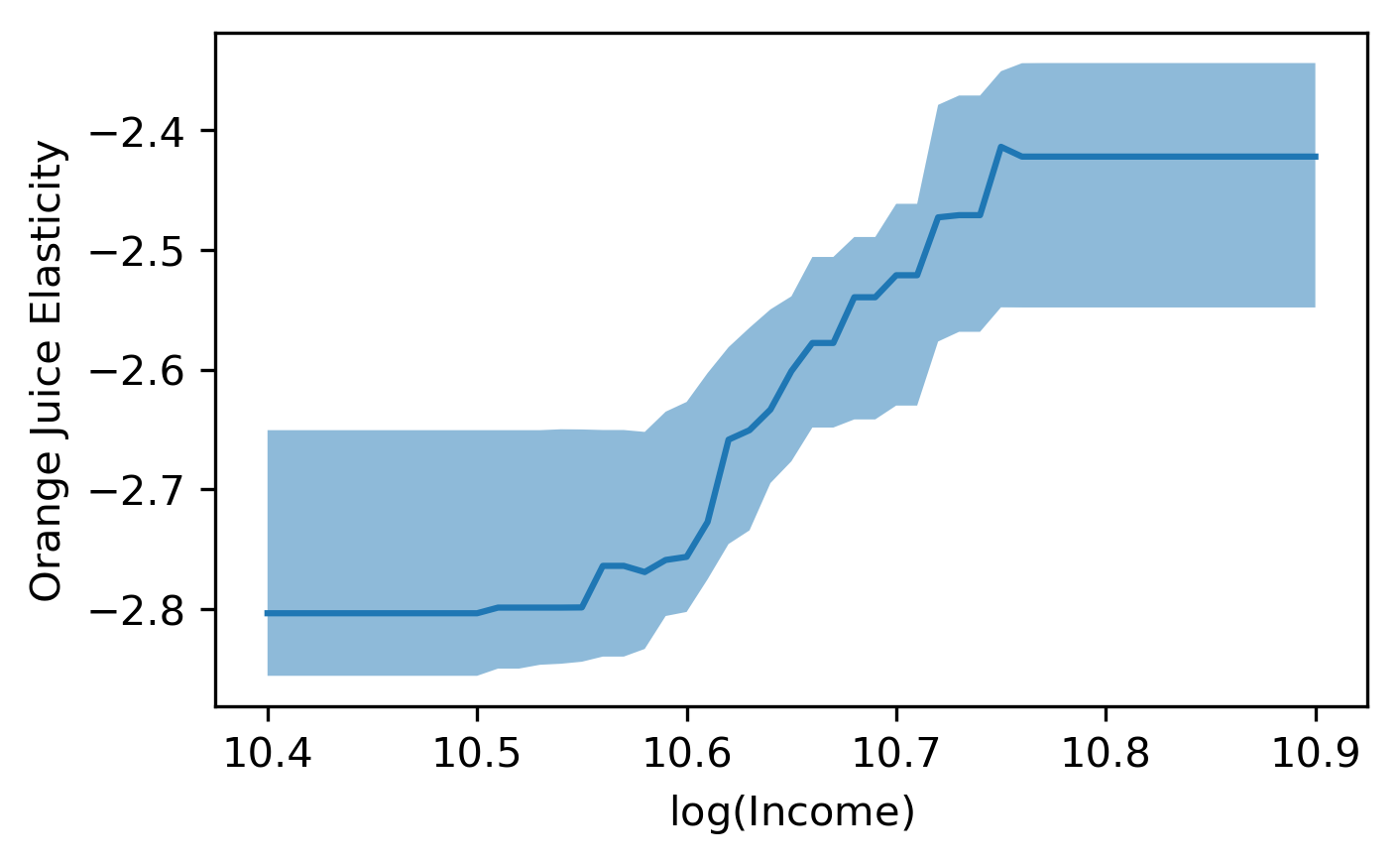}
\caption{ORF estimates for the effect of orange juice price on demand from a high-dimensional dataset. We depict the estimated heterogeneity in elasticity by income level. The shaded region depicts the 1\%-99\% confidence interval obtained via bootstrap.}\label{fig:oj}
\end{figure}

\vscomment{\subsection{Related Work}
we need to add an extensive comparison to many relevant literatures.}

% \swcomment{TODO: \cite{NW17}, \cite{ZSE17}}
%%% Local Variables:
%%% mode: latex
%%% TeX-master: "main.tex"
%%% End:

% !TEX root = main.tex
\section{Estimation via Local Orthogonal Moments}\label{GMM}
\vsedit{
We study non-parametric estimation of models defined via conditional moment restrictions, in the presence of nuisance functions.  Suppose we have a set of $2n$ observations
$Z_1, \ldots , Z_{2n}$ drawn independently from some underlying
distribution $\cD$ over the observation domain $\cZ$.  Each observation $Z_i$
contains a feature vector $X_i\in \cX \defeq [0, 1]^d$.
} 

\vsedit{
Given a target feature $x\in \cX$, our goal is to estimate a
parameter vector $\theta_0(x)\in \RR^p$ that is defined via a local moment condition, i.e. for all $x\in \cX$, $\theta_0(x)$ is the unique solution with respect to $\theta$ of:
\begin{equation}
 \textstyle{\Ex{}{\psi(Z;
    \theta, h_0(x, W)) \mid X=x} = 0,}
\end{equation}
where 
$\psi\colon \cZ\times \RR^p \times \RR^\ell \rightarrow \RR^p$ is a
score function that maps an observation $Z$, parameter vector
  $\theta(x)\in \Theta \subset \RR^p$, and nuisance vector $h(x, w)$ to a
vector-valued score $\psi(z; \theta(x), h(x, w))$ and $h_0 \in H \subseteq \bp{\RR^d\times \RR^L\rightarrow \RR^\ell}$ is an
unknown nuisance function that takes as input $X$ and a subvector
$W$ of $Z$, and outputs a
nuisance vector in $\RR^\ell$. 
For any feature $x \in \cX$, parameter 
  $\theta\in \Theta$, and nuisance function
$h \in H$, we define the \emph{moment} function as:
\begin{equation}
\textstyle{m(x; \theta, h) = \E\bb{\psi(Z; \theta, h(X, W))\mid X=x}}
\end{equation}
We assume that the
dimensions $p, \ell, d$ are constants, while the dimension $L$ of $W$ can be growing with
$n$.
}

\vsedit{
We will analyze the following two-stage estimation process.
\begin{enumerate}[topsep=0pt,itemsep=-1ex,partopsep=1ex,parsep=1ex]
\item \textit{First stage.} Compute a nuisance estimate $\hat h$ for
  $h_0$ using data $\{Z_{n+1}, \ldots, Z_{2n}\}$ 
  with some guarantee on the conditional root mean squared error:\footnote{Throughout the paper we denote with $\|\cdot\|$ the euclidean norm and with $\|\cdot\|_p$ the $p$-norm.}
  \begin{equation*}
  \textstyle{\cE(\hat{h}) = \sqrt{\E\bb{\|\hat{h}(x, W) - h_0(x, W)\|^2 \mid X=x}}}
  \end{equation*}
\item \textit{Second stage.} Compute a set of similarity weights
  $\{a_i\}$ over the data $\{Z_1, \ldots , Z_{n}\}$
  that measure the similarity between their feature vectors $X_i$ and the
  target $x$.
  \vsdelete{\footnote{We obtain the similarity weights using random forests, but they
can also be computed through other methods such as
$k$-NN.}} 
Compute the estimate $\hat \theta(x)$ using the nuisance
  estimate $\hat h$ via the plug-in weighted moment condition:
 \begin{equation}\label{weighted_moment}
   \textstyle{\hat \theta(x) \text{ solves: } \sum_{i=1}^{n}
   a_i\psi(Z_i; \theta, \hat h(X_i, W_i)) = 0}
\end{equation}
\end{enumerate}
}

\vsedit{
In practice, our framework permits the use of any method to estimate the nuisance function in the first stage.
However, since our description is a bit too abstract let us give a special case,
which we will also need to assume for our normality result. Consider the case when the nuisance function
$h$ takes the form $h(x, w) = g(w; \nu(x))$, for some known function $g$
but unknown function $\nu: \cX\rightarrow \reals^{d_{\nu}}$ (with $d_{\nu}$ potentially growing with $n$), i.e. locally around each $x$ the function
$h$ is a parametric function of $w$.
Moreover, the
parameter $\nu_0(x)$ of the true nuisance function
$h_0$ is identified as the minimizer of a local
loss:
\begin{equation}\label{eqn:local-loss}
\textstyle{\nu_0(x) = \argmin_{\nu \in \cV} \E\bb{\ell(Z; \nu) \mid X=x}}
\end{equation}
Then we can estimate $\nu_0(x)$ via a locally weighted and penalized 
empirical loss minimization algorithm. In particular in Section~\ref{sec:kernel-lasso}
we will consider the case of local $\ell_1$-penalized estimation that we 
will refer to as \emph{forest lasso} and which provides
formal guarantees in the
case where $\nu_0(x)$ is sparse.
}

The key technical condition that allows us to reliably perform the
two-stage estimation is the following \emph{local orthogonality}
condition, which can be viewed as a localized version of the
\emph{Neyman orthogonality} condtion~\cite{Neyman,doubleML} around the
neighborhood of the target feature $x$. 
Intuitively, the condition says that the score function $\psi$ is
insensitive to local perturbations in the nuisance parameters around their true values. 
%, and hence robust to the estimation error of these parameters.

\begin{definition}[Local Orthogonality]
  Fix any estimator $\hat h$ for the nuisance function. Then the
  Gateaux derivative with respect to $h$, denoted
  $D_\psi[\hat{h} - h_0 \mid x]$, is defined as:
  \begin{align*}
\textstyle{\Ex{}{\nabla_{h} \psi(Z, \theta_0(x), h_0(x, W)) (\hat{h}(x, W) - h_0(x, W)) \mid x}}
  \end{align*} 
  where $\nabla_{h}$ denotes the gradient of $\psi$ with respect to
  the final $\ell$ arguments. We say that the moment conditions are
  \emph{locally orthogonal} if for all $x$:
  \omitarxiv{$D_\psi[\hat{h} - h_0 \mid x] = 0$.}
  \omitcm{\begin{equation}
  D_\psi[\hat{h} - h_0 \mid x] = 0.
  \end{equation}}
\end{definition}

%%% Local Variables:
%%% mode: latex
%%% TeX-master: "main.tex"
%%% End:

% !TEX root = main.tex
\section{Orthogonal Random Forest}
\label{sec:plr}
\vsedit{We describe our main algorithm \emph{orthogonal random forest}
(ORF) for calculating the similarity weights in the second stage of the 
two stage estimation. In the next section we will see that we will be
using this algorithm for the estimation of the nuisance functions, so
as to perform a local nuisance estimation.}
At a high level, ORF
can be viewed as an orthogonalized version of GRF that is more robust
to the nuisance estimation error. 
%We will defer the details of the
%first-stage nuisance estimation to \Cref{sec:kernel-lasso} and focus
%on the second stage target estimation here. 
Similar to GRF, the
algorithm runs a tree learner over $B$ random \emph{subsamples $S_b$
(without replacement) of size $s< n$}, to build $B$ trees such that each tree indexed
by $b$ provides a tree-based weight $a_{ib}$ for each observation
$Z_i$ in the input sample. \vsedit{Then the ORF weight $a_i$ for each sample $i$ is
the average over the tree-weights $a_{ib}$.}
%, and solve for the ORF
%estimator given by \Cref{weighted_moment}. 

The tree learner starts with a root node that contains the entire
$\cX$ and recursively grows the tree to split $\cX$ into a set of
leaves until the number of observations in each leaf is not too small.
The set of neighboods defined by the leaves naturally gives a
simlarity measure between each observation and the target $x$.
Following the same approach of \cite{GRF,WA}, we maintain the
following tree properties in the process of building a tree.
\begin{specification}[Forest Regularity]\label{forest_spec}
The tree satisfies
\begin{itemize}[topsep=0pt,itemsep=-1ex,partopsep=1ex,parsep=1ex]
\item \textbf{Honesty}: we randomly partition the input sample $S$
  into two subsets $S^1$, $S^2$, then uses $S^1$ to place splits in
  the tree, and uses $S^2$ for estimation.
\item \textbf{$\rho$-balanced}: each split leaves at least a fraction
  $\rho$ of the observations in $S^2$ on each side of the split for
  some parameter of $\rho\leq 0.2$.
\item \textbf{Minimum leaf size $r$}: there are between $r$ and
  $2r - 1$ observations from $S^2$ in each leaf of the tree. 
\item \textbf{$\pi$-random-split}: at every step, marginalizing over the
  internal randomness of the learner, the probability that the next
  split occurs along the $j$-th feature is at least $\pi/d$
  for some $0 < \pi \leq 1$, for all $j = 1, ..., d.$\footnote{For example, this can
  be achieved by uniformly randomizing the splitting variable with
  probability $\pi$ or via a Poisson sampling scheme where a random subset
  of the variables of size $m$ is chosen to consider for candidate splits, with $m\sim \text{Poisson}(\lambda)$.}
\end{itemize}
\end{specification}

The key modification to GRF's tree learner is our incorporation of
orthogonal nuisance estimation in the splitting criterion.
\vsdelete{In order to accurately capture
the heterogeneity in the target parameter $\theta$ in the presence of
high-dimensional nuisance, 
we perform a two-stage estimation locally at
each internal node to decide the next split. }
While the splitting
criterion does not factor into our theoretical analysis (similar to
\cite{GRF}), we find it to be an effective practical heuristic.

\paragraph{Splitting criterion with orthogonalization.}
\vsedit{At each internal node $P$ we perform a two-stage estimation over
$(P\cap S^1)$, i.e.~the set of examples in $S^1$ that reach node $P$:
1) compute a nuisance estimate $\hat h_P$ using only data $P\cap S^1$ (e.g. by estimating
a parameter $\hat \nu_P$ that minimizes $\sum_{i\in (P\cap S^1)} \ell(Z_i; \nu) + \lambda \|\nu\|_1$ and setting $\hat{h}_P(\cdot) = g(\cdot; \hat{\nu}_P)$),
and then 2) form estimate $\hat \theta_P$ using $\hat h_P$:\footnote{In our implementation we actually use a cross-fitting approach, where we use 
half of $P\cap S^1$ to compute a nuisance function to apply to the other half and vice versa.}
\[
  \textstyle{\hat \theta_P \in\argmin_{\theta\in \Theta} \left\|\sum_{i\in (P\cap
      S^1)} \psi(Z_i; \theta, \hat h_P(W_i)) \right\|}
\]
We now generate a large random set of candidate axis-aligned splits (satisfying Specification~\ref{forest_spec} and
we want to find 
the split into two children $C_1$ and $C_2$
such that if we perform the same two-stage estimation separately at
each child, the new estimates $\hat \theta_{C_1}$ and
$\hat \theta_{C_2}$ take on very different values, so that the
heterogeneity of the two children nodes is maximized. Performing the two-stage estimation of $\hat \theta_{C_1}$ and
$\hat \theta_{C_2}$ for all candidate splits is too computationally
expensive. Instead, we will approximate these estimates by taking a
Newton step from the parent node estimate $\hat \theta_P$: for any
child node $C$ given by a candidate split, our proxy estimate is: 
\[
  \textstyle{\tilde \theta_{C} = \hat \theta_P - \frac{1}{|C\cap S^1|}\sum_{i\in
    C_j\cap S^1}} A_P^{-1} \psi(Z_i; \hat \theta_P, \hat h_P(X_i, W_i))}
\] 
where
$A_P = \frac{1}{|P\cap S^1|} \sum_{i\in P\cap S_b^1} \nabla_\theta
\psi(Z_i; \hat \theta_P, \hat h_P(X_i, W_i))$.}
\vsedit{We select the candidate split
that maximizes the following proxy heterogeneity score: for each coordinate $t\in [p]$ let 
\begin{equation}
 \textstyle{\tilde \Delta_t(C_1, C_2) = \sum_{j=1}^2 \frac{1}{|C_j\cap S^1|} \left(\sum_{i\in C_j\cap S^1}\rho_{t,i}\right)^2}
\end{equation}
where 
$\rho_{t,i} = A_P^{-1} \psi_t(Z_i; \hat \theta_P, \hat h_P(X_i, W_i))$. 
\omitcm{$\tilde \Delta_t$ is a heterogeneity score of the candidate split for the $t$-th coordinate of parameter $\theta$.}
We then create a single heterogeneity score per split as a convex combination that puts weight $\eta$ on the mean and $(1-\eta)$ on the maximum score across coordinates. $\eta$ is chosen uniformly at random in $[0,1]$ at each iteration of splitting. Hence, some splits focus on heterogeneity on average\omitcm{(i.e. capturing latent factors of global heterogeneity in all coordinates of $\theta$)}, while others focus on creating heterogeneity on individual coordinates\omitcm{(i.e. capturing coordinate specific sources of heterogeneity)}.
}

\paragraph{ORF weights and estimator.}
\vsedit{For each tree indexed $b\in [B]$ based on subsample $S_b$, let
$L_b(x)\subseteq \cX$ be the leaf that contains the target feature $x$.
We assign \emph{tree weight} and \emph{ORF weight} to each observation
$i$:
\begin{align*}
  \textstyle{a_{ib} = \frac{\mathbf{1}[(X_i\in L_b(x)) \wedge (Z_i\in S_b^2)]}{|L_b(x) \cap S_b^2|},} \quad \textstyle{a_i = \frac{1}{B}\sum_{b=1}^B a_{ib}}
\end{align*}
}

\citet{WA} show that under the structural specification of the trees,
the tree weights are non-zero only around a small neighborhood of $x$; 
a property that we will leverage in our analysis.

\begin{theorem}[Kernel shrinkage \cite{WA}]\label{kernel-bound}
  Suppose the minimum leaf size
  parameter $r = O(1)$, the tree is $\rho$-balanced and
  $\pi$-random-split and the distribution of $X$ 
  admits a density in $[0,1]^d$ that is bounded away from zero and infinity. Then
  the tree weights satisfy
  $$\Ex{}{\sup\{\|x- x_i\|\colon a_{ib} > 0 \}} =
  O(s^{-\frac{1}{2\alpha d}}),$$ where
  $$\alpha = \frac{\log(\rho^{-1})}{\pi\log((1 - \rho)^{-1}) }$$ and $s$ is size of the subsamples.
\end{theorem}

\iffalse
\begin{algorithm}[h]
\begin{algorithmic}[1]
  \STATE{\textbf{Input}: a dataset $D$ of size $n$, target point $x$,
    kernel base learner $\cL_k$, parameter $B$: number of kernels,
    parameter $s$ subsample size} 

  \STATE{Randomly split the data set $D$ into two data sets of equal
    size $D_1$ and $D_2$}
  \STATE{\textbf{First stage: nuisance estimation.}}
  \STATE{Compute nuisance estimate $\hat h$ by performing a kernel
    nuisance estimation using dataset $D_1$}

    \STATE{\textbf{Second stage: moment estimation with $\hat h$}.}

    \STATE{For each $Z_i\in D_2$, set weight $a_{ib} = 0$ for all
      $b\in[B]$}

    \STATE{\textbf{For} $b = 1 ,\ldots , B$: }

    \INDSTATE{Let $S_{b}$ be a set of $s$ observations randomly
      subsampled without replacement from $D_2$}
    % ; Randomly partition
    % $S_{b}$ into two data sets of even size: $S_{b}^1$ and
    % $S_{b}^2$

    \INDSTATE{Run tree learner on $S_b$ and compute tree weight on
      each observation $i\in D_2$.}
    \STATE{\textbf{For} each $Z_i\in D_2$, let $a_i = \frac{1}{B}\sum_{b=1}^B a_{ib}$}
  \STATE{Compute $\hat \theta$ by solving the weighted moment
    condition in \Cref{weighted_moment}.
    }
    \caption{Orthogonal Random Forest}
  \label{alg:mainalg}
\end{algorithmic}
\end{algorithm}
\fi

\section{Convergence and Asymptotic Analysis}\label{sec:conv-rate}

\vsedit{The \emph{ORF estimate} $\hat{\theta}$ is computed by solving the
weighted moment condition in \Cref{weighted_moment}, using the \emph{ORF weights}
as described in the previous section. We now provide theoretical guarantees for $\hat{\theta}$
under the following
assumption on the moment, score fuction and the data generating process.}
\begin{assumpt}\label{tech-conditions}
  The moment condition and the score function satisfy the following:
\begin{enumerate}[topsep=0pt,itemsep=-1ex,partopsep=1ex,parsep=1ex]
\item \textbf{Local Orthogonality.} The moment condition satisfies
  local orthogonality.
\item \textbf{Identifiability.}  The moments $m(x; \theta, h_0)=0$ has
  a unique solution $\theta_0(x)$.
% , i.e.
  % $\forall \epsilon > 0: \inf\{ \|m(x; \theta, h_0)\|: \|\theta -
  % \theta_0(x)\| \geq \epsilon\} > 0$
\item \textbf{Smooth Signal.} The moments $m(x; \theta, h)$ are
  $O(1)$-Lipschitz in $x$ for any $\theta \in \Theta, h\in H$.
\item \textbf{Curvature.} The Jacobian
  $\nabla_\theta m(x; \theta_0(x), h_0)$ has minimum eigenvalue
  bounded away from zero.
\item \textbf{Smoothness of scores.} For every $j\in [p]$ and for all
  $\theta$ and $h$, the eigenvalues of the expected Hessian
  $\Ex{}{\nabla^2_{(\theta, h)} \psi_j(Z; \theta, h(W))\mid x, W}$ are
  bounded above by a constant $O(1)$.  For any $Z$, the score
  $\psi(Z; \theta, \xi)$ is $O(1)$-Lipschitz in
  $\theta$ for any $\xi$ and $O(1)$-Lipschitz in $\xi$
  for any $\theta$. The gradient of the score with respect
  to $\theta$ is $O(1)$-Lipschitz in $\xi$.

\item \textbf{Boundedness.} The parameter set $\Theta$ has constant
  diameter.  There exists a bound $\psimax$ such that for any
  observation $Z$, the first-stage nuisance estimate $\hat h$
  satisfies {$\|\psi(Z; \theta, \hat h)\|_{\infty}\leq \psimax$} for
  any $\theta\in \Theta$.

\vsedit{\item \textbf{Full Support $X$.} The distribution of $X$ admits a density that is bounded away from zero and infinity.}
% \item The subsampling parameter is chosen such that
%   $s^{-\tau} + \psimax \sqrt{s\ln(n)/n} = o(1)$.
\end{enumerate}
\end{assumpt}

 All the results presented in the remainder
of the paper will assume these conditions and we omit stating so in each of the theorems. Any extra conditions required for each theorem will be explicitly provided. Note that except for the local orthogonality condition, all of the
assumptions are imposing standard boundedness and regularity
conditions of the moments. 

\omitcm{We now present estimation error and asymptotic 
normality guarantees for the ORF estimate $\hat{\theta}$.}

\begin{theorem}[$L^q$-Error Bound]\label{thm:lq_error}
Suppose that: 
$\E\bb{\cE(\hat{h})^{2q}}^{1/2q} \leq \chi_{n,2q}$. Then:
\begin{equation*}
\textstyle{\E\bb{\|\hat{\theta} - \theta_0\|^q}^{1/q} = O\bp{ \frac{1}{s^{\frac{1}{2\alpha d}}} + \sqrt{\frac{s\log(\frac{n}{s})}{n} }  + \chi_{n, 2q}^2}}
\end{equation*}
\end{theorem}

\begin{theorem}[High Probability Error Bound]\label{thm:high-prob-error}
Suppose that the score is the gradient of a convex loss\omitcm{, i.e.:
\begin{equation}
\psi(z; \theta, h) = \nabla_{\theta} \ell(z; \theta, h)
\end{equation}}
and let $\sigma>0$ denote the minimum eigenvalue of the jacobian $M$. Moreover, suppose that the nuisance estimate satisfies that w.p. $1-\delta$: $\cE(\hat{h}) \leq \chi_{n,\delta}$.
Then w.p. $1-2\delta$:
\begin{equation}
\textstyle{\|\hat{\theta} - \theta_0\| = \frac{O\bp{ s^{-\frac{1}{2\alpha d}} + \sqrt{\frac{s\, \log(\frac{n}{s\, \delta})}{n} } + \chi_{n,\delta}^2 }}{\sigma - O(\chi_{n,\delta})}}
\end{equation}
\end{theorem}

For asymptotic normality we will restrict our framework to the case of parametric nuisance functions, i.e. $h(X, W) = g(W; \nu(X))$ for some
known function $g$ and to a particular type of nuisance estimators that recover the true parameter $\nu_0(x)$. Albeit we note that the parameter $\nu(X)$ can be an arbitrary non-parametric function of $X$ and can also be high-dimensional.
We will further assume that the moments also have a smooth co-variance structure in $X$, i.e. if we let 
$$V=\psi(Z; \theta_0(x), g(W; \nu_0(x)))$$ 
then $\Var(V \mid X=x')$ is Lipschitz in $x'$ for any $x'\in [0,1]^d$.

\begin{theorem}[Asymptotic Normality]\label{thm:asym_norm}
Suppose that $h_0(X, W)$ takes a locally parametric form $g(W; \nu_0(X))$, for some known function $g(\cdot; \nu)$ that is $O(1)$-Lipschitz in $\nu$ w.r.t. the $\ell_r$ norm for some $r\geq 1$ and
the nuisance estimate is of the form $\hat{h}(X, W) = g(W; \hat{\nu}(x))$ and satisfies:
\begin{equation*}
\textstyle{\E\bb{ \bn{\hat{\nu}(x) - \nu_0(x)}_r ^4}^{1/4} \leq \chi_{n,4} = o\bp{\bp{s/n}^{1/4}}}
\end{equation*}
Suppose that $s$ is chosen such that: $$s^{-1/(2\alpha d)} = o((s/n)^{1/2-\epsilon}),$$ for any $\epsilon>0$, and $s=o(n)$. Moreover, $\Var(V \mid X=x')$ is Lipschitz in $x'$ for any $x'\in [0,1]^d$.
Then
for any coefficient $\beta\in \RR^p$, with $\|\beta\|\leq 1$, assuming $\Var(\beta^{\intercal} M^{-1} V | X=x') > 0$ for any $x'\in [0,1]^d$, there exists a sequence $\sigma_n = \Theta(\sqrt{\polylog(n/s) s/n})$,
such that:
\begin{equation}
\textstyle{\sigma_n^{-1} \ldot{\beta}{\hat{\theta}-\theta_0} \rightarrow_d \normal(0,1)}
\end{equation}
\end{theorem}

Given the result in \Cref{thm:asym_norm}, we can follow the same
approach of \emph{Bootstrap of Little Bags} by
\cite{ATW17,SEXTON2009801} to build valid confidence intervals.

%
%\begin{theorem}\label{thm:converge}
%  Suppose that the marginal each feature vectors $x_i$ are
%  independently and uniformly distributed over $[0, 1]^d$ and that the
%  moment conditions satisfy \Cref{tech-conditions}. {If the first stage
%  estimates are consistent at a sufficiently fast rate:
%\begin{align*}
%  &\Ex{}{\|\hat h(W) - h_0(W)\|^2\mid x}\\ \leq &{o}_p\left(\left( s^{-\tau} +
%    \psimax\sqrt{\frac{s \, p\, \ln(p\, n/ s)}{n}} \right) \right)
%  \end{align*}
%and the second stage estimate is consistent, i.e. $\Ex{}{ \|\hat{\theta} - \theta_0(x)\|} = o_p(1)$, 
%then the first stage estimation error can asymptotically be ignored and the final stage estimate $\hat{\theta}$ converges at an oracle rate of:
%\begin{align*}
%&\Ex{}{ \|\hat
%    \theta - \theta_0(x) \|} \\\leq &{O}\left( \sqrt{p} \left(s^{-\tau} +
%    \psimax\sqrt{\frac{s \, p\, \ln(p\, n/ s)}{n}} \right) \right)
%\end{align*}}
%\end{theorem}
%

%%% Local Variables:
%%% mode: latex
%%% TeX-master: "main.tex"
%%% End:

% !TEX root = main.tex
\section{Nuisance Estimation: Forest Lasso}\label{sec:kernel-lasso}
Next, we study the nuisance estimation problem in the first stage 
and provide a general nuisance estimation method that leverages
locally sparse parameterization of the nuisance function, permitting
low error rates even for high-dimensional problems.
%Recall that we
%would like to compute an estimate $\hat h$ such that the error
%$\sqrt{\Ex{}{\|\hat h(W) - h_0(W)\|^2 \mid x}}$ is small. Recall that
%we consider nuisance functions $\hat h$ that are parameterized by
%vectors $\hat \nu\in H\subset \mathbb{R}^{d_\nu}$, and estimate the
%target parameter $\nu_0(x)$ via the following $M$-esimation problem:
%\begin{equation}
%\nu_0(x) = \argmin_{\nu \in H} \Ex{}{\ell(Z; \nu) \mid x}
%\end{equation}
\vsedit{Consider the case when the nuisance function
$h$ takes the form $h(x, w) = g(w; \nu(x))$ for some known functional form $g$, for some known function $g$
but unknown function $\nu: \cX\rightarrow \reals^{d_{\nu}}$, with $d_{\nu}$ potentially growing with $n$.
Moreover, the
parameter $\nu_0(x)$ of the true nuisance function
$h_0$ is identified as the minimizer of some local loss, as defined in Equation~\eqref{eqn:local-loss}.}
% We get to observe data of the form $(x_i, z_i)$ drawn iid from some
% distribution and

\vsedit{We consider the following estimation process: given a set of
observations $D_1$, we run the same tree learner in
Section~\ref{sec:plr} over $B$ random subsamples (without replacement)
to compute ORF weights $a_i$ for each observation $i$ over
$D_1$.  Then we apply a local $\ell_1$
penalized $M$-estimation:
\begin{equation}
\textstyle{\hat{\nu}(x) = \arg\min_{\nu\in \cV} \sum_{i=1}^n a_i\, \ell(Z_i;\nu) + \lambda \|\nu\|_1}
\end{equation}
}

\vsedit{To provide formal guarantees for this method we will need to make the following assumptions.}
\begin{assumpt}[Assumptions for nuisance estimation]\label{ass:nuisance}
The target parameter and data distribution satisfy:
  \begin{itemize}[topsep=0pt,itemsep=-1ex,partopsep=1ex,parsep=1ex]
  \item For any $x\in \cX$, $\nu(x)$ is $k$-sparse with support
    $S(x)$.
  \item $\nu(x)$ is a $O(1)$-Lipschitz in $x$ and the function
    $\nabla_{\nu} L(x; \nu)=\Ex{}{\nabla_{\nu} \ell(Z; \nu) \mid X=x}$
    is $O(1)$-Lipschitz in $x$ for any $\nu$, with respect to the $\ell_2$
    norm.

  \item The data distribution satisfies the conditional restricted
    eigenvalue condition: for all $\nu\in \cV$ and for all $z\in \cZ$, for
    some matrix $\mcH(z)$ that depends only on the data:
$\nabla_{\nu\nu}\ell(z; \nu) \succeq \mcH(z) \succeq 0$,
and for all $x$ and for all 
$\nu\in C(S(x);3) \equiv \{\nu\in \RR^d: \|\nu_{S(x)^c}\|_1 \leq
3\|\nu_{S(x)}\|_1\}$:
\begin{equation}
\textstyle{\nu^T \Ex{}{\mcH(Z) \mid X=x} \nu \geq \gamma \|\nu\|_2^2}
\end{equation}
  \end{itemize}
\end{assumpt}
\vsedit{Under Assumption~\ref{ass:nuisance}
we show that the local penalized estimator achieves the following parameter recovery
guarantee.}

% i.e. $a_i(x) = \frac{1}{B} \sum_{b\in B} a_{ib}(x)$, where $a_{ib}$
% puts non-zero weights only on samples contained in sub-sample
% $b$. Moreover, the weights satisfy the kernel shrinkage property:
% $\Ex{}{\sup \left\{ \|x_i - x\|_2: a_{ib}(x) > 0\right\}} \leq
% s^{-\tau}$ for some sub-sampling tunable parameter $s$.

\iffalse
\begin{theorem}\label{nuisance-main} Assuming $\lambda \geq 2 \left\|\sum_{i} a_i(x) \nabla_\nu \ell(z_i;\nu_0(x))\right\|_{\infty}$ then with probability $1-\delta$:
\begin{equation}
\|\hat{\nu}(x) - \nu_0(x)\|_1 \leq \frac{2 \lambda k}{\gamma - 32k\sqrt{s\ln(p/\delta)/n}}
\end{equation}
Moreover:
\begin{equation}
\left\|\sum_{i} a_i(x) \nabla_\nu \ell(z_i;\nu_0(x))\right\|_{\infty} \leq L s^{-\tau} + \sqrt{\frac{s \ln(p/\delta)}{n}}
\end{equation}
\end{theorem}
\fi

\begin{theorem}\label{nuisance-main} With probability $1-\delta$:
  \begin{equation*}
\textstyle{\|\hat{\nu}(x) - \nu_0(x)\|_1 \leq \frac{2 \lambda k}{\gamma - 32k\sqrt{s\ln(d_{\nu}/\delta)/n}}}
\end{equation*}
as long as
$\lambda \geq \Theta\bp{s^{-1/(2 \alpha d)} + \sqrt{\frac{s \ln(d_\nu/\delta)}{n}}}$.
\end{theorem}

\begin{example}[Forest Lasso]
  For locally sparse linear regression,
  $Z_i=(x_i, y_i, W_i)$ and
  $\ell(Z_i; \nu) = (y_i - \ldot{\nu}{W_i})^2$. This means,
  $\nabla_{\nu\nu}\ell(Z_i; \nu) = W_iW_i^T = \mcH(Z_i)$. Hence, the
  conditional restricted eigenvalue condition is simply a conditional
  covariance condition: $\Ex{}{WW^\intercal \mid x} \succeq \gamma I$.
\end{example}

\begin{example}[Forest Logistic Lasso]
  For locally sparse logistic regression, $Z_i=(x_i, y_i, W_i)$,
  $y_i\in \{0,1\}$ and
  $$\ell(Z_i; \nu) = y_i \ln\left(\cL(\ldot{\nu}{W_i})\right) + (1-y_i)
  \ln\left(1-\cL(\ldot{\nu}{W_i})\right),$$ where
  $\cL(t) = 1/(1+e^{-t})$ is the logistic function. In this case,
  \begin{align*}
\nabla_{\nu\nu}\ell(Z_i; \nu) &= \cL(\ldot{\nu}{W_i})
  (1-\cL(\ldot{\nu}{W_i})) W_iW_i^\intercal\\ & \succeq \rho
    W_iW_i^\intercal = \mcH(Z_i)
    \end{align*}
  assuming the index $\ldot{\nu}{w}$ is
  bounded in some finite range. Hence, our conditional restricted
  eigenvalue condition is the same conditional covariance condition:
  $$\rho \Ex{}{WW^T \mid x} \succeq \rho \gamma I$$
\end{example}

\omitcm{The latter reasoning extends to any loss whose gradient takes the form $(G(\ldot{x}{\nu}) - y)x$, for some strictly monotone increasing function $G$. This is the class of single index regression models with a monotone link function, which encompasses a broad class of estimation problems.}

%%% Local Variables:
%%% mode: latex
%%% TeX-master: "main.tex"
%%% End:

% !TEX root = main.tex
\section{Heterogeneous Treatment Effects}\label{sec:hte}
Now we apply ORF to the problem of estimating \emph{heterogeneous
	treatment effects}.  We will consider the following \emph{extension}
of the \emph{partially linear regression (PLR)} model due to
\citet{Robinson}.
\footnote{The standard PLR model \cite{Robinson} considers solely the case of constant treatment effects, $Y = \ldot{\theta_0}{T} + f_0(X, W) + \eps$, and 
the goal is the estimation of the parameter $\theta_0$.} We have $2n$ i.i.d. observations
$D = \{Z_i = (T_i, Y_i, W_i, X_i)\}_{i=1}^{2n}$ such that for each
$i$, $T_i$ represents the treatment applied that can be either
real-valued (in $\mathbb{R}^p$) or discrete (taking values in
$\{0, e_1, \ldots , e_p\}$, where each $e_j$ denotes the standard
basis in $\RR^p$), $Y_i\in \RR$ represents the outcome,
$W_i\in [-1, 1]^{d_\nu}$ represents potential confounding variables
(controls), and $X_i\in \cX = [0, 1]^d$ is the feature vector that
captures the heterogeneity. The set of parameters are related via the
following equations:
\begin{align}
Y &= \ldot{\mu_0(X, W)}{T} + f_0(X, W) + \eps, \label{maineq}\\
T &= g_0(X, W) + \eta, \label{confounding}
\end{align}
where $\eta, \eps$ are bounded unobserved noises such that
$\Ex{}{\eps \mid W, X, T} = 0$ and
$\Ex{}{\eta \mid X, W, \eps} = 0$.\swcomment{added the last $\eps$}
In the main equation \eqref{maineq},
$\mu_0\colon \RR^d\times \RR^{d_\nu}\rightarrow [-1, 1]^p$ represents
the treatment effect function. Our goal is to estimate
\emph{conditional average treatment effect} (CATE) $\theta_0(x)$
conditioned on target feature $x$:
\begin{equation}
\textstyle{\theta_0(x) = \Ex{}{\mu_0(X, W)\mid X=x}}.
\end{equation}
%\swdelete{Note that $\theta_0$ is a sufficient statistics for
%	optimizing any treatment policy that depends on the features $x$,
%	since we can re-write the policy optimization objective as the
%	difference to reference treatment of $0$:
%	\begin{equation*}
%	\argmin_{\pi \in \cX \rightarrow \RR} \Ex{}{\theta_0(x_i)\, \pi(x_i) + f_0(W_i)} = \argmin_{\pi \in \cX \rightarrow \RR} \Ex{}{\theta_0(x_i) \pi(x_i)}
%	\end{equation*}}
%\swdelete{Our approach also extends to the case where
%	$\theta_0(x_i)$ is replaced with a function that also depends on the
%	controls, albeit in a linear manner
%	$\theta_0(x_i, W_i)=\alpha(x_i) + \langle \beta(x_i), W_i\rangle$, and
%	$\beta(x_i)$ is a sparse vector with support that is invariant to
%	$x_i$, while we are interested in estimating the average effect conditioned on $x_i$ and  averaged over $W_i$. We focus on the case where the treatment effect solely
%	depends on $x_i$ for expository
%	purposes.}\swcomment{removed non-linear stuff}
The confounding equation \eqref{confounding} determines the
relationship between treatments variable $T$ and the feature $X$
and confounder $W$. 
To create an orthogonal moment for identifying $\theta_0(x)$, we follow the classical \emph{residualization} approach similar to \cite{doubleML}. First, observe
that
\begin{equation*}
Y - \Ex{}{Y \mid X, W} = \ldot{\mu_0(X, W)}{  T - \Ex{}{T\mid X, W}} + \eps
\end{equation*}
Let us define the function $q_0(X, W) = \Ex{}{Y\mid X, W}$, and
consider the residuals 
\begin{align*}
\tilde Y &= Y - q_0(X, W)\\
\tilde T &= T - g_0(X, W) = \eta
\end{align*}
 Then we can simplify the equation
as $\tilde Y = \mu_0(X, W) \cdot \tilde T + \eps$. \swedit{As long as
	$\eta$ is independent of $\mu_0(X, W)$ conditioned on $X$ (e.g.
	$\eta$ is independent of $W$ or $\mu_0(X, W)$ does not depend on $W$),
} we also have
$$\Ex{}{\mu_0(X, W) \mid X, \eta} = \Ex{}{\mu_0(X, W)\mid X} = \theta(X).$$
Since $\Ex{}{\eps \mid X, \eta} = \E\bb{\E\bb{\eps\mid X, W, T}\mid X, \eta}=0$, then
$$\Ex{}{\tilde Y \mid X, \tilde T} = \Ex{}{\mu_0(X, W) \mid X}\cdot \tilde T = \theta(X) \cdot \tilde T.$$
This relationship suggests that we can obtain an estimate of
$\theta(x)$ by regressing $\tilde Y$ on $\tilde T$ locally around $X=x$. We can thus define
the \emph{orthogonalized} score function: for any observation
$Z = (T, Y, W, x)$, any parameter $\theta\in \RR^p$, any estimates $q$
and $g$ for functions $q_0$ and $g_0$, the score $\psi(Z; \theta, h(X, W))$ is: % defined as
\begin{align*}
\left\{Y - q(X,W) -  \theta \left(T - g(X, W)\rangle\right)  \right\} (T - g(X, W)),
\end{align*}
where $h(X, W) = (q(X, W), g(X, W))$. In the appendix, we show that
this moment condition satisfies local orthogonality, and it identifies
$\theta_0(x)$ as long as as the noise $\eta$ is independent of
$\mu_0(X,W)$ conditioned on $X$ and the expected matrix
$\Ex{}{\eta \eta^\intercal \mid X=x}$ is invertible.  Even though the
approach applies generically, to obtain formal guarantees on the
nuisance estimates via our Forest Lasso, we will restrict their
functional form.
%In the first stage estimation of ORF, we need to estimate the
%functions $q_0$ and $g_0$. To leverage our Forest Lasso procedure, we
%will restrict the functional forms of $f_0$, $\mu_0$ and $g_0$ so as to
%provide formal guarantees.

\paragraph{Real-valued treatments.}
Suppose $f_0$ and each coordinate $j$ of $g_0$ and $\mu_0$ are
given by high-dimensional linear functions:
$f_0(X, W) = \langle W, \beta_0(X) \rangle$,
$\mu_0^j(X, W) = \langle W, u_0^j(X) \rangle, \quad g_0^j(X, W) =
\langle W, \gamma_0^j(X)\rangle$, where
$\beta_0(X), \gamma_0^j(X), u_0^j(X)$ are $k$-sparse vectors in
$\RR^{d_\nu}$. Consequently,  $q_0(X, W)$ can be written as
%$$\langle W, \beta_0(X)\rangle + \sum_{j\in [p]} \sum_{r,
%	r'\in [d_\nu]}(u_0^j(X))_r (\gamma_0^j(X))_{r'} (W_r W_{r'}),$$
%which is 
a $k^2$-sparse linear function over degree-2 polynomial
features $\phi_2(W)$ of $W$. Then
as long as
$\gamma_0, \beta_0$ and $\mu_0$ are Lipschitz in $X$ and the
confounders $W$ satisfy $$\Ex{}{\phi_2(W) \phi_2(W)^\intercal\mid X} \succeq \Omega(1) I,$$
then we can use Forest Lasso to estimate both $g_0(x, w)$ and $q_0(x,w)$.
Hence, we can apply the ORF algorithm to get estimation error rates and asymptotic normality results
for $\hat{\theta}$. 
\omitarxiv{(see Appendix~\ref{app:hte_cor} for formal statement).}
\omitcm{
\begin{corollary}[Accuracy for real-valued treatments]\label{cor:continuous}
  Suppose that $\beta_0(X)$ and each coorindate
  $u_0^j(X), \gamma_0^j(X)$ are Lipschitz in $X$ and have $\ell_1$
  norms bounded by $O(1)$ for any $X$. Assume that distribution of $X$
  admits a density that is bounded away from zero and infinity.  For
  any feature $X$, the conditional covariance matrices satisfy
  $\Ex{}{\eta \eta^\intercal \mid X} \succeq \Omega(1) I_p$,
  $\Ex{}{W W^\intercal \mid X} \succeq \Omega(1)I_{d_\nu}$ and
  $\Ex{}{\varphi_2(W) \varphi_2(W)^\intercal \mid X} \succeq
  \Omega(1)I_{d_\nu^2 + d_\nu}$, where $\varphi_2(W)$ denotes the
  degree-2 polynomial feature vector of $W$. Then with probability
  $1 - \delta$, ORF returns an estimator $\hat \theta$ such that
  \[
    \textstyle{\|\hat \theta - \theta_0\| \leq O\left(n^{\frac{-1}{2 + 2\alpha d}}\sqrt{\log(n d_{\nu}/\delta)} \right)}
  \]
  as long as the sparsity
  $k \leq O\left(n^{\frac{1}{8 + 8\alpha d}} \right)$ and the
  subsampling rate of ORF $s = \Theta(n^{\alpha d/(1 + \alpha d)})$.
  Moreover, for any $b \in \RR^p$ with $\|b\|\leq 1$, there exists a
  sequence $\sigma_n = \Theta(\sqrt{\polylog(n) n^{-1/(1+\alpha d)}})$
  such that
  \[
    \textstyle{\sigma_n^{-1} \ldot{b}{\hat\theta -\theta} \rightarrow_d \mathcal{N}(0, 1),}
  \]
  as long as the sparsity
  $k=o\left(n^{\frac{1}{8 + 8\alpha d}} \right)$ and the subsampling
  rate of ORF $s = \Theta(n^{\epsilon + \alpha d/(1 + \alpha d)})$ for any $\epsilon>0$.
\end{corollary}
}

\paragraph{Discrete treatments.} We now describe how our theory can be
applied to discrete treatments.  Suppose $f_0$ and each coordinate $j$
of $g_0$ are of the form: $f_0(X, W) = \langle W, \beta_0(X) \rangle$
and $g_0^j(X, W) = \cL(\langle W, \gamma_0^j(X)\rangle)$, where
$\cL(t) = 1/(1 + e^{-t})$ is the logistic function. Note in this case
$\eta$ is not independent of $W$ since
$$\Var(\eta_j) = g_0^j(X,W)(1 - g_0^j(X, W)).$$ To maintain the
conditional independence between $\mu_0(X, W)$ and $\eta$ conditioned
on $X$, we focus on the setting where $\mu_0$ is only a function of
$X$, i.e.  $\mu(X, W) = \theta(X)$ for all $W, X$. In this setting we
can estimate $g_0$ by running a forest logistic lasso for each
treatment $j$. Then we can estimate $q_0(x, W)$ as follows: For each
$t\in \{e_1, \ldots, e_p\}$ estimate the expected counter-factual
outcome function:
$m_0^t(x, W) = \mu_0^t(x, W) + f_0(x, W)$, by running a forest lasso between $Y$ and $X, W$ only among the subset of samples that received treatment $t$. Similarly, estimate $f_0(x, W)$ by running a forest lasso between $Y$ and $X, W$ only among the subset of samples that received treatment $t=0$. 
\omitcm{Then observe that:
\begin{equation*}
q_0(x, W) = \sum_{t=1}^p (m_0^t(x, W) - f_0(x, w)) g_0^t(x, W) + f_0(x, W).
\end{equation*}}
\omitarxiv{Then observe that $q_0(x, W)$ can be written as a function of $f_0$, $g_0^t$ and $m_0^t$.}
Thus we can combine these estimates to get an estimate of
$q_0$. Hence, we can obtain a guarantee similar to that of
Corollary~\ref{cor:continuous} (see appendix).

\paragraph{Doubly robust moment for discrete treatments.} In the setting where $\mu$ also depends on $W$ and treatments are discrete, we 
can formulate an alternative orthogonal moment that identifies the CATE even when $\eta$ is correlated with $\mu(X, W)$. This moment is based on first constructing unbiased estimates of the counterfactual outcome $m_0^t(X, W) = \mu_0^t(X, W) + f_0(X, W)$ for every observation $X, W$ and for any potential treatment $t$, i.e. even for $t\neq T$. The latter is done by invoking the doubly robust formula \cite{robins1995semiparametric,cassel1976some,kang2007demystifying}:
\begin{equation*}
\textstyle{Y^{(t)} = m_0^t(X, W) + \frac{(Y - m_0^t(X, W))\mathbf{1}\{T=t\}}{g_0^t(X, W)}}
\end{equation*}
with the convention that 
\begin{align*}
g_0^0(X, W) &= 1 - \sum_{t\neq 0} g_0^t(X, W)\\
m_0^0(X,W)  &= f_0(X,W)
\end{align*}
Then we can identify the parameter $\theta^t(x)$ using the moment:
\omitarxiv{$\textstyle{\E[Y^{(t)}-Y^{(0)} | X=x] = \theta_t(x).}$}
\omitcm{\begin{equation}
\textstyle{\E[Y^{(t)}-Y^{(0)} | X=x] = \theta_t(x).}
\end{equation}}
One can easily show that this moment satisfies the Neyman
orthogonality condition with respect to the nuisance functions $m$ and
$g$ (see appendix). In fact this property is essentially implied by
the fact that the estimates $Y^{(t)}$ satisfy the double robustness
property, since double robustness is a stronger condition than
orthogonality\omitcm{(see e.g. \cite{Chernozhukov2016locally})}. We will again consider $\mu_0^j(X, W) = \langle W, u_0^j(X)
\rangle$. Then using similar reasoning as in the previous paragraph,
we see that with a combination of forest logistic lasso for $g_0^t$
and forest lasso for $m_0^t$, we can estimate these nuisance functions
at a sufficiently fast rate for our ORF estimator (based on this
doubly robust moment) to be asymptotically normal, assuming they have
locally sparse linear or logistic parameterizations.

\section{Monte Carlo Experiments}  \label{sec:mcexp}
\begin{figure*}[h!]
	\centering
	\includegraphics[width=\textwidth]{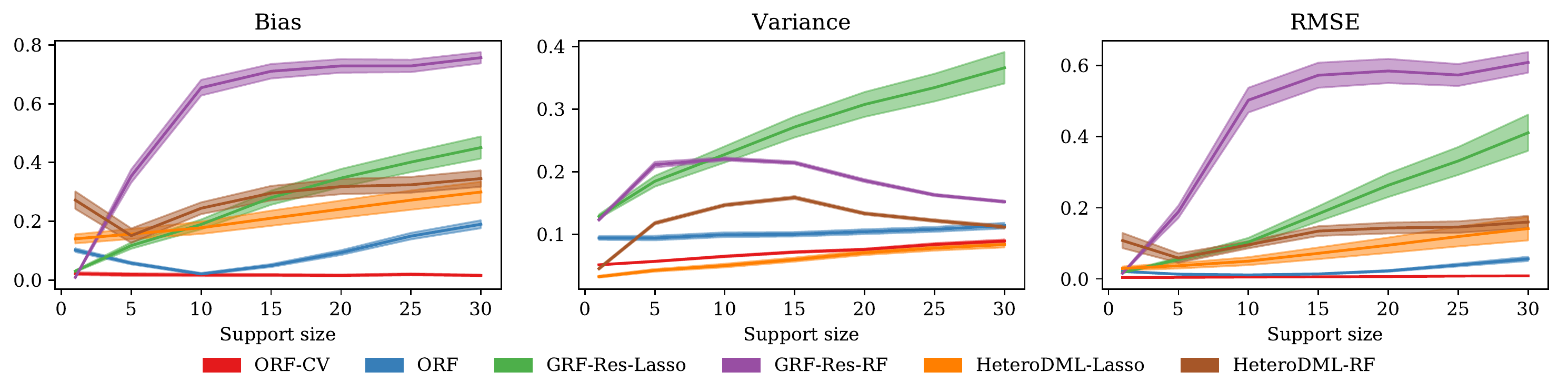}
	\caption{Bias, variance and RMSE as a function of support size for $n=5000$, $p=500$, $d=1$ and a piecewise linear treatment response function. The solid lines represent the mean of the metrics over Monte Carlo experiments and test points, and the filled regions depict the standard deviation, scaled down by 3 for clarity.}
	\label{fig:support}
\end{figure*}

\begin{figure*}[h]
	\centering
	\includegraphics[width=0.8\textwidth]{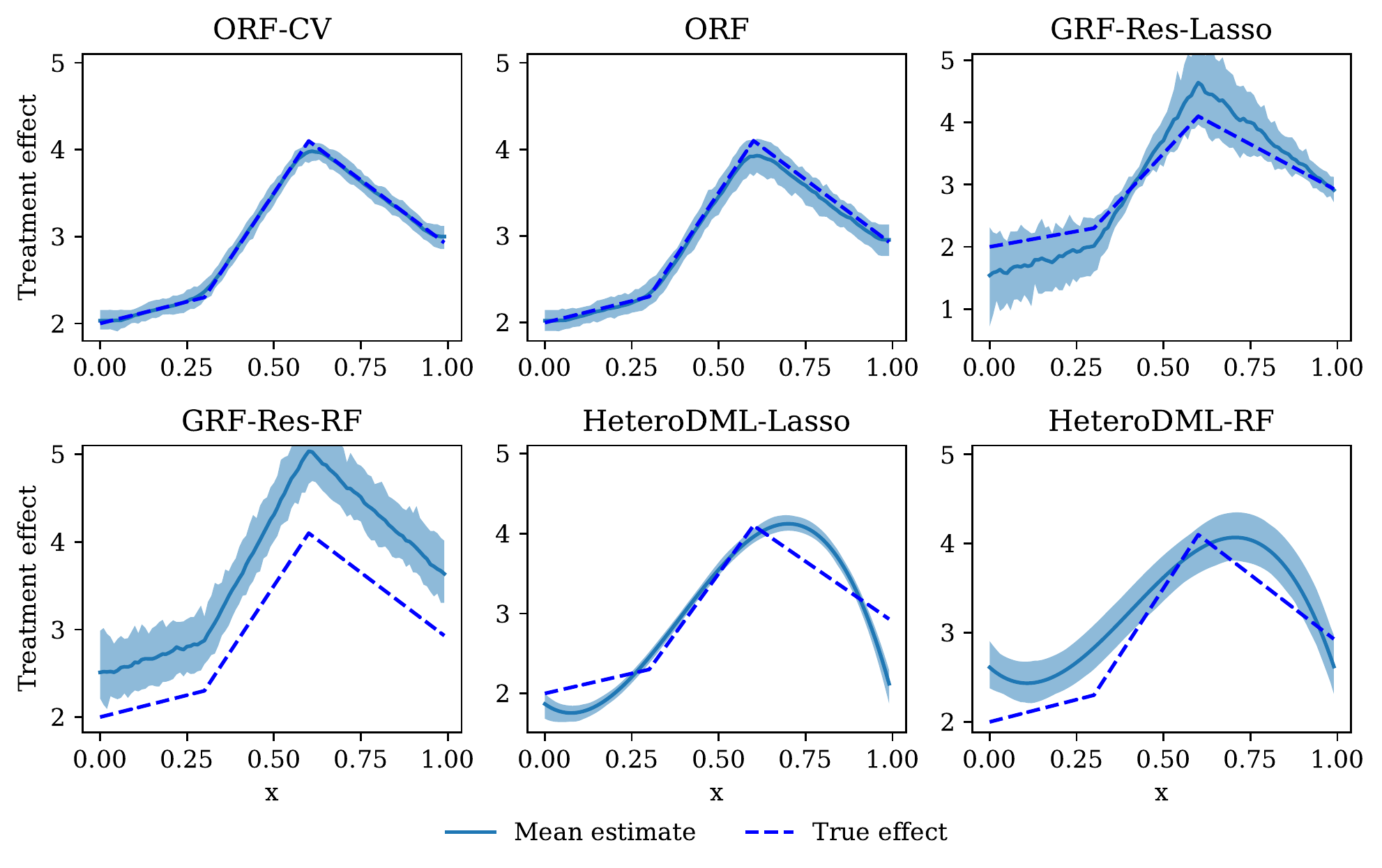}
	\caption{Treatment effect estimations for 100 Monte Carlo experiments with parameters $n=5000$, $p=500$, $d=1$, $k=15$, and $\theta(x) = (x+2)\mathbb{I}_{x\leq 0.3} + (6x+0.5)\mathbb{I}_{x>0.3 \text{ and } x \leq 0.6} + (-3x+5.9)\mathbb{I}_{x> 0.6}$. The shaded regions depict the mean and the $5\%$-$95\%$ interval of the 100 experiments.}
	\label{fig:examples}
\end{figure*}

We compare the empirical performance of ORF with other methods in the literature (and their variants).\footnote{The source code for running these experiments is available in the git repo \href{https://github.com/Microsoft/EconML/tree/master/prototypes/orthogonal\_forests}{Microsoft/EconML}.}
\omitcm{\paragraph{Data Generating Process (DGP).}}
The data generating process we consider is described by the following equations:
\begin{align*}
Y_i &= \theta_0(x_i)\, T_i + \langle W_i, \gamma_0 \rangle + \eps_i\\ 
T_i &= \langle W_i,\beta_0 \rangle + \eta_i
\end{align*}

Moreover, $x_i$ is drawn from the uniform distribution $U[0, 1]$, $W_i$ is drawn from
$\mathcal{N}(0, I_p)$, and the noise terms
$\eps_i \sim U[-1, 1], \eta_i \sim U[-1, 1]$. The $k$-sparse vectors
$\beta_0, \gamma_0 \in \RR^{p}$ have coefficients drawn independently
from $U[0, 1]$. The dimension $p = 500$ and we vary the support size
$k$ over the range of $\{1, 5, 10, 15, 20, 25, 30\}$.
We examine a treatment function $\theta(x)$ that is
continuous and piecewise linear (detailed in \Cref{fig:examples}).  
In Appendix~\ref{sec:allexp} we analyze other forms for $\theta(x)$.
\vsdelete{we show that ORF also outperforms other methods with
discontinuous $\theta(x)$ that is piecewise constant or piecewise
polynomial.}
 
\omitcm{\paragraph{Experiments setup.}}
For each fixed treatment function, we
repeat $100$ experiments, each of which consists of generating $5000$ observations from the DGP, drawing the vectors $\beta_0$
and $\gamma_0$, and estimating $\hat{\theta}(x)$ at 100 test points
$x$ over a grid in $[0, 1]$. We then calculate the bias, variance and
root mean squared error (RMSE) of each estimate $\hat{\theta}(x)$.
\iffull Here we report summary statistics of the median and $5-95$
percentiles of these three quantities across test points, so as to
evaluate the average performance of each method.  \fi \iffalse To
estimate $\hat{\theta}(x)$, we used a variety of methods in the
literature, as well as the Orthogonal Forest variants proposed in this
work.\fi We compare two variants of ORF with two variants of
GRF~\citep{ATW17} (see Appendix~\ref{sec:allexp} for a third variant) and two extensions of double ML methods for
heterogeneous treatment effect estimation~\citep{Chernozhukov2017}. \omitcm{Let us
describe the methods in more detail:}

\textit{ORF variants.} (1) ORF: We implement ORF as described in
\Cref{sec:plr}, setting parameters under the guidance of our
theoretical result: %(\Cref{thm:final-hetero-te} and \ref{lasso-err}): 
subsample size
$s \approx (n/\log(p))^{1/(2\tau+1)}$, Lasso regularization
$\lambda_\gamma, \lambda_q \approx \sqrt{\log(p) s / n}/20$ % \mocomment{We should remove these 2 numbers. They only make sense for certain n and p.}
(for both tree learner and kernel estimation), number of trees
$B=100 \geq n/s$, a max tree depth of $20$, and a minimum leaf size of
$r=5$. (2) ORF with LassoCV (ORF-CV): we replaced the Lasso algorithm
in ORF's kernel estimation, with a cross-validated Lasso for the
selection of the regularization parameter $\lambda_\gamma$ and
$\lambda_q$. ORF-CV provides a more systematic optimization over the
parameters.

\textit{GRF variants.} (1) GRF-Res-Lasso: We perform a naive combination of
double ML and GRF by first residualizing the treatments and outcomes
on both the features $x$ and controls $W$, then running GRF R
package by \citep{GRF} on the residualized treatments $\hat T$, residualized outcomes
$\hat Y$, and features $x$. A cross-validated Lasso is used for residualization. (2) GRF-Res-RF: We combine DoubleML and GRF as above, but we use cross-validated Random Forests for calculating residuals $\hat T$ and $\hat Y$.

\textit{Double ML with Polynomial Heterogeneity (DML-Poly).} An
extension of the classic Double ML procedure for heterogeneous
treatment effects introduced in \cite{Chernozhukov2017}. This method
accounts for heterogeneity by creating an expanded linear base of
composite treatments (cross products between treatments and features).
(1) Heterogeneous Double ML using LassoCV for first-stage estimation
(HeteroDML-Lasso): In this version, we use Lasso with cross-validation
for calculating residuals on $x\cup W$ in the first stage.  (2)
Heterogeneous Double ML using random forest for first-stage estimation
(HeteroDML-RF): A more flexible version that uses random forests to
perform residualization on treatments and outcomes. The latter
performs better when treatments and outcomes have a non-linear
relationship with the joint features of $(x, W)$.

We generated data according to the Monte Carlo process above
and set the parameters to $n=5000$ samples, $p=500$
controls, $d=1$ features and support size
$k \in \{1, 5, 10, 15, 20, 25, 30\}$ and three types of treatment
effect functions. In this section, we present the results for a \omitcm{continuous,} piecewise linear treatment effect
function. \omitcm{The remainder of the results can be found in
Appendix~\ref{sec:allexp}.}

In Figure
\ref{fig:examples}, we inspect the goodness of fit for the chosen
estimation methods across 100 Monte Carlo experiments. We note the
limitations of two versions of the GRF-Res estimators, GRF-Res-Lasso and GRF-Res-RF, in capturing the treatment effect function well. The GRF-Res-RF estimations have a
consistent bias as the Random Forest residualization cannot capture the dependency on the controls $W$ given their high-dimensionality. The HeteroDML methods are not
flexible enough to capture the complexity of the treatment effect
function. The best performers are the ORF-CV, ORF, and GRF-Res-Lasso, with
the latter estimator having a larger bias and variance.

\begin{figure}[htpb]
 	\centering
 	\includegraphics[scale=\figurescale]{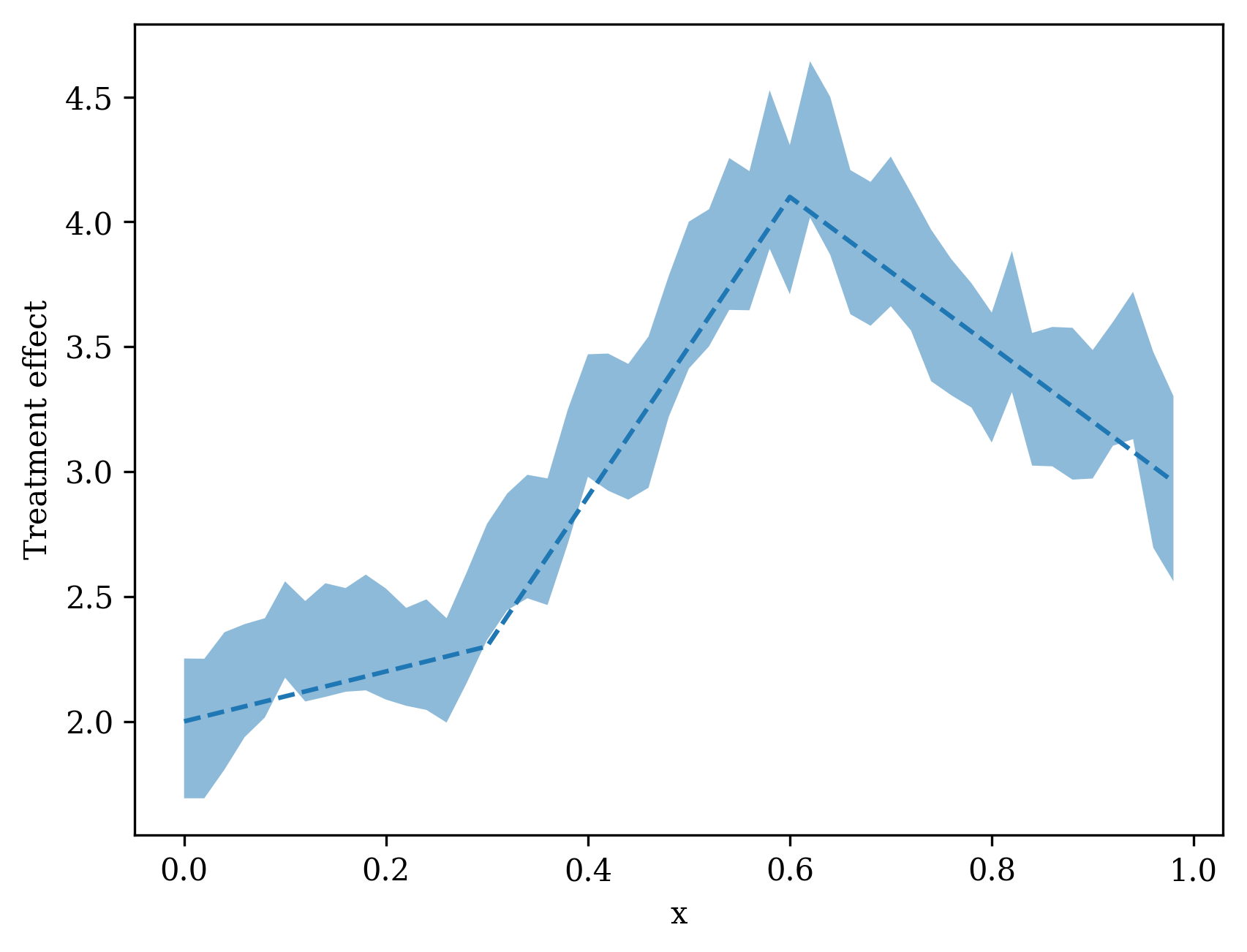}
 	\caption{Sample 1\%-99\% confidence intervals for 1000 bootstrap iterations with parameters $n=5000$, $p=500$, $d=1$, $k=15$, and $\theta(x) = (x+2)\mathbb{I}_{x\leq 0.3} + (6x+0.5)\mathbb{I}_{x>0.3 \text{ and } x \leq 0.6} + (-3x+5.9)\mathbb{I}_{x> 0.6}$. Approximately 90\% of the sampled test points are contained in the interval.}
 	\label{fig:coverage}
\end{figure}

% We quantify these observations with measurements of the bias, variance and root mean squared error (RMSE) of the estimates. 

%\begin{figure*}[!]
%	\centering
%	\includegraphics[scale=.45]{figures/step_n_5000_postprocess_results/pdf_low_res/Metrics_by_support_low_res.pdf}
%	\caption{Bias, variance and RMSE as a function of support size for $n=5000$, $p=500$, $d=1$ a piecewise constant treatment response function. The solid lines represent the mean of the metrics across test points, averaged over the Monte Carlo experiments, and the filled regions depict the standard deviation, scaled down by a factor of 3 for clarity.}
%	\label{fig:support_step}
%\end{figure*}
%
%\begin{figure*}[!]
%	\centering
%	\includegraphics[scale=.45]{figures/poly_n_5000_postprocess_results/pdf_low_res/Metrics_by_support_low_res.pdf}
%	\caption{Bias, variance and RMSE as a function of support size for $n=5000$, $p=500$, $d=1$ a piecewise polynomial treatment response function. The solid lines represent the mean of the metrics across test points, averaged over the Monte Carlo experiments, and the filled regions depict the standard deviation, scaled down by a factor of 3 for clarity.}
%	\label{fig:support_poly}
%\end{figure*}

We analyze %the behavior of 
these estimators as we increase the support size of $W$. Figures \ref{fig:support} %-\ref{fig:support_poly} 
illustrate %the variability in 
the evaluation metrics across different support sizes.
%for the three treatment effect functions described above. 
The ORF-CV performs very well, with consistent bias and RMSE across support sizes and treatment functions. The bias, variance and RMSE of the ORF grow with support size, but this growth is at a lower rate compared to the alternative estimators. The ORF-CV and ORF algorithms perform better than the GRF-Res methods on all metrics for this example. We observe this pattern for the other choices of support size, sample size and treatment effect function (see Appendix~\ref{sec:allexp}).
%: our algorithms consistently outperform the other methods we considered \iffull(see Appendix~\ref{sec:allexp}).\else (see the full version).\fi 
In Figure \ref{fig:coverage}, we provide a snapshot of the bootstrap confidence interval coverage for this example. \omitcm{We note that 90\% (45) of the test points lie within the calculated intervals.} 

%%% Local Variables:
%%% mode: latex
%%% TeX-master: "main"
%%% End:

\section*{Acknowledgements}
ZSW is supported in part by a Google Faculty Research Award, a J.P.
Morgan Faculty Award, a Mozilla research grant, and a Facebook
Research Award. Part of this work was completed while ZSW was at
Microsoft Research-New York City.

\newpage

\bibliographystyle{icml2019}

\bibliography{main.bbl}

\appendix
\onecolumn

% !TEX root = main.tex
\section{Two-forest ORF v.s. One-forest ORF} 
A natural question
about ORF is whether it is necessary to have two separate forests for
the two-stage estimation. We investigate this question by implementing
a variant of ORF without sample splitting (ORF-NS)---it builds only
one random forest over the entire dataset, and perform the two-stage
estimation using the same set of importance weights. We empirically
compare ORF-CV with ORF-NS. In Figure~\ref{fig:honesty}, we note that
the bias, variance and RMSE of the ORF-NS increase drastically with
the subsample ratio $(s/n)$, whereas the same metrics are almost
constant for the ORF-CV. This phenomenon is consistent with the
theory, since larger subsamples induce a higher probability of
collision between independently drawn samples, and the ``spill-over''
can incur large bias and error.

\begin{figure*}[!]
	\centering
	\includegraphics[scale=.45]{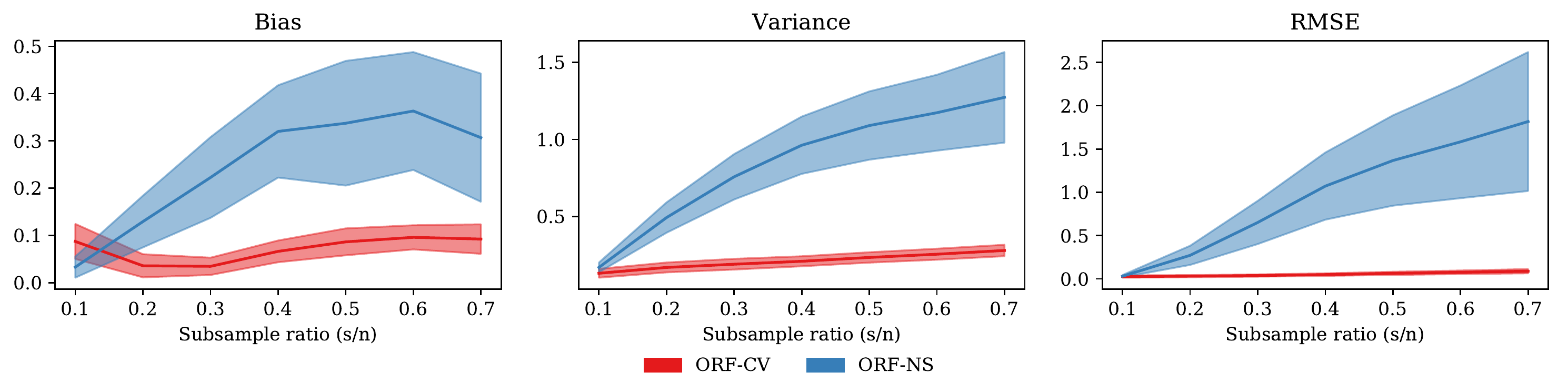}
	\caption{Bias, variance and RMSE versus subsample ratio used for training individual trees. The solid lines represent the means and the filled regions depict the standard deviation for the different metrics across test points, averaged over 100 Monte Carlo experiments.}
	\label{fig:honesty}
\end{figure*}

\omitarxiv{
\section{Formal Guarantee for Real-Valued Treatments}\label{app:hte_cor}
\begin{corollary}[Accuracy for real-valued treatments]\label{cor:continuous}
  Suppose that $\beta_0(X)$ and each coorindate
  $u_0^j(X), \gamma_0^j(X)$ are Lipschitz in $X$ and have $\ell_1$
  norms bounded by $O(1)$ for any $X$. Assume that distribution of $X$
  admits a density that is bounded away from zero and infinity.  For
  any feature $X$, the conditional covariance matrices satisfy
  $\Ex{}{\eta \eta^\intercal \mid X} \succeq \Omega(1) I_p$,
  $\Ex{}{W W^\intercal \mid X} \succeq \Omega(1)I_{d_\nu}$ and
  $\Ex{}{\varphi_2(W) \varphi_2(W)^\intercal \mid X} \succeq
  \Omega(1)I_{d_\nu^2 + d_\nu}$, where $\varphi_2(W)$ denotes the
  degree-2 polynomial feature vector of $W$. Then with probability
  $1 - \delta$, ORF returns an estimator $\hat \theta$ such that
  \[
    \textstyle{\|\hat \theta - \theta_0\| \leq O\left(n^{\frac{-1}{2 + 2\alpha d}}\sqrt{\log(n d_{\nu}/\delta)} \right)}
  \]
  as long as the sparsity
  $k \leq O\left(n^{\frac{1}{8 + 8\alpha d}} \right)$ and the
  subsampling rate of ORF $s = \Theta(n^{\alpha d/(1 + \alpha d)})$.
  Moreover, for any $b \in \RR^p$ with $\|b\|\leq 1$, there exists a
  sequence $\sigma_n = \Theta(\sqrt{\polylog(n) n^{-1/(1+\alpha d)}})$
  such that
  \[
    \textstyle{\sigma_n^{-1} \ldot{b}{\hat\theta -\theta} \rightarrow_d \mathcal{N}(0, 1),}
  \]
  as long as the sparsity
  $k=o\left(n^{\frac{1}{8 + 8\alpha d}} \right)$ and the subsampling
  rate of ORF $s = \Theta(n^{\epsilon + \alpha d/(1 + \alpha d)})$ for any $\epsilon>0$.
\end{corollary}
}

\section{Uniform Convergence of Lipschitz $U$-Processes}\label{app:concentration}

\begin{lemma}[Stochastic Equicontinuity for $U$-statistics via Bracketing]\label{lem:stoch_eq}
Consider a parameter space $\Theta$ that is a bounded subset of $\reals^p$, with $\diam(\Theta)=\sup_{\theta,\theta'\in \Theta} \|\theta-\theta'\|_2\leq R$. Consider the $U$-statistic over $n$ samples of order $s$:
\begin{equation}
    \mG_{s,n} f(\cdot; \theta) = \binom{n}{s}^{-1} \sum_{1\leq i_1\leq \ldots \leq i_s \leq n} f(z_{i_1},\ldots, z_{i_s}; \theta)
\end{equation}
where $f(\cdot; \theta): \cZ^s \rightarrow \reals$ is a known symmetric function in its first $s$ sarguments and $L$-Lipschitz in $\theta$. Suppose that $\sup_{\theta \in \Theta} \sqrt{\E\bb{f(Z_1,\ldots, Z_s; \theta)^2}} \leq \eta$ and $\sup_{\theta\in \Theta, Z_1, \ldots, Z_s \in \cZ^s} f(Z_1,\ldots, Z_s; \theta) \leq G$. Then w.p. $1-\delta$:
\begin{equation}
    \sup_{\theta \in \Theta} \ba{\mG_{s,n} f(\cdot; \theta) - \E[f(Z_{1:s}; \theta)]} = O\left(  \eta \sqrt{\frac{s\, (\log(n/s) + \log(1/\delta))}{n}} + (G+L\, R) \frac{s (\log(n/s) + \log(1/\delta))}{n} \right)
\end{equation}
\end{lemma}
\begin{proof}[Proof of Lemma~\ref{lem:stoch_eq}]
    Note that  for any fixed $\theta\in \Theta$, $\Psi_s(\theta, Z_{1:n})$ is a U-statistic of order $s$. Therefore by the Bernestein inequality for $U$-statistics (see e.g. Theorem~2 of \cite{peel2010empirical}), for any fixed $\theta\in \Theta$, w.p. $1-\delta$
    \begin{equation*}
        \ba{\mG_{s,n} f(\cdot; \theta)- \E[f(Z_{1:s};\theta)]} \leq \eta \sqrt{\frac{2\log(1/\delta)}{n/s}} + G\,\frac{2 \log(1/\delta)}{3 (n/s)}
    \end{equation*}
    Since $\diam(\Theta)\leq R$, we can find a finite space $\Theta_\epsilon$ of size $R/\epsilon$, such that for any $\theta\in\Theta$, there exists $\theta_\epsilon\in \Theta_\epsilon$ with $\|\theta-\theta_\epsilon\|\leq \epsilon$. Moreover, since $f$ is $L$-Lipschitz with respect to $\theta$:
    \begin{align*}
    \ba{\mG_{s,n} f(\cdot; \theta)- \E[f(Z_{1:s};\theta)]} \leq \ba{\mG_{s,n} f(\cdot; \theta_\epsilon)- \E[f(Z_{1:s}\theta_\epsilon)]} + 2 L \|\theta-\theta_\epsilon\|
    \end{align*}
    Thus we have that:
    \begin{align*}
    \sup_{\theta \in \Theta} \ba{\mG_{s,n} f(\cdot; \theta)- \E[f(Z_{1:s};\theta)]} \leq \sup_{\theta \in \Theta_\epsilon} \ba{\mG_{s,n} f(\cdot; \theta_\epsilon)- \E[f(Z_{1:s}\theta_\epsilon)]} + 2 L \epsilon
    \end{align*}
    Taking a union bound over $\theta \in \Theta_\epsilon$, we have that w.p. $1-\delta$:
    \begin{equation*}
   \sup_{\theta \in \Theta_\epsilon} \ba{\mG_{s,n} f(\cdot; \theta_\epsilon)- \E[f(Z_{1:s}\theta_\epsilon)]} 
     \leq \eta \sqrt{\frac{2\log(R/(\epsilon\,\delta))}{n/s}} + G\,\frac{2 \log(R/(\epsilon\, \delta))}{3 (n/s)}
    \end{equation*}
	Choosing $\epsilon = \frac{s\, R}{n}$ and applying the last two inequalities, yields the desired result.	
\end{proof}

\section{Estimation Error and Asymptotic Normality}

Since throughout the section we will fix the target vector $x$, we will drop it from the notation when possible, e.g. we will let $\theta_0 = \theta_0(x)$ and $\hat{\theta}=\hat{\theta}(x)$. 
%In this section, we will assume consistency of the overall estimation
%process, i.e. that $\|\hat \theta - \theta_0\| = o_p(1)$, and prove
%a rate of convergence conditional on consistency. In the next section
%we will also argue consistency under some extra set of convexity
%conditions. In essence, consistency allows us to restrict the estimate
%$\hat{\theta}$ in a cone around the true parameter, such that
%$\|\hat{\theta}-\theta_0\|^2 < \kappa_n
%\|\hat{\theta}-\theta_0\|$, for some $\kappa_n = o_p(1)$. Within
%this cone, we prove an explicit convergence rate of the estimation
%error.  We have a stronger consistency rate when the score is a
%gradient of loss.
%
%
We begin by introducing some quantities that will be useful throughout our theoretical analysis. 
First we denote with $\omega$ the random variable that corresponds to the internal randomness of the tree-splitting algorithm. Moreover, when the tree splitting algorithm is run
with target $x$, an input dataset of $\{Z_i\}_{i=1}^s$ and internal randomness $\omega$, we denote with $\alpha_i\bp{\{Z_i\}_{i=1}^s, \omega}$ the weight that it assigns to the sample with index $i$. Finally, for each sub-sample $b=1\ldots B$ we denote with $S_b$ the index of the samples chosen and $\omega_b$ the internal randomness that was drawn.

We then consider the weighted empirical score, weighted by the sub-sampled ORF weights:
\begin{align}
  \Psi(\theta, h) &= \sum_{i=1}^n a_{i} \, \psi(Z_i;
                 \theta,  h(W_i)) = \frac{1}{B} \sum_{b=1}^B \sum_{i\in S_b} \alpha_i\bp{\{Z_i\}_{i\in S_b}, \omega_b}\, \psi(Z_i; \theta, h(W_i))
\end{align}
We will also be considering the complete multi-dimensional $U$-statistic, where we average over all sub-samples of size $s$:
\begin{equation}
    \Psi_0(\theta, h) = \binom{n}{s}^{-1} \sum_{1\leq i_1 \leq \ldots \leq i_s \leq n} \E_{\omega}\bb{\sum_{t = 1}^s \alpha_{i_t}\bp{\{Z_{i_t}\}_{t=1}^s, \omega}\, \psi(Z_{i_t}; \theta, h(W_{i_t}))}\,.
\end{equation}
and we denote with:
\begin{equation}
f(Z_{i_1}, \ldots, Z_{i_s};\, \theta, h) = \E_{\omega}\bb{\sum_{t = 1}^s \alpha_{i_t}\bp{\{Z_{i_t}\}_{t=1}^s, \omega}\, \psi(Z_{i_t}; \theta, h(W_{i_t}))}
\end{equation}

First, we will bound the estimation error as a sum of
$m(x; \theta, \hat h)$ and second order terms. The proof follows from
the Taylor expansion of the moment function and the mean-value theorem.

\begin{lemma}\label{lem:taylor}
  Under \Cref{tech-conditions}, for any nuisance estimate $\hat h$ and for the ORF estimate $\hat{\theta}$ estimated with plug-in nuisance estimate $\hat{h}$:
  \begin{align*}
   \hat{\theta} - \theta_0 = M^{-1}\, \left(m(x; \hat{\theta}, \hat h) - \Psi(\hat{\theta}, \hat{h})\right) + \xi
  \end{align*}
  where $\xi$ satisfies $\|\xi\|=\cO\left(\Ex{}{\|\hat h(W) - h_0(W)\|^2\mid x}
      +\|\hat{\theta}-\theta_0\|^2\right)$.
\end{lemma}

\paragraph{Proof outline of main theorems.} We will now give a rough outline of the proof of our main results. In doing so we will also present some core technical Lemmas that we will use in the formal proofs of these theorems in the subsequent corresponding subsections. 

Lemma~\ref{lem:taylor} gives rise to the following core quantity:
\begin{equation}
  \Lambda(\theta, h) =   m(x; \theta, h) - \Psi(\theta, h)
\end{equation}
Suppose that our first stage estimation rate guarantees a local root-mean-squared-error (RMSE) of $\chi_n$, i.e.:
\begin{equation}
\cE(h) = \sqrt{\Ex{}{\|\hat h(W) - h_0(W)\|^2\mid x}} \leq \chi_n
\end{equation}
Then we have that:
\begin{align*}
   \hat{\theta} - \theta_0 = M^{-1}\,\Lambda(\hat{\theta}, \hat h)  + O(\chi_n^2 + \|\hat{\theta} - \theta_0\|^2)
\end{align*}
Thus to understand the estimation error of $\hat{\theta}$ and its asymptotic distribution, we need to analyze the concentration of $\Lambda(\hat{\theta}, \hat h)$ around zero and its asymptotic distribution. Subsequently, invoking consistency of $\hat{\theta}$ and conditions on a sufficiently fast nuisance estimation rate $\chi_n$, we would be able to show that the remainder terms are asymptotically negligible.

Before delving into our two main results on mean absolute error (MAE) and asymptotic normality we explore a bit more the term $\Lambda(\theta, h)$ and decompose it into three main quantities, that we will control each one separately via different arguments.
\begin{lemma}[Error Decomposition]\label{lem:mae_decomposition}
For any $\theta, h$, let $\mu_0(\theta, h) = \E\bb{\Psi_0(\theta,h)}$. Then:
\begin{align}
   \Lambda(\theta, h) =~& \underbrace{m(x; \theta,h) - \mu_0(\theta, h)}_{\Gamma(\theta,h) \textnormal{ = kernel error}} + \underbrace{\mu_0(\theta, h) - \Psi_0(\theta, h)}_{\Delta(\theta,h) \textnormal{ = sampling error}} + \underbrace{\Psi_0(\theta, h) - \Psi(\theta, h)}_{E(\theta,h) \textnormal{ = subsampling error}}\,.
\end{align}
\end{lemma}

When arguing about the MAE of our estimator, the decomposition presented in Lemma~\ref{lem:mae_decomposition} is sufficient to give us the final result by arguing about concentration of each of the terms. However, for asymptotic normality we need to further refine the decomposition into terms that when scaled appropriately converge to zero in probability and terms that converge to a normal random variable. In particular, we need to further refine the sampling error term $\Delta(\theta,h)$ as follows:
\begin{align}\label{eqn:stoch_eq_decomposition}
\Delta(\theta, h) = \underbrace{\Delta(\theta_0, \tilde{h}_0)}_{\textnormal{asymptotically normal term}} + \underbrace{\Delta(\theta, h) - \Delta(\theta_0, \tilde{h}_0)}_{F(\theta, h)\textnormal{ = stochastic equicontinuity term}}
\end{align}
for some appropriately defined fixed function $\tilde{h}_0$. Consider for instance the case where $\theta$ is a scalar. If we manage to show that there exists a scaling $\sigma_n$, such that $\sigma_n^{-1} \Delta(\theta_0, \tilde{h}_0) \rightarrow_d \normal(0,1)$, 
and all other terms $\Gamma, E, F$ and $\chi_n^2$ converge to zero in probability when scaled by $\sigma_n^{-1}$, then we will be able to conclude by Slutzky's theorem that:
$\sigma_n^{-1} M \left(\hat{\theta} - \theta_0\right) \rightarrow_d  \normal(0, 1)$ and establish the desired asymptotic normality result. 

Since controlling the convergence rate to zero of the terms $\Gamma, \Delta, E$ would be useful in both results, we provide here three technical lemmas that control these rates.

\begin{lemma}[Kernel Error]\label{lem:kernel_error}
If the ORF weights when trained on a random sample $\{Z_i\}_{i=1}^s$, satisfy that:
\begin{equation}
\E\bb{\sup \{\|X_i - x\|: a_{i}(\{Z_i\}_{i=1}^s, \omega) > 0\}} \leq \epsilon(s)
\end{equation}
where expectation is over the randomness of the samples and the internal randomness $\omega$ of the ORF algorithm. Then
\begin{equation}
\sup_{\theta, h} \|\Gamma(\theta, h)\| = \sqrt{p}\, L\, \epsilon(s)
\end{equation}
\end{lemma}

\begin{lemma}[Sampling Error]\label{lem:sampling_error}
Under Assumption~\ref{tech-conditions}, conditional on any nuisance estimate $\hat{h}$ from the first stage, with probability $1-\delta$:
\begin{equation}
\sup_{\theta} \|\Delta(\theta, \hat{h})\| = O\bp{\sqrt{\frac{s\, (\log(n/s) + \log(1/\delta))}{n}} }
\end{equation}
\end{lemma}
\begin{proof}
Since $\Delta(\theta,\hat{h})$ is a $U$-statistic as it can be written as: $\mG_{s,n} f(\cdot; \theta, \hat{h}) - \E\bb{f(Z_{1:s};\theta, \hat{h})}$. Moreover, under Assumption~\ref{tech-conditions}, the function $f(\cdot; \theta, \hat{h})$ satisfies the conditions of Lemma~\ref{lem:stoch_eq} with $\eta=G=\psi_{\max}=O(1)$. Moreover, $f(\cdot; \theta, \hat{h})$ is $L$-Lipschitz for $L=O(1)$, since it is a convex combination of $O(1)$-Lipschitz functions. Finally, $\diam(\Theta)=O(1)$. Thus applying Lemma~\ref{lem:stoch_eq}, we get, the lemma.
\end{proof}

\begin{lemma}[Subsampling Error]\label{lem:subsampling_error}
If the ORF weights are built on $B$ randomly drawn sub-samples with replacement, then
\begin{equation}
\sup_{\theta, h} \| E(\theta, \hat{h}) \| = O\bp{\frac{\log(B) + \log(1/\delta)}{\sqrt{B}}}
\end{equation}
\end{lemma}

\subsection{Consistency of ORF Estimate}
\begin{theorem}[Consistency]\label{thm:consistency} Assume that the nuisance estimate satisfies:
\begin{equation}
 \E\bb{\cE(\hat{h})}= o(1)
\end{equation}
and that $B \geq n/s$, $s=o(n)$ and $s\rightarrow \infty$ as $n\rightarrow \infty$. Then the ORF estimate $\hat{\theta}$ satisfies:
\begin{equation*}
\|\hat{\theta} - \theta_0(x)\| = o_p(1)
\end{equation*}
Moreover, for any constant integer $q\geq 1$:
\begin{equation}
\bp{\E[\|\hat{\theta} -\theta_0\|^{2q}]}^{1/q} = o\bp{\bp{\E\bb{\|\hat{\theta} -\theta_0\|^{q}}}^{1/q}}
\end{equation}
\end{theorem}
\begin{proof}
By the definition of $\hat{\theta}$, we have that: $\Psi(\hat{\theta}, \hat{h})=0$. Thus we have that: 
\begin{equation*}
\bn{m(x; \hat{\theta}, \hat{h})} = \bn{m(x; \hat{\theta}, \hat{h}) - \Psi(\hat{\theta}, \hat{h}} = \bn{\Lambda(\hat{\theta}, \hat{h})}
\end{equation*}
By Lemmas~\ref{kernel-bound}, \ref{lem:mae_decomposition}, \ref{lem:kernel_error}, \ref{lem:sampling_error} and \ref{lem:subsampling_error}, we have that with probability $1-2\delta$:
\begin{align*}
\bn{\Lambda(\hat{\theta}, \hat{h})} = O\bp{s^{-1/(2\alpha d)} + \sqrt{\frac{s\, (\log(n/s) + \log(1/\delta))}{n}} + \sqrt{\frac{\log(B) + \log(1/\delta)}{B}} }
\end{align*}
Integrating this tail bound we get that:
\begin{align*}
\E\bb{\bn{\Lambda(\hat{\theta}, \hat{h})}} = O\bp{s^{-1/(2\alpha d)} + \sqrt{\frac{s\, \log(n/s)}{n}} + \sqrt{\frac{\log(B)}{B}} }
\end{align*}
Thus if $B\geq n/s$, $s = o(n)$ and $s \rightarrow \infty$ then all terms converge to zero as $n\rightarrow \infty$.

Since $\psi(x; \theta, h(w))$ is $L$-Lipschitz in $h(w)$ for some constant $L$:
\begin{equation*}
\|m(x; \hat{\theta}, h_0)- m(x; \hat{\theta}, \hat{h})\| = L \E\bb{\bn{\hat{\theta}(W) - \hat{h}(W)} \mid x} \leq L \sqrt{\E\bb{\bn{\hat{\theta}(W) - \hat{h}(W)}^2 \mid x}} = L\cE(\hat{h})
\end{equation*}
Moreover, by our consistency guarantee on $\hat{h}$: 
\begin{equation*}
\E\bb{\|m(x; \hat{\theta}, h_0)- m(x; \hat{\theta}, \hat{h})\|} \leq L \E\bb{\cE(\hat{h})} = o(1)
\end{equation*}
Thus we conclude that:
\begin{equation*}
\E[\|m(x; \hat{\theta}, h_0)\|] = o(1)
\end{equation*}
which implies that $\|m(x;\hat{\theta}, h_0)\| = o_p(1)$.

By our first assumption, for any $\epsilon$, there exists a $\delta$, such that: $\Pr[\|\hat{\theta}-\theta_0(x)\| \geq \epsilon] \leq \Pr[\|m(x; \hat{\theta}, h_0)\|\geq \delta]$. Since $\|m(x; \hat{\theta}, h_0)\| = o_p(1)$, the probability on the right-hand-side converges to $0$ and hence also the left hand side. Hence, $\|\hat{\theta}-\theta_0(x)\| = o_p(1)$.

We now prove the second part of the theorem which is a consequence of consistency. By consistency of $\hat{\theta}$, we have that for any $\epsilon$ and $\delta$, there exists $n^*(\epsilon, \delta)$ such that for all $n\geq n(\epsilon, \delta)$:
\begin{align*}
\Pr\bb{\|\hat{\theta}-\theta_0\| \geq \epsilon} \leq \delta
\end{align*}
Thus for any $n\geq n^*(\epsilon, \delta)$:
\begin{align*}
\E[\|\hat{\theta} -\theta_0\|^{2q}] \leq \epsilon^{q} \E\bb{\|\hat{\theta} -\theta_0\|^{q}} + \delta \E\bb{|\hat{\theta} -\theta_0\|^{2q}}
\end{align*}
Choosing $\epsilon = (4C)^{-q}$ and $\delta=(4C)^{-1} (\diam(\Theta))^{-q}$ yields that:
\begin{align*}
\E[\|\hat{\theta} -\theta_0\|^{2q}] \leq \frac{1}{2C}\E\bb{\|\hat{\theta} -\theta_0\|^{q}}
\end{align*}
Thus for any constant $C$ and for $n\geq n^*((4C)^{-q}, (4C)^{-1} (\diam(\Theta))^{-q}) = O(1)$, we get that:
\begin{align*}
 \bp{\E[\|\hat{\theta} -\theta_0\|^{2q}]}^{1/q}
 =~&\frac{1}{(2C)^{1/q}}\bp{\E\bb{\|\hat{\theta} -\theta_0\|^{q}}}^{1/q}
 \end{align*}
 which concludes the claim that $\bp{\E[\|\hat{\theta} -\theta_0\|^{2q}]}^{1/q} = o\bp{\bp{\E\bb{\|\hat{\theta} -\theta_0\|^{q}}}^{1/q}}$.
\end{proof}

\subsection{Proof of Theorem \ref{thm:lq_error}: Mean $L^q$ Estimation Error}

%\begin{theorem}[$L^q$-Error Bound]\label{thm:lq_error}
%Assume that for some constant integer $q\geq 1$ the nuisance estimate satisfies:
%\begin{equation}
%\bp{\E\bb{\cE(\hat{h})^{2q}}}^{1/2q} \leq \chi_{n,2q}
%\end{equation}
%and that $B \geq n/s$, $s=o(n)$ and $s\rightarrow \infty$ as $n\rightarrow \infty$. Then the ORF estimate $\hat{\theta}$ satisfies:
%\begin{equation}
%\E\bb{\|\hat{\theta} - \theta_0\|^q}^{1/q} = O\bp{ \frac{1}{s^{\frac{1}{2\alpha d}}} + \sqrt{\frac{s\log(\frac{n}{s})}{n} }  + \chi_{n, 2q}^2}
%\end{equation}
%\end{theorem}
\begin{proof}
Applying Lemma~\ref{lem:taylor} and the triangle inequality for the $L^q$ norm, we have that:
\begin{align*}
  \bp{\E\bb{\|\hat{\theta} - \theta_0\|^q}}^{1/q} =~& O\bp{ \bp{\E\bb{\|\Lambda(\hat{\theta}, \hat h)\|^{q}}}^{1/q}  + \bp{\cE(\hat{h})^{2q}}^{1/q} + \bp{\E\bb{\|\hat{\theta} - \theta_0\|^{2q}}}^{1/q}}
\end{align*}
By assumption $\bp{\cE(\hat{h})^{2q}}^{1/q} \leq \chi_{n,2q}^2$. By the consistency Theorem~\ref{thm:consistency}: 
$$\bp{\E\bb{\|\hat{\theta} - \theta_0\|^{2q}}}^{1/q} = o\bp{\bp{\E\bb{\|\hat{\theta} - \theta_0\|^{q}}}^{1/q}}.$$
and therefore this term can be ignored for $n$ larger than some constant. Moreover, by Lemma~\ref{lem:mae_decomposition}:
\begin{align*}
\bp{\E\bb{\|\Lambda(\hat{\theta}, \hat h)\|^{q}}}^{1/q} =~& O\bp{ \bp{\E\bb{\|\Gamma(\hat{\theta}, \hat{h})\|^{q}}}^{1/q}  + \bp{\E\bb{\|\Delta(\hat{\theta}, \hat{h})\|^{q}}}^{1/q} + \bp{\E\bb{\|E(\hat{\theta}, \hat{h})\|^{q}}}^{1/q} }
\end{align*}
By Lemma~\ref{lem:kernel_error} and Lemma~\ref{kernel-bound} we have:
\begin{align*}
\bp{\E\bb{\|\Gamma(\hat{\theta}, \hat{h})\|^{q}}}^{1/q}  = O\bp{\epsilon(s)} = O\bp{s^{-1/(2\alpha d)}}
\end{align*}
for a constant $\alpha = \frac{\log(\rho^{-1})}{\pi\log((1 - \rho)^{-1}) }$. Moreover, by integrating the exponential tail bound provided by the high probability statements in Lemmas~\ref{lem:sampling_error} and \ref{lem:subsampling_error}, we have that for any constant $q$:
\begin{align*}
\bp{\E\bb{\|\Delta(\hat{\theta}, \hat{h})\|^{q}}}^{1/q}=~& O\bp{ \sqrt{\frac{s\log(n/s)}{n} } }\\
\bp{\E\bb{\|E(\hat{\theta}, \hat{h})\|^{q}}}^{1/q}=~& O\bp{ \sqrt{\frac{\log(B)}{B} } }
\end{align*}
For $B > n/s$, the second term is negligible compared to the first and can be ignored.
Combining all the above inequalities:
\begin{align*}
 \bp{\E\bb{\|\hat{\theta} - \theta_0\|^q}}^{1/q} =~& O\bp{ s^{-1/(2\alpha d)} + \sqrt{\frac{s\log(n/s)}{n} } }
 \end{align*}
 \end{proof}

\subsection{Proof of Theorem \ref{thm:high-prob-error}: Finite Sample High Probability Error Bound for Gradients of Convex Losses}
%\begin{theorem}[High Probability Error Bound]\label{thm:high-prob-error}
%Suppose that the score is the gradient of a convex loss, i.e.:
%\begin{equation}
%\psi(z; \theta, h) = \nabla_{\theta} \ell(z; \theta, h)
%\end{equation}
%and let $\sigma>0$ denote the minimum eigenvalue of the jacobian $M$. Moreover, suppose that the nuisance estimate satisfies that w.p. $1-\delta$: $\cE(\hat{h}) \leq \chi_{n,\delta}$
%and that $B \geq n/s$, $s=o(n)$ and $s\rightarrow \infty$ as $n\rightarrow \infty$. Then w.p. $1-3\delta$:
%\begin{equation}
%\|\hat{\theta} - \theta_0\| = \frac{O\bp{ s^{-\frac{1}{2\alpha d}} + \sqrt{\frac{s\, \log(\frac{n}{s\, \delta})}{n} } + \chi_{n,\delta}^2 }}{\sigma - O(\chi_{n,\delta})}
%\end{equation}
%\end{theorem}
\begin{proof}
We condition on the event that $\cE(\hat{h})\leq \chi_{n,\delta}$, which occurs with probability $1-\delta$. Since the Jacobian of $m(x; \theta, h_0)$ has eigenvalues lower bounded by $\sigma$ and each entry of the Jacobian of $\psi$ is $L$-Lispchitz with respect to the nuisance for some constant $L$, we have that for every vector $\nu \in \RR^p$ (with $p$ the dimension of $\theta_0$):
\begin{align*}
\frac{\nu^T \nabla_{\theta} m(x; \theta, \hat{h}) \nu}{\|\nu\|^2} \geq~& \frac{\nu^T \nabla_{\theta} m(x;\theta, h_0) \nu}{\|\nu\|^2} + \frac{\nu^T \nabla_{\theta}  \left(m(x; \theta, \hat{h}) - m(x;\theta, h_0)\right) \nu}{\|\nu\|^2}\\
\geq~& \sigma - L\cdot \Ex{}{\|\hat{h}(W) - h_0(W) \| \mid x} \frac{\|\nu\|_1^2}{\|\nu\|^2}\\
\geq~& \sigma - L\, \chi_{n,\delta}\, p = \sigma - O(\chi_{n,\delta})
\end{align*}
Where in the last inequality we also used Holder's inequality to upper bound the $L^1$ norm by the $L^2$ norm of the first stage error. Thus the expected loss function $L(\theta) = \Ex{}{\ell(Z; \theta, \hat{h}(W) \mid x}$ is $\hat{\sigma} = \sigma - O(\chi_{n,\delta})$ strongly convex, since $\nabla_{\theta} m (x; \theta, \hat{h})$ is the Hessian of $L(\theta)$. We then have:
\begin{align*}
L(\hat{\theta}) - L(\theta_0) \geq \nabla_{\theta} L(\theta_0)' (\hat{\theta} - \theta_0) + \frac{\hat{\sigma}}{2} \|\hat{\theta}-\theta_0\|^2 =  m(x; \theta_0, \hat{h})' (\hat{\theta}-\theta_0) + \frac{\hat{\sigma}}{2} \|\hat{\theta}-\theta_0\|^2
\end{align*}
Moreover, by convexity of $L(\theta)$, we have:
\begin{align*}
L(\theta_0) - L(\hat{\theta}) \geq \nabla_\theta L(\hat{\theta})' ( \theta_0 - \hat{\theta}) = m(x; \hat{\theta}, \hat{h})' ( \theta_0 - \hat{\theta})
\end{align*}
Combining the above we get:
\begin{align*}
\frac{\hat{\sigma}}{2} \|\hat{\theta}-\theta_0\|^2 \leq (m(x; \hat{\theta}, \hat{h}) - m(x; \theta_0, \hat{h}) )' ( \hat{\theta}- \theta_0 ) \leq \| m(x; \hat{\theta}, \hat{h}) - m(x; \theta_0, \hat{h})\|\, \|  \hat{\theta}- \theta_0 \|
\end{align*}
Dividing over by $\|\hat{\theta}-\theta_0\|$, we get:
\begin{align*}
\|\hat{\theta}-\theta_0\| \leq \frac{2}{\hat{\sigma}} \left( \| m(x; \hat{\theta}, \hat{h})\| + \| m(x; \theta_0, \hat{h})\|\right)
\end{align*}
The term $\| m(x; \hat{\theta}, \hat{h})\|$ is upper bounded by $\|\Lambda(\hat{\theta}, \hat{h})\|$ (since $\Psi(\hat{\theta}, \hat{h})=0$). Hence, by Lemmas~\ref{kernel-bound}, \ref{lem:mae_decomposition}, \ref{lem:kernel_error}, \ref{lem:sampling_error} and \ref{lem:subsampling_error} and our assumptions on the choice of $s, B$, we have that with probability $1-2\delta$:
\begin{align*}
\bn{\Lambda(\hat{\theta}, \hat{h})} = O\bp{s^{-1/(2\alpha d)} + \sqrt{\frac{s\, (\log(n/s) + \log(1/\delta))}{n}}}
\end{align*}

Subsequently, using a second order Taylor expansion around $h_0$ and orthogonality argument almost identical to the proof of Lemma \ref{lem:taylor}, we can show that the second term $\|m(x; \theta_0, \hat{h})\|$ is it upper bounded by $O(\chi_{n,\delta}^2)$. More formally, since $m(x;\theta_0, h_0)=0$ and the moment is locally orthogonal, invoking a second order Taylor expansion:
\begin{align*}
m_j(x; \theta_0, \hat{h}) =~& m_j(x; \theta_0, h_0)  + D_{\psi_j}[\hat h - h_0 \mid x]+ \underbrace{\frac{1}{2}\Ex{}{(\hat{h}(W)- h_0(W))^\intercal\nabla^2_{h} \psi_j(Z; \theta_0, \tilde{h}^{(j)}(W)) (\hat{h}(W)- h_0(W))\mid x}}_{\rho_j}\\
=~& \rho_j
\end{align*}
for some function $\tilde{h}_j$ implied by the mean value theorem. Since the moment is smooth, we have: $\|\rho\| = O\bp{\E\bb{\bn{\hat{h}(W)- h_0(W)}^2\mid x}} = O\bp{\chi_{n,\delta}^2}$. Thus $\bn{m_j(x; \theta_0, \hat{h})}=O\bp{\chi_{n,\delta}^2}$. Combining all the latter inequalities yields the result.
\end{proof}

\subsection{Proof of Theorem \ref{thm:asym_norm}: Asymptotic Normality}

%\begin{theorem}[Asymptotic Normality]\label{thm:asym_norm}
%Suppose that $h_0(X, W)$ takes a locally parametric form $g(W; \nu_0(X))$, for some known function $g$ and
%the nuisance estimate is of the form $\hat{h}(X, W) = g(W; \hat{\nu}(x))$ and satisfies:
%\begin{equation*}
%\E\bb{ \bn{\hat{\nu}(x) - \nu_0(x)} ^4}^{1/4} = o\bp{\bp{s/n}^{1/4}}
%\end{equation*}
%Suppose that $s$ is chosen such that: $s^{-1/(2\alpha d)} = o(\sqrt{s/n})$ and $s=o(n)$. Moreover, if we let $V=\psi(Z; \theta_0(x), g(W; \nu_0(x)))$ the $\Var(V \mid X=x')$ is Lipschitz in $x'$ for any $x'\in [0,1]^d$.
%Then
%for any coefficient $\beta$, with $\|\beta\|\leq 1$, assuming $\Var(\beta^{\intercal} M^{-1} V | X=x') > 0$ for any $x'\in [0,1]^d$, there exists a sequence $\sigma_n = \Theta(\sqrt{\polylog(n/s) s/n})$,
%s.t.:
%\begin{equation}
%\sigma_n^{-1} \ldot{\beta}{\hat{\theta}-\theta_0} \rightarrow_d \normal(0,1)
%\end{equation}
%\end{theorem}

\begin{proof}
We want to show asymptotic normality of any fixed projection $\ldot{\beta}{\hat{\theta}}$ with $\|\beta\|\leq 1$. First consider the random variable $V=\ldot{\beta}{M^{-1}\Delta(\theta_0, \tilde{h}_0)}$, where $\tilde{h}_0(X, W) = g(W; \nu_0(x))$, i.e. the nuisance function $\tilde{h}_0$ ignores the input $X$ and uses the parameter $\nu_0(x)$ for the target point $x$. Asymptotic normality of $V$ follows by identical arguments as in \cite{WA} or \cite{mentch2016quantifying}, since this term is equivalent to the estimate of a random forest in a regression setting, where we want to estimate $\E\bb{Y \mid X=x}$ and where the observation of sample $i$ is:
\begin{align}
Y_i = \ldot{\beta}{M^{-1}\, (m(X_i; \theta_0, \tilde{h}_0) - \psi(Z_i; \theta_0, \tilde{h}_0(X_i, W_i)))}
\end{align}
By Theorem~1 of \cite{WA} and the fact that our forest satisfies Specification~\ref{forest_spec} and under our set of assumptions, we have, that there exists a sequence $\sigma_n$, such that:
\begin{equation}
\sigma_n^{-1} V \rightarrow \normal(0,1)
\end{equation}
for $\sigma_n = \Theta\bp{\sqrt{\polylog(n/s)^{-1}\, s / n}}$. More formally, we check that each requirement of Theorem~1 of \cite{WA} is satisfed: 
\begin{enumerate}
\item[(i)] We assume that the distribution of $X$ admits a density that is bounded away from zero and infinity,
\item[(ii)] $\E[Y | X=x^*] = 0$ and hence is continuous in $x^*$ for any $x^*$, 
\item[(iii)] The variance of the $Y$ conditional on $X=x^*$ for some $x^*$ is:
\begin{align*}
\Var(Y|X=x^*) = \E\bb{\ldot{\beta}{M^{-1}\, \psi(Z; \theta_0, \tilde{h}_0(X, W))}^2\mid X=x^*} - \E\bb{\ldot{\beta}{M^{-1}\, \psi(Z; \theta_0, \tilde{h}_0(X, W)) \mid X=x^*}}^2
\end{align*}
The second term is $O(1)$-Lipschitz in $x^*$ by Lipschitzness of $m(x^*; \theta_0, \tilde{h}_0)=\E\bb{\psi(Z; \theta_0, \tilde{h}_0(X, W)) \mid X=x^*}$. For simplicity of notation consider the random variable $V = \psi(Z; \theta_0, \tilde{h}_0(X, W))$. Then the first part is equal to some linear combination of the covariance terms:
\begin{align*}
Q(x^*)\triangleq \E\bb{\beta^\intercal M^{-1} V V^T (M^{-1})^{\intercal} \beta \mid X=x^*} = \beta^\intercal M^{-1} \E\bb{V V^T\mid X=x^*} (M^{-1})^{\intercal} \beta  = 
\end{align*}
Thus by Lipschitzness of the covariance matrix of $\psi$, we have that: $\|\E\bb{V V^{\intercal} \mid X=x^*} - \E\bb{V V^{\intercal} \mid X=\tilde{x}}\|_{F} \leq L \|x^*-\tilde{x}\|$ and therefore by the Cauchy-Schwarz inequality and the lower bound $\sigma >0$ on the eigenvalues of $M$:
\begin{align*}
|Q(x^*) - Q(\tilde{x})| \leq L \|x^* - \tilde{x}\| \|\beta^\intercal M^{-1}\|^2 \leq \frac{L}{\sigma^2} \|x^* - \tilde{x}\| 
\end{align*}
Thus $\Var(Y | X=x^*)$ is $O(1)$-Lipschitz continuous in $x^*$ and hence also $\E[Y^2 | X=x^*]$ is $O(1)$-Lipschitz continuous.
\item[(iv)] The fact that  $\E[ |Y- \E[Y | X=x]|^{2+\delta} | X=x] \leq H$ for some constant $\delta, H$ follows by our assumption on the boundedness of $\psi$ and the lower bound on the eigenvalues of $M$,
\item[(v)] The fact that  $\Var[ Y | X=x']  > 0$ follows from the fact that $\Var\bp{\beta^{\intercal} M^{-1} \psi(Z; \theta_0, \tilde{h}_0(X, W)) \mid X=x'} > 0$,
\item[(vi)] The fact that tree is honest, $\alpha$-balanced with $\alpha\leq 0.2$ and symmetric follows by Specification~\ref{forest_spec}, 
\item[(vii)] From our assumption on $s$ that $s^{-1/(2\alpha d)} = o((s/n)^{1/2-\epsilon})$, it follows that $s = \Theta(n^{\beta})$ for some $\beta\in \left(1- \frac{1}{1+\alpha d}, 1\right]$.
\end{enumerate}

Since, by Lemmas~\ref{lem:taylor}, \ref{lem:mae_decomposition} and Equation~\eqref{eqn:stoch_eq_decomposition}:
\begin{equation*}
\bn{\ldot{\beta}{\hat{\theta}-\theta_0} - V } = O\left(\|\Gamma(\hat{\theta},\hat{h})\| + \|\Delta(\hat{\theta},\hat{h}) - \Delta(\theta_0, \tilde{h}_0)\| + \|\cE(\hat{h})\|^2 + \|\hat{\theta}-\theta_0\|^2\right)
\end{equation*}
it suffices to show that:
\begin{align*}
\sigma_n^{-1} \E\bb{\|\Gamma(\hat{\theta},\hat{h})\| + \|\Delta(\hat{\theta},\hat{h}) - \Delta(\theta_0, \tilde{h}_0)\| + \|\cE(\hat{h})\|^2 + \|\hat{\theta}-\theta_0\|^2} \rightarrow 0
\end{align*}
as then by Slutzky's theorem we have that $\sigma_n^{-1} \ldot{\beta}{\hat{\theta}-\theta_0} \rightarrow_d \normal(0,1)$.
The first term is of order $O(s^{-1/(2\alpha d)})$, hence by our assumption on the choice of $s$, it is $o(\sigma_n)$. The third term is $O(\chi_{n,2}^2)=O(\chi_{n,4}^2)$, which by assumption is also $o(\sigma_n)$. The final term, by applying our $L^q$ estimation error result for $q=2$ and the assumption on our choice of $s$, we get that it is of order $O\bp{\frac{s \log(n/s)}{n}}=o(\sigma_n)$.

Thus it remains to bound the second term. For that we will invoke the stochastic equicontinuity Lemma~\ref{lem:stoch_eq}. Observe that each coordinate $j$ of the term corresponds to the deviation from its mean of a $U$ statistic with respect to the class of functions:
\begin{equation}
\gamma_j(\cdot; \theta, \hat{h}) = f_j(\cdot; \theta, \hat{h}) - f_j(\cdot; \theta_0, h_0)
\end{equation}

Observe that by Lipschitzness of $\psi$ with respect to $\theta$ and the output of $h$ and the locally parametric form of $h$, we have that:
\begin{align*}
|\gamma_{j}(Z_{1:s};\theta, h)| =~& \ba{\E_{\omega}\bb{\sum_{t = 1}^s \alpha_{t}\bp{\{Z_{t}\}_{t=1}^s, \omega}\, \bp{\psi_j(Z_{t}; \theta, \hat{h}(W_{t})) - \psi_j(Z_{t}; \theta, \tilde{h}_0(W_{t}))}}}\\
\leq~& \E_{\omega}\bb{\sum_{t = 1}^s \alpha_{t}\bp{\{Z_{t}\}_{t=1}^s, \omega}\, \ba{\psi_j(Z_{t}; \theta, \hat{h}(W_{t})) - \psi_j(Z_{t}; \theta_0, \tilde{h}_0(W_{t}))}}\\
\leq~& L\, \E_{\omega}\bb{\sum_{t = 1}^s \alpha_{t}\bp{\{Z_{t}\}_{t=1}^s, \omega}\, \bp{\|\theta - \theta_0\| + \|g(W_{t}; \nu) - g(W_{t}; \nu_0(x))\|}}\\
\leq~& L\, \E_{\omega}\bb{\sum_{t = 1}^s \alpha_{t}\bp{\{Z_{t}\}_{t=1}^s, \omega}\, \bp{\|\theta - \theta_0\| + L\, \|\nu -  \nu_0(x)\|}}\\
=~& L\bp{\|\theta - \theta_0\| + L\, \|\nu -  \nu_0(x)\|}
\end{align*}
Thus by Jensen's inequality and the triangle inequality:
\begin{align*}
\sqrt{\E\bb{|\gamma_{j}(Z_{1:s};\theta, h)|^2}} \leq L\|\theta - \theta_0\| + L^2\, \|\nu - \nu_0(x)\| 
\end{align*}
Thus:
\begin{equation*}
\sup_{\theta: \|\theta-\theta_0\|\leq \eta, \|\nu-\nu_0(x)\| \leq \gamma} \sqrt{\E\bb{|\gamma_{j}(Z_{1:s};\theta, g(\cdot; \nu))|^2}} = O(\eta + \gamma)
\end{equation*}

By our $L^q$ error result and Markov's inequality, we have that with probability $1-\delta$: $\|\hat{\theta}-\theta_0\| \leq \eta = O(\sigma_n/\delta)$. Similarly, by our assumption on the nuisance error $\bp{\E\bb{\|\hat{\nu}-\nu_0(x)\|^4}}^{1/4} \leq \chi_{n,4}$ and Markov's inequality we have that with probability $1-\delta$: $\|\hat{\nu}-\nu_0(x)\| \leq O(\chi_{n, 4}/\delta)$.
Thus applying Lemma~\ref{lem:stoch_eq}, we have that conditional on the event that $\|\hat{\nu}-\nu_0(x)\| \leq O(\chi_{n, 4}/\delta)$, w.p. $1-\delta$:
\begin{align*}
\sup_{\theta: \|\theta-\theta_0\|\leq \sigma_n/\delta} \sqrt{\E\bb{|\gamma_{j}(Z_{1:s};\theta, \hat{h})|^2}} =~& O\bp{(\sigma_n/\delta + \chi_{n,4}/\delta) \sqrt{\frac{s(\log(n/s) + \log(1/\delta))}{n}} + \frac{s(\log(n/s) + \log(1/\delta))}{n}}\\
=~& O\bp{\sigma_n^2 \polylog(n/s)/\delta + \chi_{n,4} \sigma_n \polylog(n/s)/\delta + \frac{s(\log(n/s) + \log(1/\delta))}{n}}\\
 =~& O(\sigma_n^{3/2} \polylog(n/s) /\delta)
\end{align*}
where we used the fact that $\chi_{n,4}^2 = o(\sigma_n)$, $\sqrt{\log(1/\delta)}\leq 1/\delta$ and that $\sigma_n =  \Theta\bp{\sqrt{\polylog(n/s)^{-1}\, s / n}}$. By a union bound we have that w.p. $1-3\delta$:
\begin{equation*}
\|\Delta(\hat{\theta},\hat{h}) - \Delta(\theta_0, \tilde{h}_0)\| = O(\sigma_n^{3/2}\polylog(n/s) /\delta)
\end{equation*}
Integrating this tail bound and using the boundedness of the score we get:
\begin{equation}
\E\bb{\|\Delta(\hat{\theta},\hat{h}) - \Delta(\theta_0, \tilde{h}_0)\|} = O(\sigma_n^{3/2} \polylog(n/s) \log(1/\sigma_n)) = o(\sigma_n)
\end{equation}
This completes the proof of the theorem.
\end{proof}

\subsection{Omitted Proofs of Technical Lemmas}

\begin{proof}[Proof of~\Cref{lem:taylor}]
  Fix a conditioning vector $x$. By performing a second order Taylor expansion of each coordinate $j\in [p]$ of the expected score
  function $m_j$ around the true parameters $\theta_0=\theta_0(x)$ and $h_0$ and
  applying the multi-dimensional mean-value theorem, we can write that for any $\theta\in \Theta$:
  \begin{align*}
 m_j(x; \theta, \hat h) &= m_j(x; \theta_0, h_0) + \nabla_\theta  m_j(x; \theta_0, h_0)' (\theta - \theta_0) + D_{\psi_j}[\hat h - h_0 \mid x]\\
    &+ \underbrace{\frac{1}{2}\Ex{}{(\theta - \theta_0, \hat{h}(W)- h_0(W))^\intercal\nabla^2_{\theta, h} \psi_j(Z; \tilde{\theta}^{(j)}, \tilde{h}^{(j)}(W)) (\theta - \theta_0, \hat{h}(W)- h_0(W))\mid x}}_{\rho_j}
  \end{align*}
  where each $\tilde{\theta}^{(j)}$ is some convex combination of $\theta$
  and $\theta_0$ and each $\tilde{h}^{(j)}(W)$ is some convex combination of
  $\hat h(W)$ and $h_0(W)$.  
  Note that $m(x; \theta_0, h_0) = 0$ by definition and
$D_{\psi_j}[\hat h - h_0 \mid x] = 0$ by local orthogonality.  Let $\rho$ denote the vector of second order terms. We can thus write the above set of equations in matrix form as:
\begin{equation*}
M  (\theta - \theta_0)  = m(x; \theta, \hat h) - \rho
\end{equation*}
where we remind that 
$M=\nabla_\theta m(x; \theta_0, h_0)$ is the Jacobian of the moment vector. Since by our assumptions $M$ is invertible
and has eigenvalues bounded away from zero by a constant, we can write:
\[
  ({\theta} - \theta_0) = M^{-1}\, m(x;\theta, \hat h) - M^{-1}\, \rho
\]
Letting $\xi=  - M^{-1}\, \rho$, we have that by the boundedness of the eigenvalues of $M^{-1}$:
\begin{equation*}
\|\xi\| = O(\|\rho\|)
\end{equation*}
By our bounded eigenvalue Hessian assumption on $\E\bb{\nabla^2_{\theta, h} \psi_j(Z; \tilde{\theta}^{(j)}, \tilde{h}^{(j)}(W))\mid x, W}$, we know that:
\[
  \|\rho\|_{\infty} = O\left( \Ex{}{\|\hat h(W) - h_0(W) \|^2\mid x} + \|
  \theta - \theta_0\|^2 \right)
\]
Combining the above two equations and using the fact that $\|\rho\| \leq \sqrt{p} \|\rho\|_{\infty}$, yields that for any $\theta\in \Theta$: 
\begin{equation*}
\theta - \theta_0 = M^{-1}\, \left(m(x; \theta, \hat h) - \Psi(\theta, \hat{h})\right) + \xi
\end{equation*}
Evaluating the latter at $\theta=\hat{\theta}$ and also observing that by the definition of $\hat{\theta}$, $\Psi(\hat{\theta}, \hat{h})=0$ yields the result.
\end{proof}

\begin{proof}[Proof of Lemma~\ref{lem:kernel_error}]
First we argue that by invoking the honesty of the ORF weights we can re-write $\mu_0(\theta, h)$ as:
\begin{align}
\mu_0(\theta, h) = \E\bb{ \binom{n}{s}^{-1} \sum_{1\leq i_1 \leq \ldots \leq i_s \leq n} \E_{\omega}\bb{\sum_{t = 1}^s \alpha_{i_t}\bp{\{Z_{i_t}\}_{t=1}^s, \omega}\, m(X_{i_t}; \theta, h)}}
\end{align}
To prove this claim, it suffices to show that for any subset of $s$ indices:
\begin{align}
\E\bb{\alpha_{i_t}\bp{\{Z_{i_t}\}_{t=1}^s, \omega}\, \psi(Z_{i_t}; \theta, h)} = 
\E\bb{\alpha_{i_t}\bp{\{Z_{i_t}\}_{t=1}^s, \omega}\, m(X_{i_t}; \theta, h)}
\end{align}
By honesty of the ORF weights, we know that either $i_t\in S^1$, in which case $\alpha_{i_t}\bp{\{Z_{i_t}\}_{t=1}^s, \omega}=0$, or otherwise $i_t\in S^2$ and then $\alpha_{i_t}\bp{\{Z_{i_t}\}_{t=1}^s, \omega}$ is independent of $Z_{i_t}$, conditional on $X_{i_t}, Z_{-i_t}, \omega$. Thus in any case $\alpha_{i_t}\bp{\{Z_{i_t}\}_{t=1}^s, \omega}$ is independent of $Z_{i_t}$, conditional on $X_{i_t}, Z_{-i_t}, \omega$. Moreover since $Z_{i_t}$ is independent of $Z_{-i_t}, \omega$ conditional on $X_{i_t}$:
  \[
    \E\bb{\psi(Z_{i_t}; \theta, h) \mid X_{i_t}, Z_{-i_t}, \omega} = 
\E\bb{m(X_{i_t}; \theta, h) \mid X_{i_t}, Z_{-i_t}, \omega}
  \]
 By the law of iterated expectation and the independence properties claimed above, we can write:
  \begin{align*}
\E\bb{\alpha_{i_t}\bp{\{Z_{i_t}\}_{t=1}^s, \omega}\, \psi(Z_{i_t}; \theta, h)} = &= \E\bb{\E\bb{\alpha_{i_t}\bp{\{Z_{i_t}\}_{t=1}^s, \omega}\mid X_{i_t}, Z_{-i_t}, \omega}\, \E\bb{\psi(Z_{i_t}; \theta, h) \mid X_{i_t}, Z_{-i_t}, \omega}}\\
    &= \E\bb{\E\bb{\alpha_{i_t}\bp{\{Z_{i_t}\}_{t=1}^s, \omega}\mid X_{i_t}, Z_{-i_t}, \omega}\, \E\bb{m(X_{i_t}; \theta, h) \mid X_{i_t}, Z_{-i_t}, \omega}}\\
    &= \E\bb{\alpha_{i_t}\bp{\{Z_{i_t}\}_{t=1}^s, \omega}\, m(X_{i_t}; \theta, h)}
  \end{align*}

Finally, by a repeated application of the triangle inequality and the lipschitz property of the conditional moments, we have:
\begin{align*}
\|\Gamma(\theta, h)\| \leq~& \binom{n}{s}^{-1} \sum_{1\leq i_1 \leq \ldots \leq i_s \leq n} \E\bb{\sum_{t = 1}^s \alpha_{i_t}\bp{\{Z_{i_t}\}_{t=1}^s, \omega}\, \|m(x; \theta, h) -m(X_{i_t}; \theta, h)\|}\\
\leq~& \sqrt{p}\, L\, \binom{n}{s}^{-1} \sum_{1\leq i_1 \leq \ldots \leq i_s \leq n} \E\bb{\sum_{t = 1}^s \alpha_{i_t}\bp{\{Z_{i_t}\}_{t=1}^s, \omega}\, \|X_{i_t} - x\|}\\
\leq~& \sqrt{p}\, L\, \binom{n}{s}^{-1} \sum_{1\leq i_1 \leq \ldots \leq i_s \leq n} \E\bb{\sup \{\|X_{i_t} - x\|: \alpha_{i_t}\bp{\{Z_{i_t}\}_{t=1}^s, \omega}>0\}}\\
\leq~& \sqrt{p}\, L\, \epsilon(s)
\end{align*}
\end{proof}

\begin{proof}[Proof of Lemma~\ref{lem:subsampling_error}]
We prove that the concentration holds conditional on the samples $Z_{1:n}$ and $\hat{h}$, the result then follows. Let 
$$\tilde{f}(S_b, \omega_b; \theta, h) = \sum_{i\in S_b} \alpha_i\bp{S_b, \omega_b}\, \psi(Z_i; \theta, h(W_i)).$$
Observe that conditional on $Z_{1:n}$ and $\hat{h}$, the random variables $\tilde{f}(S_1, \omega_1; \theta, h), \ldots, \tilde{f}(S_B, \omega_B; \theta, h) $ are conditionally independent and identically distributed (where the randomness is over the choice of the set $S_b$ and the internal algorithm randomness $\omega_b$). Then observe that we can write $\Psi(\theta, h)= \frac{1}{B} \sum_{b=1}^B \tilde{f}(S_b, \omega_b; \theta, h)$. Thus conditional on $Z_{1:n}$ and $\hat{h}$, $\Psi(\theta, \hat{h})$ is an average of $B$ independent and identically distributed random variables. Moreover, since $S_b$ is drawn uniformly at random among all sub-samples of $[n]$ of size $s$ and since the randomness of the algorithm is drawn identically and independently on each sampled tree:
\begin{equation*}
\E\bb{\tilde{f}(S_b, \omega_b; \theta, \hat{h})  \mid Z_{1:n}} =  \binom{n}{s}^{-1} \sum_{1\leq i_1 \leq \ldots \leq i_s \leq n} f(\{Z_{i_t}\}_{t =1}^s; \theta, \hat{h}) = \Psi_0(\theta, h)
\end{equation*}
Finally, observe that under Assumption~\ref{tech-conditions}, $|\tilde{f}(S_b, \omega_b; \theta, \hat{h})|\leq \psi_{\max}=O(1)$ a.s.. Thus by a Chernoff bound, we have that for any fixed $\theta\in \Theta$, w.p. $1-\delta$:
\begin{equation*}
\| \Psi(\theta, \hat{h}) - \Psi_0(\theta, \hat{h})\| \leq O\bp{\sqrt{\frac{\log(1/\delta)}{B}}}
\end{equation*}
Since $\Theta$ has constant diameter, we can construct an $\epsilon$-cover of $\Theta$ of size $O(1/\epsilon)$. By Lipschitzness of $\psi$ with respect to $\theta$ and following similar arguments as in the proof of Lemma~\ref{lem:stoch_eq}, we can also get a uniform concentration:
\begin{equation*}
\| \Psi(\theta, \hat{h}) - \Psi_0(\theta, \hat{h})\| \leq O\bp{\sqrt{\frac{\log(B) + \log(1/\delta)}{B}}}
\end{equation*}
\end{proof}

\section{Omitted Proofs from \Cref{sec:kernel-lasso}}

\begin{proof}[Proof of \Cref{nuisance-main}]
  By convexity of the loss $\ell$ and the fact that $\hat{\nu}(x)$ is
  the minimizer of the weighted penalized loss, we have:
\begin{align*}
\lambda \left(\|\nu_0(x)\|_1 - \|\hat{\nu}(x)\|_1\right) \geq~& 
\sum_{i=1}^n a_i(x)\, \ell(Z_i;\hat{\nu}(x)) - \sum_{i=1}^n a_i(x)\, \ell(Z_i;\nu_0(x)) \tag{optimality of $\hat{\nu}(x)$}\\
\geq~& \sum_{i=1}^n a_i(x)\ldot{\nabla_{\nu}\ell(z_i;\nu_0(x))}{\hat{\nu}(x) - \nu_0(x)} \tag{convexity of $\ell$}\\
 \geq~& - \left\|\sum_{i} a_i(x) \nabla_\nu \ell(z_i;\nu_0(x))\right\|_{\infty} \|\hat{\nu}(x) - \nu_0(x)\|_1 \tag{Cauchy-Schwarz} \\
 \geq~& - \frac{\lambda}{2}  \|\hat{\nu}(x) - \nu_0(x)\|_1 \tag{assumption on $\lambda$}
\end{align*}
If we let $\rho(x) = \hat{\nu}(x) - \nu_0(x)$, then observe that by the definition of the support $S$ of $\nu_0(x)$ and the triangle inequality, we have: 
\begin{align*}
\|\nu_0(x)\|_1 - \|\hat{\nu}(x)\|_1 & = \|\nu_0(x)_S\|_1 + \|\nu_0(x)_{S^c}\|_1 - \|\hat\nu(x)_S\|_1 - \|\hat\nu(x)_{S^c}\|_1 \tag{separability of $\ell_1$ norm}\\
&= \|\nu_0(x)_S\|_1 - \|\hat\nu(x)_S\|_1 - \|\hat\nu(x)_{S^c}\|_1 \tag{definition of support}\\
&= \|\nu_0(x)_S\|_1 - \|\hat\nu(x)_S\|_1 - \|\hat\nu(x)_{S^c}-\nu_0(x)_{S^c}\|_1 \tag{definition of support}\\
&= \|\nu_0(x)_S - \hat{\nu}(x)_S\|_1 - \|\hat\nu(x)_{S^c}-\nu_0(x)_{S^c}\|_1 \tag{triangle inequality}\\
&\leq \|\rho(x)_S\|_1 - \|\rho(x)_{S^c}\|_1 \tag{definition of $\rho(x)$}
% \|\theta_0(x)\|_1 - \|\hat{\theta}(x)\|_1 \leq \|\rho(x)_S\|_1 - \|\rho(x)_{S^c}\|_1
\end{align*}
%\vscomment{I would remove all the intermediate steps, but added them to address your comment.}
Thus re-arranging the terms in the latter series of inequalities, we get that $\rho(x) \in C(S(x); 3)$.

We now show that the weighted empirical loss function satisfies a conditional restricted strong convexity property with constant $\hat{\gamma}=\gamma - k\sqrt{s\ln(d_\nu/\delta)/n}$ with probability $1-\delta$. This follows from observing that:
\begin{align*}
H = \nabla_{\nu\nu} \sum_{i=1}^n a_i(x)\, \ell(z_i;\nu) = \sum_{i=1}^n a_i \nabla_{\nu\nu}\ell(z_i; \nu)  \succeq \sum_{i=1}^n a_i \mcH(z_i) = \frac{1}{B} \sum_{b} \sum_{i\in b} a_{ib}(x) \mcH(z_i)
\end{align*}
Thus the Hessian is lower bounded by a matrix whose entries correspond to a Monte-Carlo approximation of the $U$-statistic:
\begin{equation}
U = \frac{1}{\binom{n}{s}}\sum_{S\subseteq [n]: |S|=s} \frac{1}{s!} \sum_{i\in S}\Ex{\omega}{ a_i(S, \omega) \mcH(z_i)}
\end{equation}
where $\Pi_s$ denotes the set of permutations of $s$ elements, $S_{\pi}$ denotes the permuted elements of $S$ according to $\pi$ and $S_{\pi}^1, S_{\pi}^2$ denotes the first and second half of the ordered elements of $S$ according to $\pi$. Finally, $a_i(S,\omega)$ denotes the tree weight assigned to point $i$ by a tree learner trained on $S$ under random seed $\omega$.

Hence, for sufficiently large $B$, by a $U$-statistic concentration inequality \cite{Hoeffding} and a union bound, each entry will concentrate around the expected value of the $U$ statistic to within $2\sqrt{s\ln(d_\nu/\delta)/n}$, i.e.: with probability $1-\delta$:
\begin{equation}
\left\|\frac{1}{B} \sum_{b} \sum_{i\in b} a_{ib}(x) \mcH(z_i) - \Ex{}{U}\right\|_{\infty} \leq 2\sqrt{\frac{s\ln(d_\nu/\delta)}{n}}
\end{equation}
Moreover, observe that by the tower law of expectation and by honesty of the ORF trees we can write:
\begin{equation}
\Ex{}{U} = \Ex{}{\frac{1}{\binom{n}{s}}\sum_{S\subseteq [n]: |S|=s} \frac{1}{s!} \sum_{i\in S} a_i(S, \omega) \Ex{}{\mcH(z_i)\mid x_i}}
\end{equation}
Since each $\Ex{}{\mcH(z_i)\mid x_i}$ satisfies the restricted eigenvalue condition with constant $\gamma$, we conclude that $\Ex{}{U}$ also satisfies the same condition as it is a convex combination of these conditional matrices. Thus for any vector $\rho\in C(S(x);3)$, we have w.p. $1-\delta$:
\begin{align*}
\rho^T H \rho \geq~& \rho^T  \left(\sum_{i=1}^n a_i(x) \mcH(z_i)\right) \rho \tag{lower bound on Hessian}\\
 \geq~&  \rho^T \Ex{}{U}\rho - 2\sqrt{\frac{s\ln(d_\nu/\delta)}{n}} \|\rho\|_1^2 \tag{$U$-statistic matrix concentration}\\
 \geq~& \gamma \|\rho\|_2^2 - 2\sqrt{\frac{s\ln(d_\nu/\delta)}{n}} \|\rho\|_1^2 \tag{restricted strong convexity of population}\\
 \geq~& \left(\gamma - 32k\sqrt{\frac{s\ln(d_\nu/\delta)}{n}}\right) \|\rho\|_2^2 \tag{$\rho\in C(S(x); 3)$ and sparsity, imply: $\|\rho\|_1\leq 4\sqrt{k} \|\rho\|_2$}
\end{align*}

Since $\rho(x)\in C(S(x);3)$ and since the weighted empirical loss satisfies a $\hat{\gamma}$ restricted strong convexity:
\begin{align*}
\sum_{i=1}^n a_i(x)\, \ell(z_i;\hat{\nu}(x)) - \sum_{i=1}^n a_i(x)\, \ell(z_i;\nu_0(x))
\geq~&  \sum_{i=1}^n a_i(x)  \ldot{\nabla_{\nu}\ell(z_i;\nu_0(x))}{\hat{\nu} - \nu_0(x)}  
       + \hat{\gamma} \|\rho(x)\|_2^2\\
\geq~& -\frac{\lambda}{2} \|\rho(x)\|_1^2 + \hat{\gamma} \|\rho(x)\|_2^2 \tag{assumption on $\lambda$}
\end{align*}
Combining with the upper bound of $\lambda\left(\|\rho(x)_{S(x)}\|_1 - \|\rho(x)_{S(x)^c}\|_1\right)$ on the difference of the two weighted empirical losses via the chain of inequalities at the beginning of the proof, we get that:
\begin{align*}
\hat{\gamma} \|\rho(x)\|_2^2 \geq \frac{3\lambda}{2} \|\rho(x)_{S(x)}\|_1 - \frac{\lambda}{2}\|\rho(x)_{S(x)^c}\|_1 \leq  \frac{3\lambda}{2} \|\rho(x)_{S(x)}\|_1  \leq  \frac{3\lambda\sqrt{k}}{2} \|\rho(x)_{S(x)}\|_2 \leq \frac{3\lambda\sqrt{k}}{2} \|\rho(x)\|_2
\end{align*}
Dividing both sides by $\|\rho(x)\|_2$ and combining with the fact that $\|\rho(x)\|_1\leq 4\sqrt{k}\|\rho(x)\|_2$ yields the first part of the theorem.

\paragraph{Bounding the gradient.}
Let $\tau = 1/(2\alpha d)$. We first upper bound the expected value of
each entry of the gradient. By the shrinkage property of the ORF
weights:
\begin{align*}
\left|\sum_{i=1}^n \Ex{}{a_i(x)\nabla_{\nu_j}\ell(z_i;\nu_0(x)) \mid x_i}\right|  \leq~& \left|\Ex{}{\nabla_{\nu_j}\ell(z;\nu_0(x)) \mid x}\right| + \Ex{}{\sum_{i=1}^n a_i(x) \left|\Ex{}{\nabla_{\nu_j}\ell(z_i;\nu_0(x)) \mid x_i}- \Ex{}{\nabla_{\nu_j}\ell(z;\nu_0(x)) \mid x}\right|}\\
\leq~& \left|\nabla_{\nu_j} L(\nu_0(x); x)\right| + L\Ex{}{\sum_{i=1}^n a_i(x) \left\|x_i - x\right\|} \tag{Lipschitzness of $\nabla_{\nu}L(\nu; x)$}\\
\leq~& \left|\nabla_{\nu_j} L(\nu_0(x); x)\right| + L\,s^{-\tau} \tag{Kernel shrinkage}\\
\leq~& L\,s^{-\tau} \tag{First order optimality condition of $\nu_0(x)$}
\end{align*}
Moreover, since the quantity $\sum_{i} a_i(x) \nabla_{\nu_j} \ell(z_i;\nu_0(x))$ is also a Monte-Carlo approximation to an appropriately defined $U$-statistic (defined analogous to quantity $U$), for sufficiently large $B$, it will concentrate around its expectation to within $\sqrt{s \ln(1/\delta)/n}$, w.p. $1-\delta$. Since the absolute value of its expectation is at most $Ls^{-\tau}$, we get that the absolute value of each entry w.p. $1-\delta$ is at most $L s^{-\tau} + \sqrt{s \ln(1/\delta)/n}$. Thus with a union bound over the $p$ entries of the gradient, we get that uniformly, w.p. $1-\delta$ all entries have absolute values bounded within $L s^{-\tau} + \sqrt{s \ln(d_\nu/\delta)/n}$.
\end{proof}

\section{Omitted Proofs from Heterogeneous Treatment Effects
  Estimation}

We now verify the moment conditions for our CATE estimation satisfies
the required conditions in \Cref{tech-conditions}.

\subsection{Local Orthogonality}
Recall that for any observation $Z = (T, Y, W, X)$, any parameters
$\theta\in \RR^p$, nuisance estimate $\hat h$ parameterized by
functions $q, g$, we first consider the following \emph{residualized}
score function for PLR is defined as:
\begin{align}\label{eq:PLR}
  \psi(Z; \theta,  h(X, W)) = \left\{Y -  q(X, W) -  \langle \theta,\left(T - g(X, W)\right) \rangle \right\} (T - g(X, W)),
\end{align}
with
$h(X, W)= (q(X, W), g(X, W))$.

For discrete treatments, we also consider the following \emph{doubly
  robust} score function, with each coordinate indexed by treament $t$
defined as:
\begin{align}\label{eq:dr}
  \psi^t(Z;\theta, h(X, W)) = m^t(X, W) + \frac{\left(Y - m^t(X, W)\right) \mathbf{1}[T = t]}{g^t(X,
    W)} - m^0(X, W) - \frac{\left(Y - m^0(X, W)\right) \mathbf{1}[T = 0]}{g^0(X,
    W)} - \theta^t
\end{align}
where $h(X, W) = (m(X, W), g(X, W))$.

\begin{lemma}[Local orthogonality for residualized
  moments]\label{cond-orthog}
  The moment condition with respect to the score function $\psi$
  defined in \eqref{eq:PLR} satisfies conditional orthogonality.
\end{lemma}

\begin{proof}
  We establish local orthogonality via an even stronger
  \emph{conditional orthogonality}:
\begin{equation}
  \Ex{}{\nabla_{h} \psi(Z, \theta_0(x), h_0(X, W)) \mid W, x} = 0
\end{equation}

In the following, we will write $\nabla_h \psi$ to denote the gradient
of $\psi$ with respect to the nuisance argument. For any $W, x$, we
can write
\begin{align*}
  \Ex{}{\nabla_{h} \psi\left(Z; \theta_0(x), \left(q_0(x, W) ,  g_0(x, W)\right) \right) \mid W, x} = \Ex{}{\left( T - g_0(x, W), - Y +  q_0(x, W) + 2\theta_0(x)^\intercal \, (T - g_0(x, W)) \right) \mid W, x}
\end{align*}
Furthermore, we have
$\Ex{}{T - g_0(x, W) \mid W, x} = \Ex{}{\eta \mid W,
  x} = 0$ and
\begin{align*}
  \Ex{}{-Y + q_0(x, W) + 2 \theta_0(x)^\intercal\, (T - g_0(x, W))\mid W, x} &=\Ex{}{q_0(x, W)-Y  + 2\theta(x)^\intercal \eta\mid W, x} = 0
\end{align*}
where the last equality follows from that $\Ex{}{\eta \mid W, x} = 0$
and $\Ex{}{\langle W, q_0\rangle - Y\mid W, x} = 0$. 
\end{proof}

\begin{lemma}[Local orthogonality for doubly robust
  moments]\label{cond-orthog}
  The moment condition with respect to the score function $\psi$
  defined in \eqref{eq:dr} satisfies conditional orthogonality.
\end{lemma}

\begin{proof}
  For every coordinate (or treatment) $t$, we have
  \begin{align*}
    &~~~\Ex{}{\nabla_{g} \psi^t\left(Z; \theta_0(x), \left(m_0(x, W) ,
      g_0(x, W)\right) \right) \mid W, x} \\
    &= \Ex{}{-\frac{(Y - m_0^t(X, W))\mathbf{1}[T=t]}{(g_0^t(x, W))^2} + \frac{(Y - m_0^0(X, W))\mathbf{1}[T=t]}{(g_0^0(x, W))^2}\mid W, x}\\
    &= \Ex{}{-\frac{(Y - m_0^t(X, W))}{(g_0^t(x, W))^2}\mid W, x, T = t} \Pr[T = t\mid W, x]\\
    &+\Ex{}{\frac{(Y - m_0^0(X, W))}{(g_0^0(x, W))^2}\mid W, x, T = 0} \Pr[T = 0\mid W, x] = 0
  \end{align*}
and
  \begin{align*}
    &~~~\Ex{}{\nabla_{m} \psi^t\left(Z; \theta_0(x), \left(m_0(x, W) ,
      g_0(x, W)\right) \right) \mid W, x} \\
    &= \Ex{}{\nabla_{m}\left(m_0^t(x, W) + \frac{\left(Y - m_0^t(x, W)\right) \mathbf{1}[T = t]}{g_0^t(x,
      W)} - m_0^0(x, W) -  \frac{\left(Y - m_0^0(x, W)\right) \mathbf{1}[T = 0]}{g_0^0(x,
      W)}  \right)\mid W, x} \\
    &= \Ex{}{\nabla_{m}\left(m_0^t(x, W) + \frac{\left( - m_0^t(x, W)\right) \mathbf{1}[T = t]}{g_0^t(x,
      W)} - m_0^0(x, W) -  \frac{\left( - m_0^0(x, W)\right) \mathbf{1}[T = 0]}{g_0^0(x,
      W)} \right)\mid W, x} \\
    &=\nabla_{m}\left( \Ex{}{m_0^t(x, W)  - m_0^t(x, W) - m_0^0(x, W) + m_0^0(x, W)\mid W, x}\right) = 0
  \end{align*}
This complets the proof.
\end{proof}

\subsection{Identifiability}

\begin{lemma}[Identifiability for residualized moments.]
  As long as $\mu(X, W)$ is independent of $\eta$ conditioned on $X$
  and the matrix $\Ex{}{\eta \eta^\intercal \mid X=x}$ is invertible
  for any $x$, the parameter $\theta(x)$ is the unique solution to
  $m(x; \theta, h) = 0$.
\end{lemma}

\begin{proof}
  The moment conditions $m(x; \theta, h) = 0$ can be written as
  \begin{align*}
    \Ex{}{\left\{Y - q_0(X, W) -  \theta^\intercal\left(T - g_0(X, W)\right)  \right\} (T - g_0
    (X, W)) \mid X = x} = 0  \end{align*}
The left hand side can re-written as
\begin{align*}
  \Ex{}{\left\{\ldot{\eta}{\mu_0(X, W)} + \eps - \ldot{\theta}{ \eta})
  \right\} \eta\mid X = x} &= \Ex{}{\left\{\ldot{\eta}{\mu_0(X, W)} -
                             \ldot{\theta}{\eta}) \right\} \eta\mid X = x}\\
                           &= \Ex{}{\ldot{\mu_0(X, W) -
                             \theta}{\eta}) \eta\mid X = x}\\
                           &= \Ex{}{\eta \eta^\intercal\mid X = x} \Ex{}{\mu_0(X, W) -
                             \theta\mid X = x}
\end{align*}
Since the conditional expected covariance matrix
$\Ex{}{\eta \eta^\intercal\mid X = x}$ is invertible, the expression
above equals to zero only if
$\Ex{}{\mu_0(X, W) - \theta \mid X = x} = 0$. This implies that
$\theta = \Ex{}{\mu_0(X, W)\mid X=x} = \theta_0(x)$.
\end{proof}

\begin{lemma}[Identifiability for doubly robust moments.]
  As long as $\mu(X, W)$ is independent of $\eta$ conditioned on $X$
  and the matrix $\Ex{}{\eta \eta^\intercal \mid X=x}$ is invertible
  for any $x$, the parameter $\theta(x)$ is the unique solution to
  $m(x; \theta, h) = 0$.
\end{lemma}

\begin{proof}
  For each coordinate $t$, the moment condition can be written as
\begin{align}\label{eq:dr}
\Ex{}{m_0^t(X, W) + \frac{\left(Y - m_0^t(X, W)\right) \mathbf{1}[T = t]}{g_0^t(X,
    W)} - m_0^0(X, W) - \frac{\left(Y - m_0^0(X, W)\right) \mathbf{1}[T = 0]}{g_0^0(X,
    W)} - \theta^t \mid X = x} = 0
\end{align}
Equivalently,
\begin{align*}\label{eq:dr}
  \Ex{}{m_0^t(X, W)  - m_0^0(X, W)  - \theta^t \mid X = x} = \Ex{W}{\Ex{}{- \frac{\left(Y - m_0^t(X, W)\right) \mathbf{1}[T = t]}{g_0^t(X,
  W)} + \frac{\left(Y - m_0^0(X, W)\right) \mathbf{1}[T = 0]}{g_0^0(X,
  W)}  \mid W, X = x}}
\end{align*}
The inner expectation of the right hand side can be written as:
\[
  \Ex{}{- \frac{\left(Y - m_0^t(X, W)\right) \mathbf{1}[T =
      t]}{g_0^t(X, W)} + \frac{\left(Y - m_0^0(X, W)\right)
      \mathbf{1}[T = 0]}{g_0^0(X, W)} \mid W, X = x} = 0
\]
This means the moment
condition is equivalent to
\[
  \Ex{}{m_0^t(X, W) - m_0^0(X, W) \mid X = x} = \theta^t.
\]
This completes the proof.
\end{proof}

\subsection{Smooth Signal}
Now we show that the moments $m(x;\theta, h)$ are $O(1)$-Lipschitz in
$x$ for any $\theta$ and $h$ under standard boundedness conditions on
the parameters.

First, we consider the residualized moment function is defined as
  \[
    \psi(Z; \theta, h(X, W)) = \left\{Y - q(X, W) - \theta^\intercal\left(T -
        g(X, W)\right) \right\} (T - g(X, W)),
  \]
  Then for any $\theta$ and $h$ given by functions $g$ and $q$,
  \[
    m(x; \theta, h) = \Ex{}{\left\{Y - q(x, W) - \theta^\intercal\left(T - g(x,
          W)\right) \right\} (T - g(x, W)) \mid X = x}
  \]

  \paragraph{Real-valued treatments} In the real-valued treatment
  case, each coordinate $j$ of $g$ is given by a high-dimensional
  linear function: $g^j(x, W) = \langle W, \gamma^j\rangle$, where
  $\gamma^j$ is a $k$-sparse vectors in $\mathbb{R}^{d_\nu}$
  with$\ell_1$ norm bounded by a constant, and $q(x, W)$ can be
  written as a $\langle q', \phi_2(W)\rangle$ with $q'$ is a
  $k^2$-sparse vector in $\mathbb{R}^{d_\nu^2}$ and $\phi_2(W)$
  denotes the degree-2 polynomial feature vector of $W$.
  \[
    m_j(x; \theta, h) = \Ex{}{\left\{Y - \langle q', \phi_2(W)\rangle
        - \theta_j\left(T - \langle \gamma^j, W\rangle\right) \right\}
      (T - g(x, W)) \mid X = x}
  \]
  Note that as long as we restrict the space $\Theta$ and $H$ to
  satisfy $\|\theta\| \leq O(1)$, $\|\gamma\|_1, \|q'\|_1 \leq 1$, we
  know each coordinate $m_j$ is smooth in $x$.

  \paragraph{Discrete treatments with residualized moments.} In the
  discrete treatment case, each coordinate $j$ of $g$ is of the form
  $g^j(x, W) = \cL(\langle W, \gamma^j\rangle)$, where
  $\cL(t) = 1/(1 + e^{-t})$ is the logistic function.  The estimate
  $q$ consists of several components. First, consider function $f$ of
  the form $f(x, W) = \langle W, \beta \rangle$ as an estimate for the
  outcome of the null treament. For each $t\in \{e_1, \ldots, e_p\}$,
  we also have an estimate $m^t(x, W)$ for the expected
  counter-factual outcome function $\mu^t(x) + f(x, W)$, which takes
  the form of $\langle b , W\rangle$. Then the estimate $q$ is defined
  as:
\begin{equation*}
  q(x, W) = \sum_{t=1}^p (m^t(x, W) - f(x, W)) g^t(x, W) + f(x, W).
\end{equation*}
With similar reasoning, as long as we restrict $\Theta$ and $H$ to
satisfy $|\theta| \leq O(1)$, $\|\gamma^j\|_1 \leq 1$ for all $j$, and
$\|\beta\|_1, \|b\|_1 \leq 1$, we know each coordinate $m_j$ is smooth
in $x$.

\paragraph{Discrete treatments with doubly robust moments.}

Redcall that for each coordinate $t$, the moment function with input
$\theta$ and nuisance parameters $m, g$ is defined as
\begin{align}\label{eq:dr}
  \Ex{}{m^t(x, W) + \frac{\left(Y - m^t(x, W)\right) \mathbf{1}[T = t]}{g^t(x,
  W)} - m^0(x, W) - \frac{\left(Y - m^0(x, W)\right) \mathbf{1}[T = 0]}{g^0(x,
  W)} - \theta^t \mid X = x}
\end{align}
where each $m^t(x, W)$ takes the form of $\langle b, W \rangle$ and
each $g^t(x, W) = \cL(\langle W, \gamma^t\rangle)$, with $\cL$
denoting the logistic function. Then as long as we restrict the
parameter space and $H$ to satisfy $\|\gamma^t\|_1$ for all $t$, then
we know that $|\langle \gamma^t , W\rangle| \leq O(1)$ and so
$g^t(x, W)\geq \Omega(1)$. Furthermore, if we restrict the vector $b$
to satisfy $\|b\|_1 \leq 1$, we know each coordinate $m_j$ is smooth
in $x$.

\subsection{Curvature}
Now we show that the jacobian $\nabla_\theta m(x;\theta_0(x), h_0)$
has minimum eigenvalues bounded away from 0. 

\paragraph{Residualized moments.}
First, we consider the residualized moment function is defined as
  \[
    \psi(Z; \theta, h(X, W)) = \left\{Y - q(X, W) - \theta^\intercal\left(T -
        g(X, W)\right) \right\} (T - g(X, W)),
  \]
  Then for any $\theta$ and $h$ given by functions $g$ and $q$,
  \[
    m(x; \theta, h) = \Ex{}{\left\{Y - q(x, W) - \theta^\intercal\left(T - g(x,
          W)\right) \right\} (T - g(x, W)) \mid X = x}
  \]

  Let $J$ be the expected Jacobian
  $\nabla_\theta m(x; \theta_0(x), h_0)$, and we can write
  $$J_{jj'}= \Ex{}{(T_j - g_0^j(x, W))(T_{j'} - g_0^{j'}(x, W))\mid
    X=x}$$ Then for any $v\in \mathbb{R}^p$ with unit $\ell_2$ norm,
  we have
  \begin{align*}
    v J v^\intercal& = \Ex{}{\sum_j (T_j- g_0^j(x, W))^2 v_j^2 + 2\sum_{j, j'}(T_j - g_0^j(x, W))(T_{j'} - g_0^{j'}(x, W)) v_j v_{j'} \mid X=x}\\
                   &= \Ex{}{\left(\sum_j (T_j - g_0^j(x, W)) v_j \right)^2\mid X=x}\\
                   &=  \Ex{}{v^\intercal (\eta \eta^\intercal) v\mid X=x}
                   % &\geq \mathrm{Var}\left[ \left(\sum_j \eta_j v_j \right)\mid X=x\right]                     \\
                   % &= \mathrm{Var}\left[ \langle \eta, v \rangle\mid X=x\right]
  \end{align*}
  Then as long as the conditional expected covariance matrix
  $\Ex{}{ \eta \eta^\intercal \mid X=x}$ has minimum eigenvalue
  bounded away from zero, we will also have $\min_{v} v J v^\intercal$
  bounded away from zero.

\paragraph{Discrete treatments with doubly robust moments.}

Redcall that for each coordinate $t$, the moment function with input
$\theta$ and nuisance parameters $m, g$ is defined as
\begin{align}\label{eq:dr}
  \Ex{}{m^t(x, W) + \frac{\left(Y - m^t(x, W)\right) \mathbf{1}[T = t]}{g^t(x,
  W)} - m^0(x, W) - \frac{\left(Y - m^0(x, W)\right) \mathbf{1}[T = 0]}{g^0(x,
  W)} - \theta^t \mid X = x}
\end{align}
Then $\nabla_\theta m(x; \theta_0(x), h_0) = - I$, which implies the
minimum eigevalue is 1.

\subsection{Smoothness of scores}

\paragraph{Residualized moments.}
First, we consider the residualized moment function with each
coordinate defined as
  \[
    \psi_j(Z; \theta, h(X, W)) = \left\{Y - q(X, W) -
      \theta^\intercal\left(T - g(X, W)\right) \right\} (T_j - g_j(X,
    W)),
  \]
  Observe that for both real-valued and discrete treatments, the
  scales of $\theta$, $q(X, W)$, and $g(X, W)$ are bounded by
  $O(1)$. Thus, the smoothness condition immediately follows.
  % Thus, for every $j\in [p]$ and for all $\theta$ and $h$, the
  % eigenvalues of the expected Hessian
  % $\Ex{}{\nabla^2_{(\theta, h)} \psi_j(Z; \theta, h(W))\mid x, W}$
  % are
  % bounded above by a constant $O(1)$ because every entry of the
  % Hessian is bounded by $O(1)$. For any $Z$, the score
  % $\psi(Z; \theta, \xi)$ is $O(1)$-Lipschitz in $\theta$ for any
  % $\xi$
  % and $O(1)$-Lipschitz in $\xi$ for any $\theta$. The gradient of
  % the
  % score with respect to $\theta$ is $O(1)$-Lipschitz in $\xi$.

\paragraph{Doubly robust moments.}
For every treatment $t$, 
\[ 
  \psi_t(Z; \theta, h(X, W)) = {m^t(X, W) + \frac{\left(Y - m^t(X,
        W)\right) \mathbf{1}[T = t]}{g^t(X, W)} - m^0(X, W) -
    \frac{\left(Y - m^0(X, W)\right) \mathbf{1}[T = 0]}{g^0(X, W)} -
    \theta^t}
\]
Recall that each $m^t(x, W)$ takes the form of $\langle b, W \rangle$
and each $g^t(x, W) = \cL(\langle W, \gamma^t\rangle)$, with $\cL$
denoting the logistic function.  Then as long as we restrict the
parameter space and $H$ to satisfy $\|\gamma^t\|_1$ for all $t$, then
we know that $|\langle \gamma^t , W\rangle| \leq O(1)$ and so
$g^t(X, W)\geq \Omega(1)$. Furthermore, if we restrict the vector $b$ to
satisfy $\|b\|_1 \leq 1$, we know each $m_j(X, W)\leq
O(1)$. Therefore, the smoothness condition also holds.

\subsection{Accuracy for discrete treatments}

For both score functions, we require that each discrete treatment
(including the null treatment) is assigned with constant probability.

\begin{corollary}[Accuracy for residualized scores]
  Suppose that $\beta_0(X)$ and each coorindate
  $\beta_0(X), \gamma_0^j(X)$ and $\theta(X)$ are Lipschitz in $X$
  and have $\ell_1$ norms bounded by $O(1)$ for any $X$. Assume that
  distribution of $X$ admits a density that is bounded away from zero
  and infinity.  For any feature $X$, the conditional covariance
  matrices satisfy
  $\Ex{}{\eta \eta^\intercal \mid X}\succeq \Omega(1) $,
  $\Ex{}{W W^\intercal \mid X} \succeq \Omega(1)I_{d_\nu}$. Then with
  probability $1 - \delta$, ORF returns an estimator $\hat \theta$
  such that
  \[
    \|\hat \theta - \theta_0\| \leq O\left(n^{\frac{-1}{2 + 2\alpha d}}\sqrt{\log(n d_{\nu}/\delta)} \right)
  \]
  as long as the sparsity
  $k \leq O\left(n^{\frac{1}{4 + 4\alpha d}} \right)$ and the
  subsampling rate of ORF $s = \Theta(n^{\alpha d/(1 + \alpha d)})$.
  Moreover, for any $b \in \RR^p$ with $\|b\|\leq 1$, there exists a
  sequence $\sigma_n = \Theta(\sqrt{\polylog(n) n^{-1/(1+\alpha d)}})$
  such that
  \[
    \sigma_n^{-1} \ldot{b}{\hat\theta -\theta} \rightarrow_d \mathcal{N}(0, 1),
  \]
  as long as the sparsity
  $k=o\left(n^{\frac{1}{4 + 4\alpha d}} \right)$ and the subsampling
  rate of ORF $s = \Theta(n^{\epsilon + \alpha d/(1 + \alpha d)})$ for
  any $\epsilon>0$.
\end{corollary}

\begin{corollary}[Accuracy for doubly robust scores]
  Suppose that $\beta_0(X)$ and each coorindate
  $u_0^j(X), \gamma_0^j(X)$ are Lipschitz in $X$ and have $\ell_1$
  norms bounded by $O(1)$ for any $X$. Assume that distribution of $X$
  admits a density that is bounded away from zero and infinity.  For
  any feature $X$, the conditional covariance matrices satisfy
  $\Ex{}{\eta \eta^\intercal \mid X}\succeq \Omega(1) $,
  $\Ex{}{W W^\intercal \mid X} \succeq \Omega(1)I_{d_\nu}$. Then with
  probability $1 - \delta$, ORF returns an estimator $\hat \theta$
  such that
  \[
    \|\hat \theta - \theta_0\| \leq O\left(n^{\frac{-1}{2 + 2\alpha d}}\sqrt{\log(n d_{\nu}/\delta)} \right)
  \]
  as long as the sparsity
  $k \leq O\left(n^{\frac{1}{4 + 4\alpha d}} \right)$ and the
  subsampling rate of ORF $s = \Theta(n^{\alpha d/(1 + \alpha d)})$.
  Moreover, for any $b \in \RR^p$ with $\|b\|\leq 1$, there exists a
  sequence $\sigma_n = \Theta(\sqrt{\polylog(n) n^{-1/(1+\alpha d)}})$
  such that
  \[
    \sigma_n^{-1} \ldot{b}{\hat\theta -\theta} \rightarrow_d \mathcal{N}(0, 1),
  \]
  as long as the sparsity
  $k=o\left(n^{\frac{1}{4 + 4\alpha d}} \right)$ and the subsampling
  rate of ORF $s = \Theta(n^{\epsilon + \alpha d/(1 + \alpha d)})$ for
  any $\epsilon>0$.
\end{corollary}

\newpage

\section{Orange Juice Experiment}
\label{sec:oj}

Dominick's orange juice dataset (provided by the University of Chicago Booth School of Business) contains 28,947 entries of store-level, weekly prices and sales of different brands of orange juice. The dataset also contains 15 continuous and categorical variables that encode store-level customer information such as the mean age, income, education level, etc, as well as brand information. The goal is to learn the elasticity of orange juice as a function of income (or education, etc) in the presence of high-dimensional controls. 

In the experiment depicted in Figure~\ref{fig:oj}, we trained the ORF using $500$ trees, a minimum leaf size of $50$, subsample ratio of $0.02$, with Lasso models for both residualization and kernel estimation. We evaluated the resulting algorithm on $50$ $\log(Income)$ points between $10.4$ and $10.9$. We then followed-up with 100 experiments on bootstrap samples of the original dataset to build bootstrap confidence intervals. The emerging trend in the elasticity as a function of income follows our intuition: higher income levels correspond to a more inelastic demand.

\section{All Experimental Results}
\label{sec:allexp}

We present all experimental results for the parameter choices
described in Section \ref{sec:mcexp}. We vary the support size $k\in\{1,5,10,15,20,25,30\}$,
the dimension $d\in\{1,2\}$ of the feature vector $x$ and the
treatment response function
$\theta \in\{\text{piecewise linear, piecewise constant and piecewise
  polynomial}\}$. We measure the bias, variance and root mean square
error (RMSE) as evaluation metrics for the different estimators we
considered in Section \ref{sec:mcexp}. In addition, we add another version of the GRF (GRF-xW) where we run the GRF R package directly on the observations, using features and controls $(x, W)$ jointly as the covariates. For the parameter space we consider, the ORF-CV and the ORF
algorithms outperform the other estimators on all regimes.

\subsection{Experimental results for one-dimensional, piecewise linear $\theta_0$}
Consider a piecewise linear function: $
\theta_0(x) = (x+2)\mathbb{I}_{x\leq 0.3} + (6x+0.5)\mathbb{I}_{x>0.3 \text{ and } x \leq 0.6} + (-3x+5.9)\mathbb{I}_{x> 0.6}$.

\begin{figure}[H]
	\begin{minipage}[c]{0.75\textwidth}
		\includegraphics[width=\textwidth]{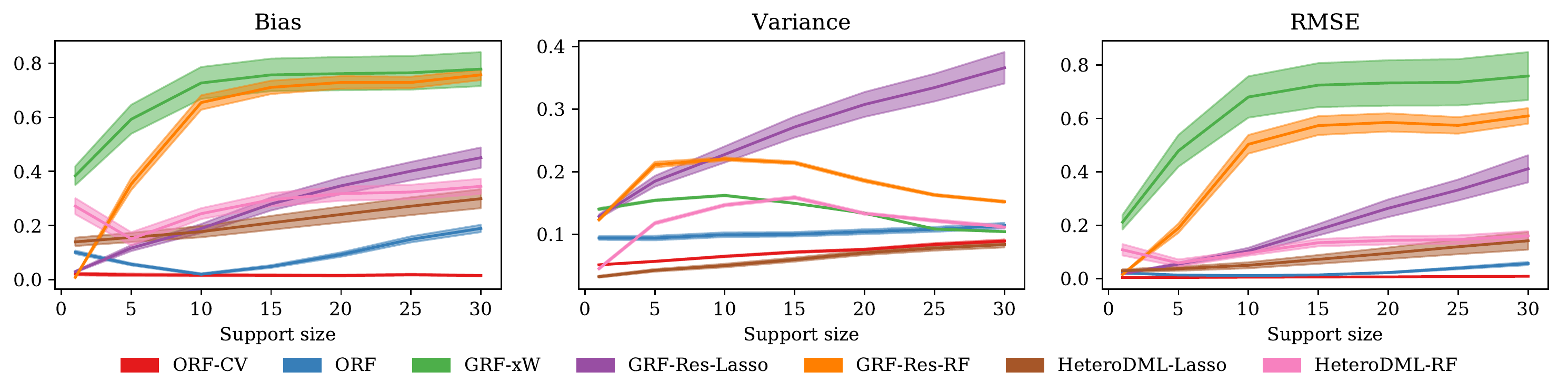}
	\end{minipage}\hfill
	\begin{minipage}[c]{0.23\textwidth}
		\caption{Bias, variance and RMSE as a function of support size for $n=5000$, $p=500$, $d=1$ and $\theta_0$. The solid lines represent the mean of the metrics across test points, averaged over the 100 experiments, and the filled regions depict the standard deviation, scaled down by a factor of 3 for clarity.}
	\end{minipage}
\end{figure}
% support 1
\begin{figure}[H]
	\begin{minipage}[c]{0.75\textwidth}
		\includegraphics[width=\textwidth]{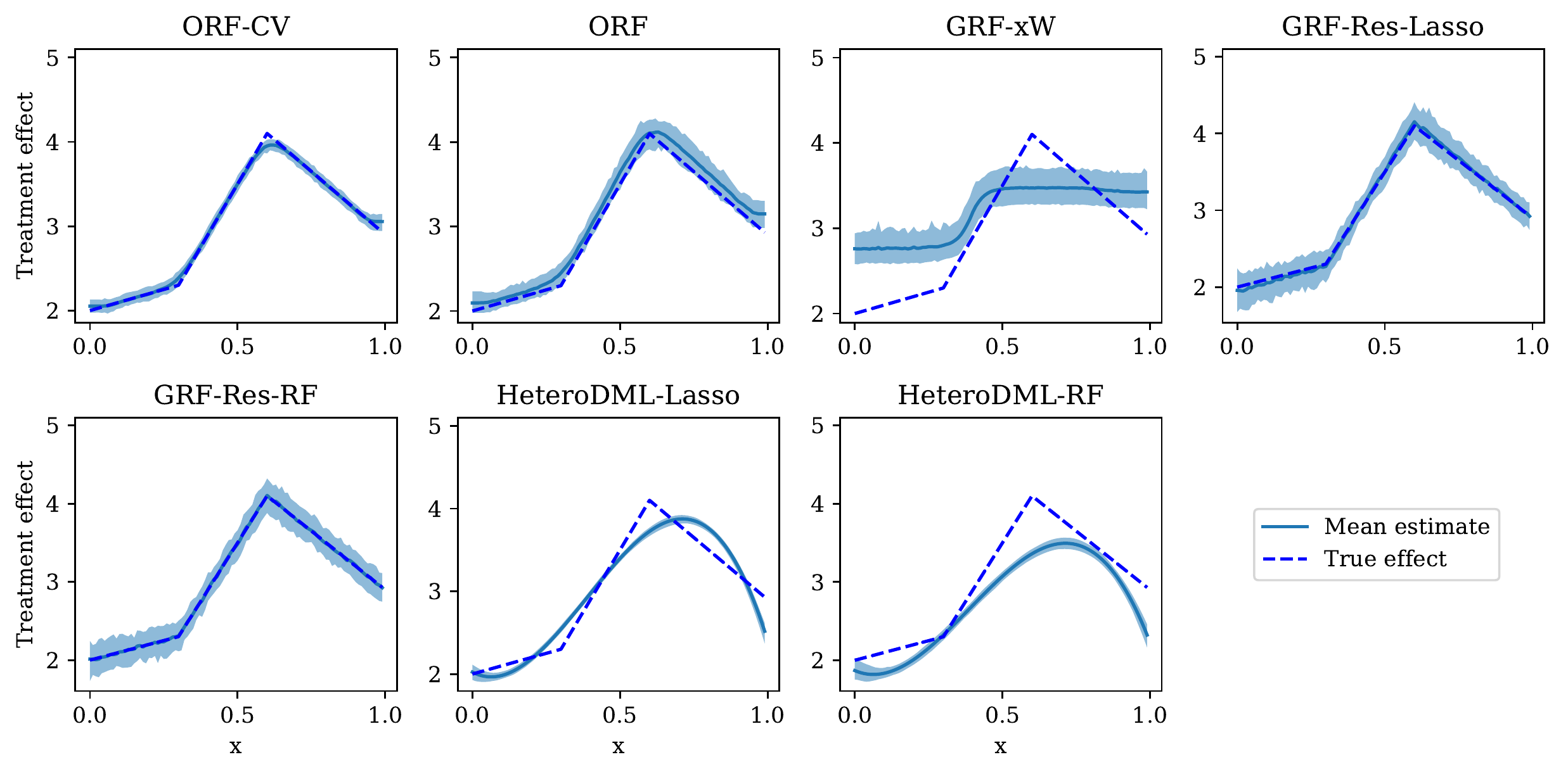}
	\end{minipage}\hfill
	\begin{minipage}[c]{0.23\textwidth}
		\caption{
		Treatment effect estimations for 100 Monte Carlo experiments with parameters $n=5000$, $p=500$, $d=1$, $\mathbf{k=1}$, and a piecewise linear treatment response. The shaded regions depict the mean and the $5\%$-$95\%$ interval of the 100 experiments.
		} 
	\end{minipage}
\end{figure}
% support 5
\begin{figure}[H]
	\begin{minipage}[c]{0.75\textwidth}
		\includegraphics[width=\textwidth]{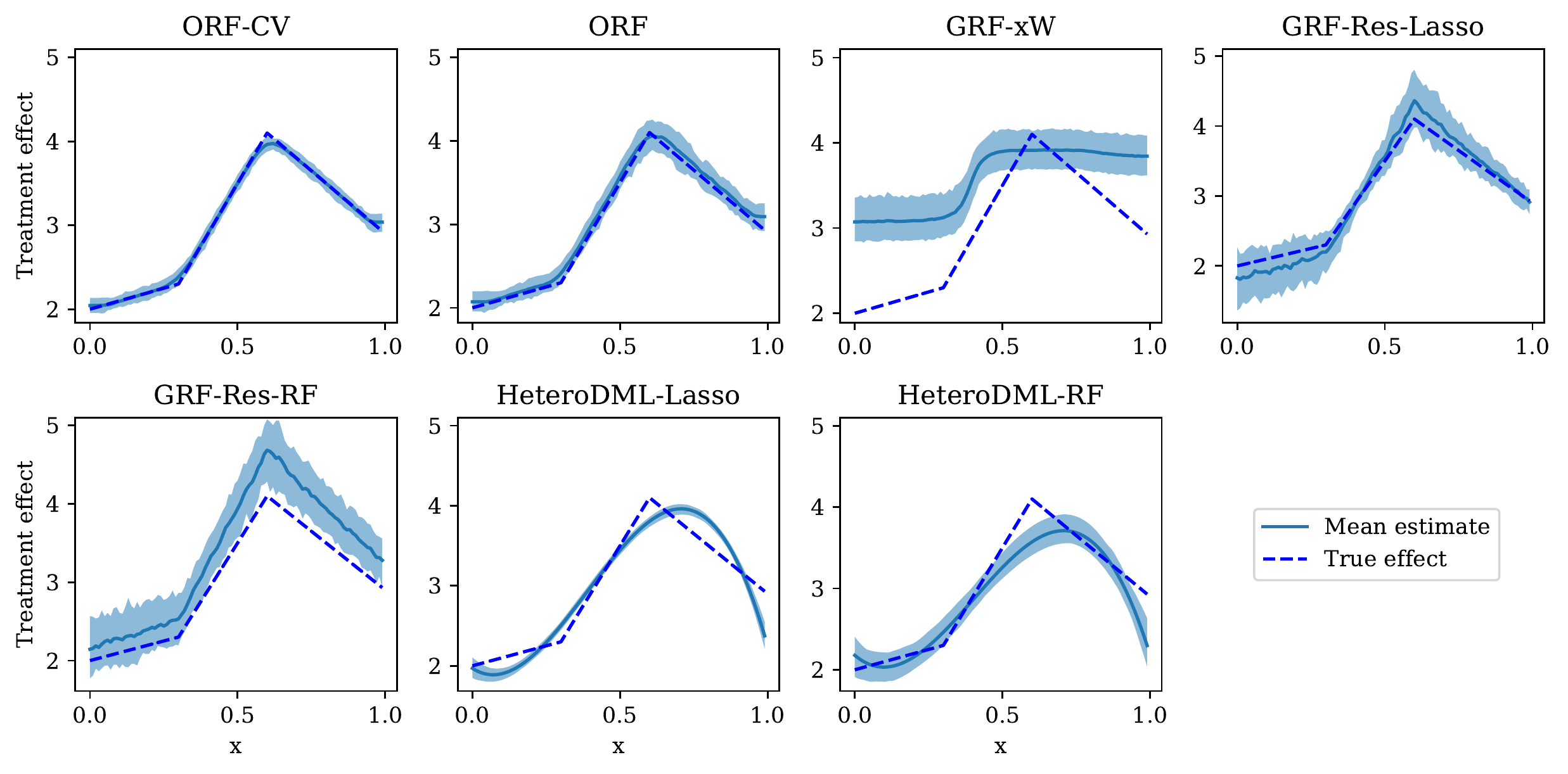}
	\end{minipage}\hfill
	\begin{minipage}[c]{0.23\textwidth}
		\caption{
			Treatment effect estimations for 100 Monte Carlo experiments with parameters $n=5000$, $p=500$, $d=1$, $\mathbf{k=5}$, and a piecewise linear treatment response. The shaded regions depict the mean and the $5\%$-$95\%$ interval of the 100 experiments.
		} 
	\end{minipage}
\end{figure}
% support 10
\begin{figure}[H]
	\begin{minipage}[c]{0.75\textwidth}
		\includegraphics[width=\textwidth]{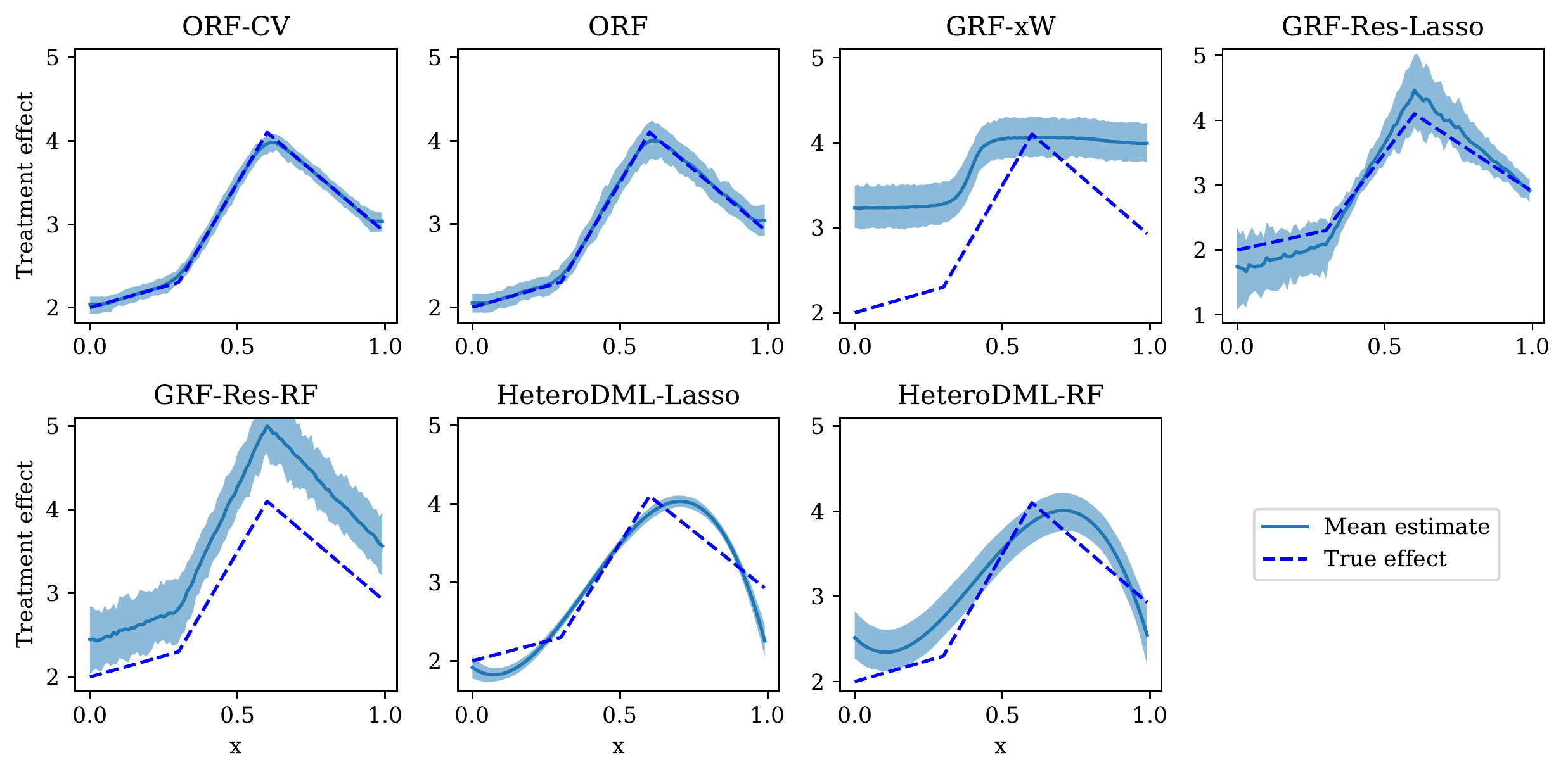}
	\end{minipage}\hfill
	\begin{minipage}[c]{0.23\textwidth}
		\caption{
		Treatment effect estimations for 100 Monte Carlo experiments with parameters $n=5000$, $p=500$, $d=1$, $\mathbf{k=10}$, and a piecewise linear treatment response. The shaded regions depict the mean and the $5\%$-$95\%$ interval of the 100 experiments.
		} 
	\end{minipage}
\end{figure}
% support 15
\begin{figure}[H]
	\begin{minipage}[c]{0.75\textwidth}
		\includegraphics[width=\textwidth]{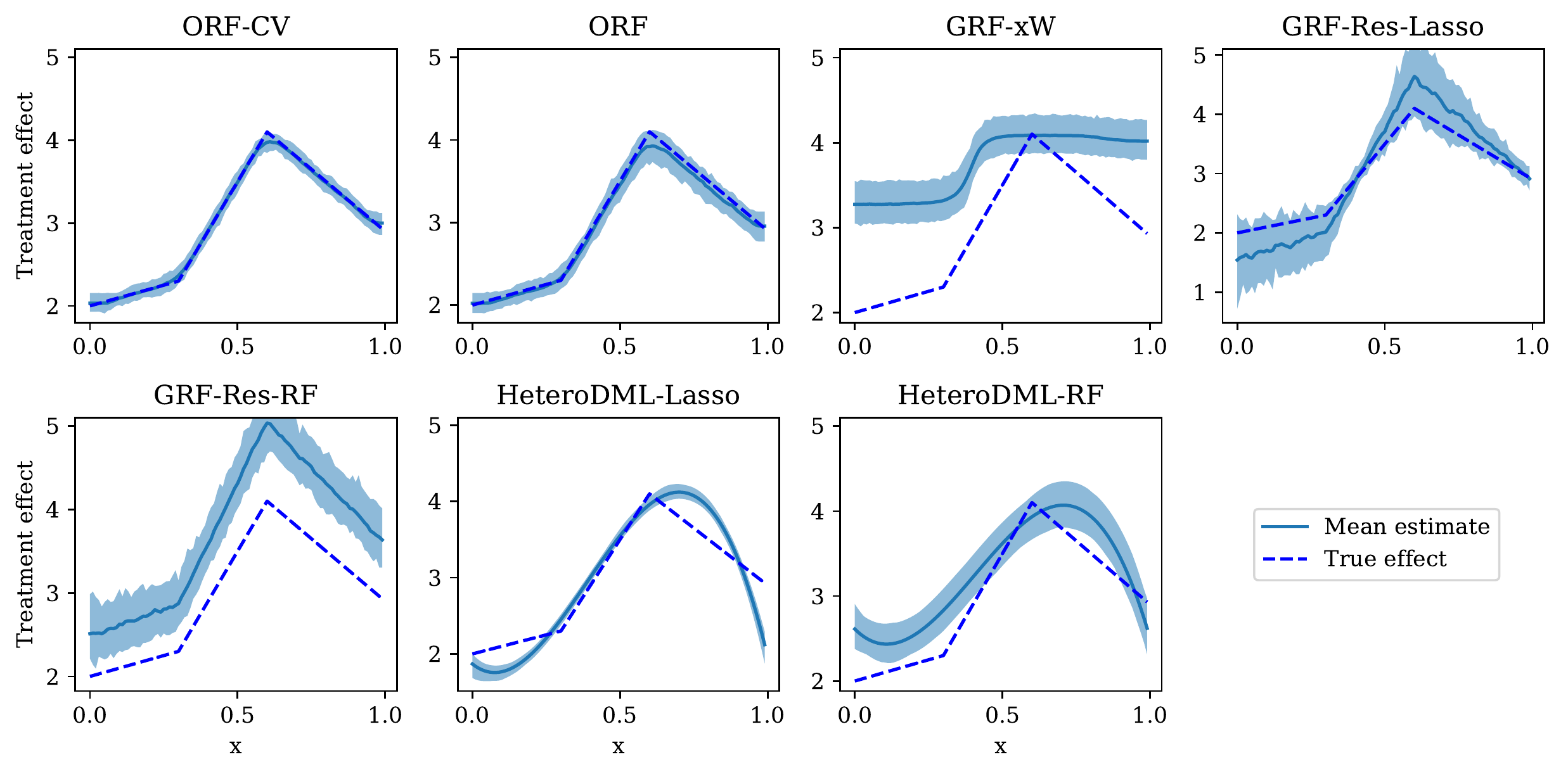}
	\end{minipage}\hfill
	\begin{minipage}[c]{0.23\textwidth}
		\caption{
			Treatment effect estimations for 100 Monte Carlo experiments with parameters $n=5000$, $p=500$, $d=1$, $\mathbf{k=15}$, and a piecewise linear treatment response. The shaded regions depict the mean and the $5\%$-$95\%$ interval of the 100 experiments.
		} 
	\end{minipage}
\end{figure}
% support 20
\begin{figure}[H]
	\begin{minipage}[c]{0.75\textwidth}
		\includegraphics[width=\textwidth]{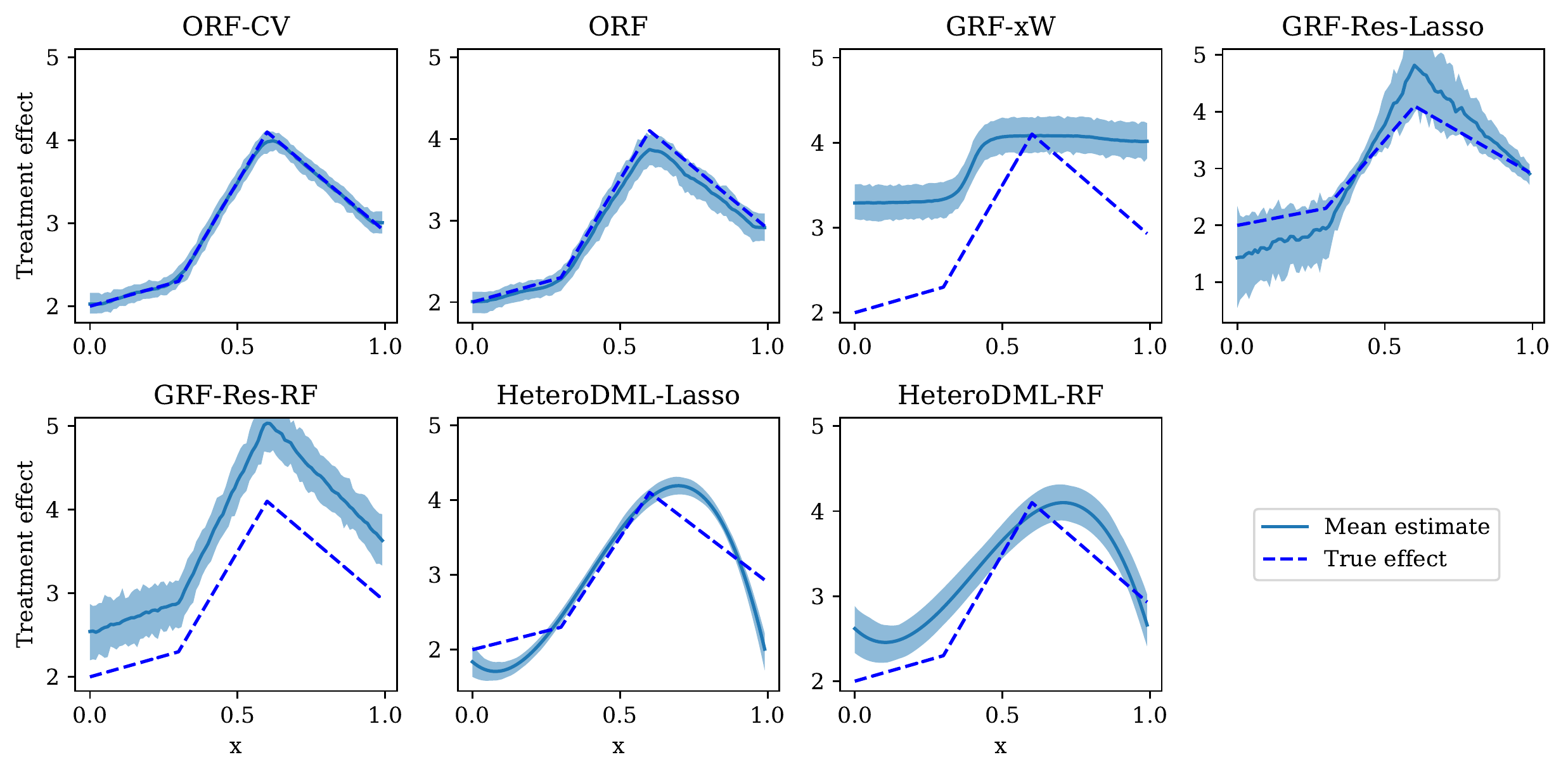}
	\end{minipage}\hfill
	\begin{minipage}[c]{0.23\textwidth}
		\caption{
			Treatment effect estimations for 100 Monte Carlo experiments with parameters $n=5000$, $p=500$, $d=1$, $\mathbf{k=20}$, and a piecewise linear treatment response. The shaded regions depict the mean and the $5\%$-$95\%$ interval of the 100 experiments.
		} 
	\end{minipage}
\end{figure}
% support 25
\begin{figure}[H]
	\begin{minipage}[c]{0.75\textwidth}
		\includegraphics[width=\textwidth]{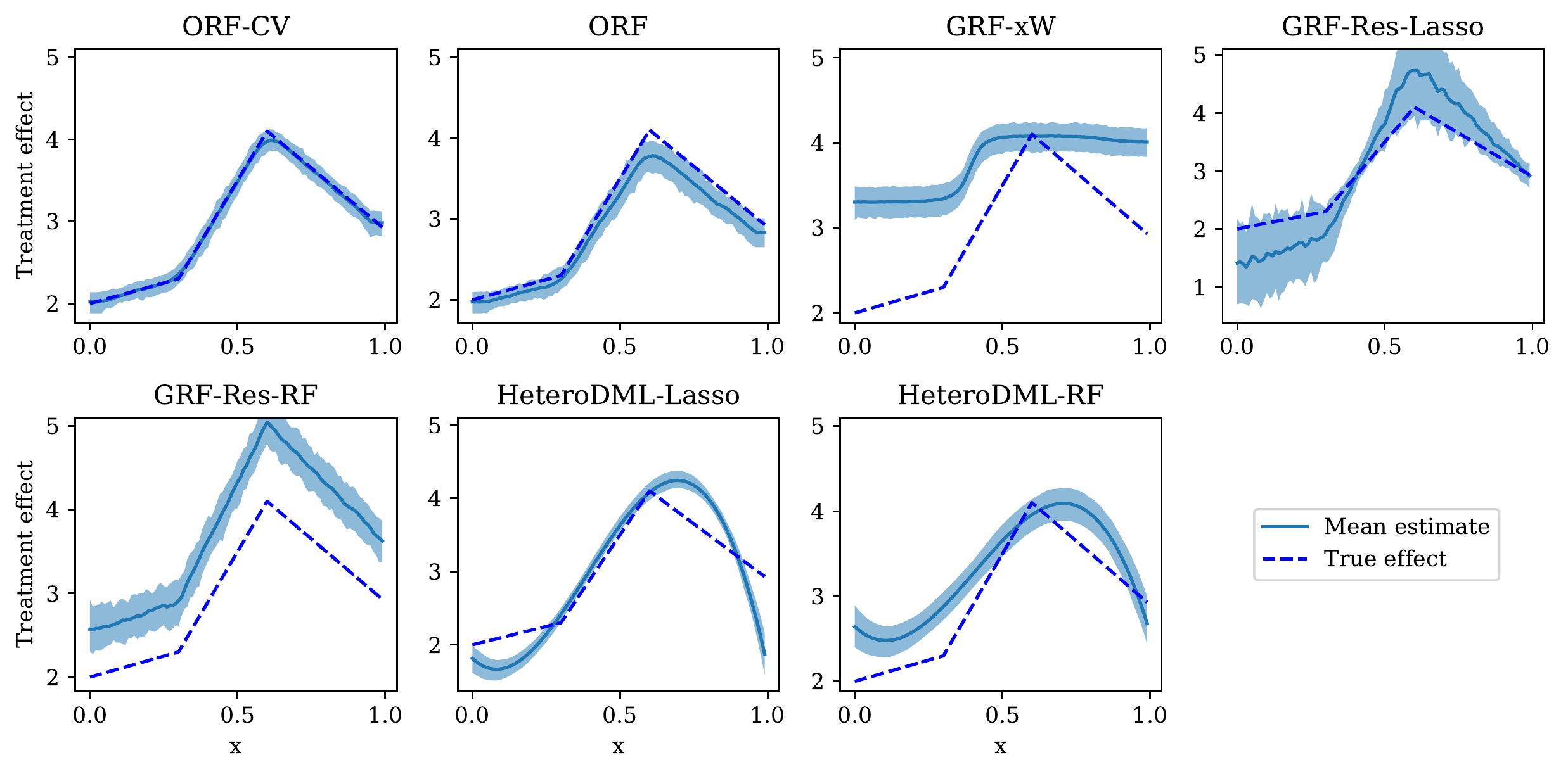}
	\end{minipage}\hfill
	\begin{minipage}[c]{0.23\textwidth}
		\caption{
			Treatment effect estimations for 100 Monte Carlo experiments with parameters $n=5000$, $p=500$, $d=1$, $\mathbf{k=25}$, and a piecewise linear treatment response. The shaded regions depict the mean and the $5\%$-$95\%$ interval of the 100 experiments.
		} 
	\end{minipage}
\end{figure}
% support 30
\begin{figure}[H]
		\begin{minipage}[c]{0.75\textwidth}
			\includegraphics[width=\textwidth]{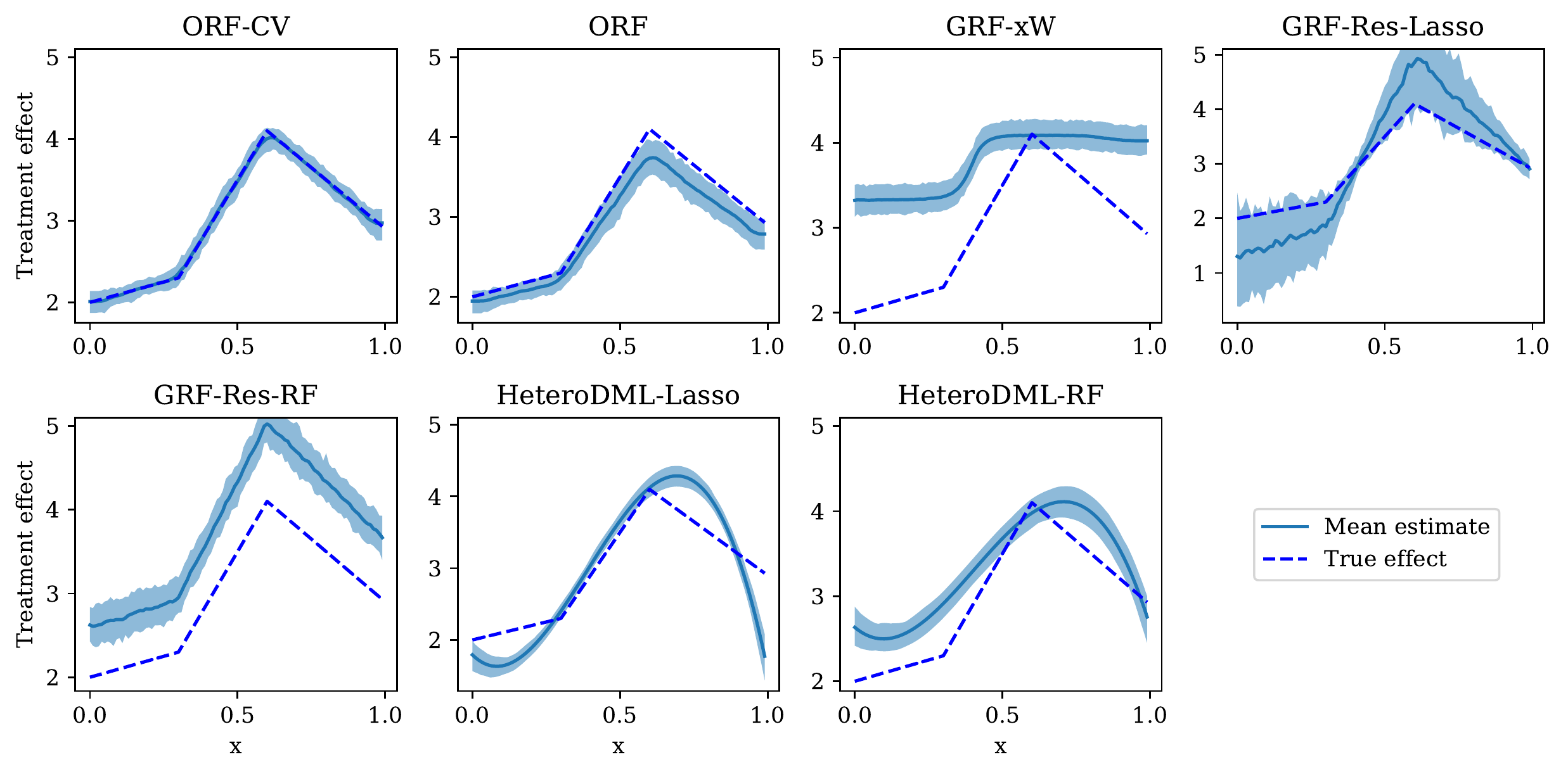}
		\end{minipage}\hfill
		\begin{minipage}[c]{0.23\textwidth}
			\caption{
				Treatment effect estimations for 100 Monte Carlo experiments with parameters $n=5000$, $p=500$, $d=1$, $\mathbf{k=30}$, and a piecewise linear treatment response. The shaded regions depict the mean and the $5\%$-$95\%$ interval of the 100 experiments.
			} 
		\end{minipage}
	\end{figure}

\pagebreak
\subsection{Experimental results for one-dimensional, piecewise constant $\theta_0$}

We introduce the results for a piecewise constant function given by:
\begin{align*}
\theta_0(x) &= \mathbb{I}_{x\leq 0.2} + 5\mathbb{I}_{x>0.2 \text{ and } x \leq 0.6} + 3\mathbb{I}_{x> 0.6}
\end{align*}

% summary
\begin{figure}[H]
	\centering
	\includegraphics[scale=.50]{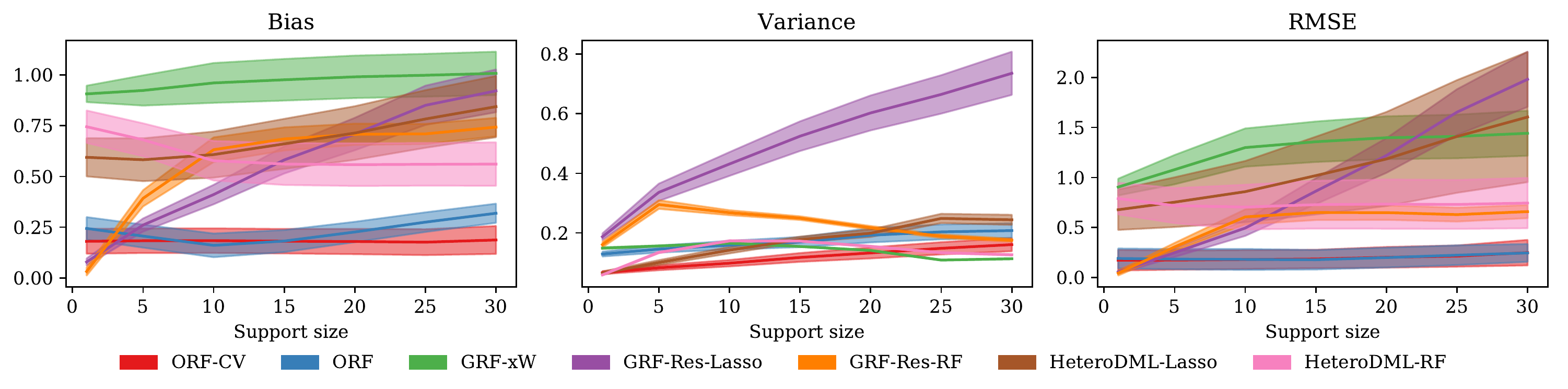}
	\caption{Bias, variance and RMSE as a function of support size for $n=5000$, $p=500$, $d=1$ and a piecewise constant treatment function. The solid lines represent the mean of the metrics across test points, averaged over the Monte Carlo experiments, and the filled regions depict the standard deviation, scaled down by a factor of 3 for clarity.}
\end{figure}
% support 1
\begin{figure}[H]
	\begin{minipage}[c]{0.75\textwidth}
		\includegraphics[width=\textwidth]{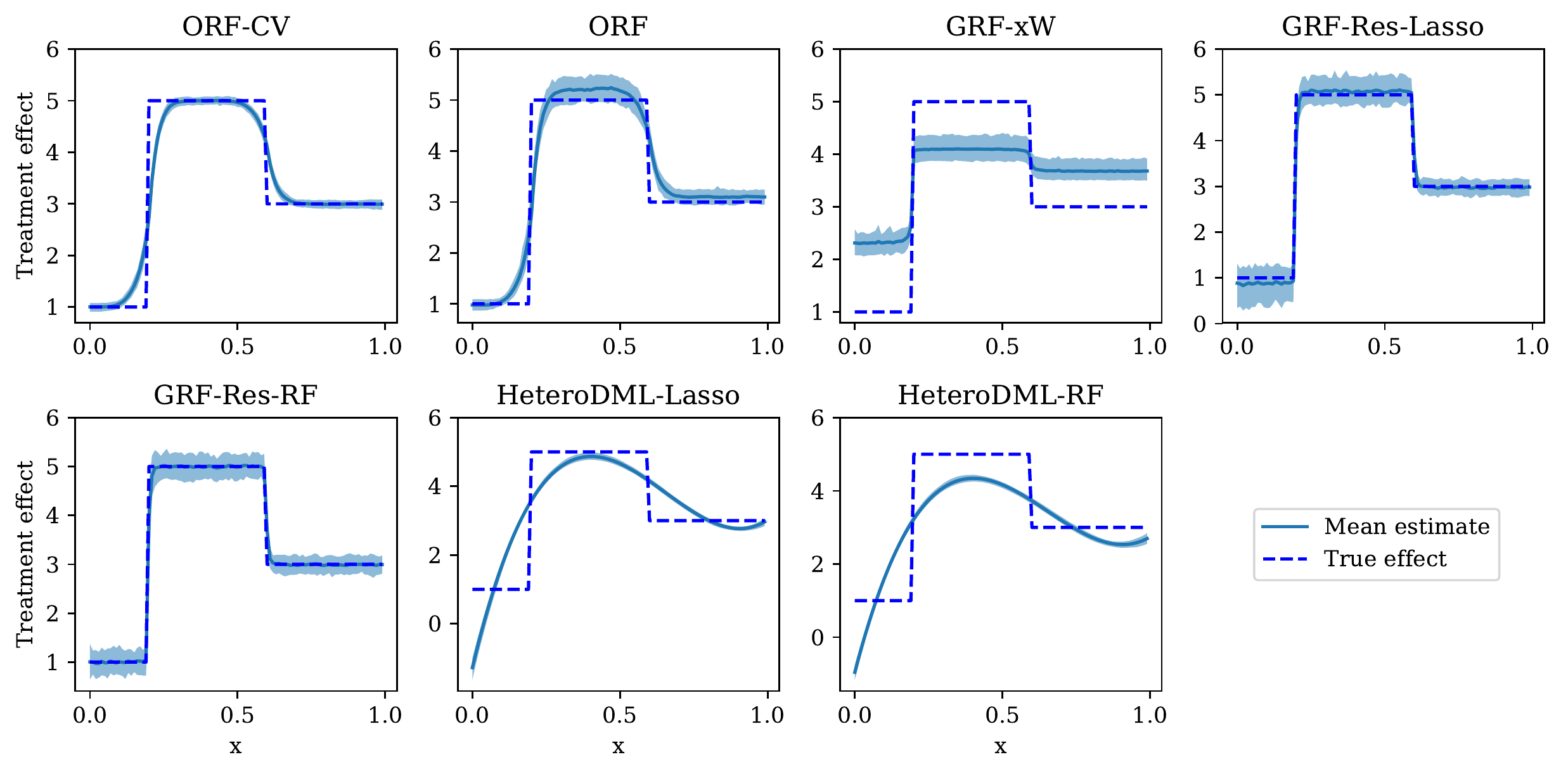}
	\end{minipage}\hfill
	\begin{minipage}[c]{0.23\textwidth}
		\caption{
			Treatment effect estimations for 100 Monte Carlo experiments with parameters $n=5000$, $p=500$, $d=1$, $\mathbf{k=1}$, and a piecewise constant treatment response. The shaded regions depict the mean and the $5\%$-$95\%$ interval of the 100 experiments.
		} 
	\end{minipage}
\end{figure}
% support 5
\begin{figure}[H]
	\begin{minipage}[c]{0.75\textwidth}
		\includegraphics[width=\textwidth]{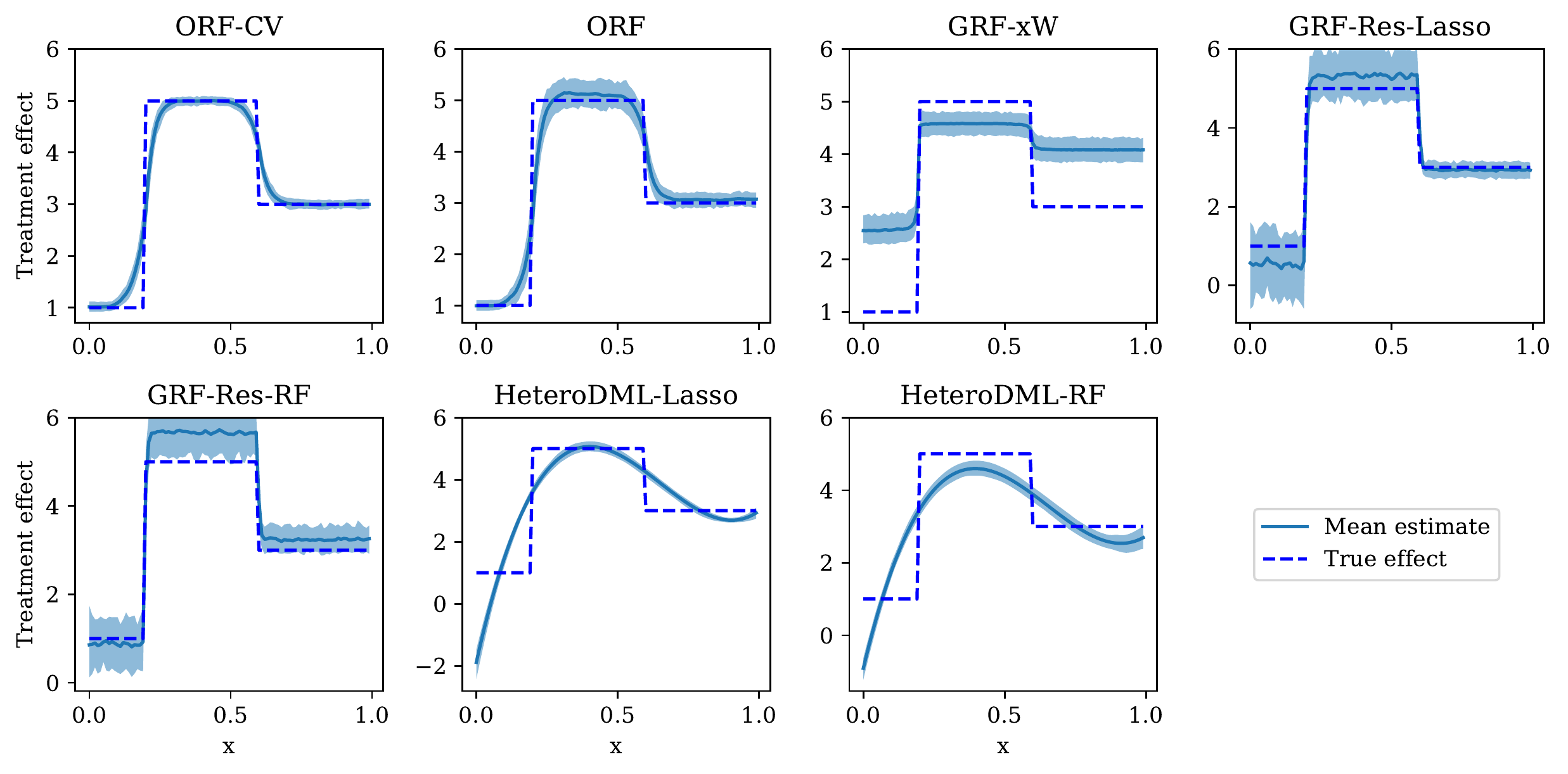}
	\end{minipage}\hfill
	\begin{minipage}[c]{0.23\textwidth}
		\caption{
			Treatment effect estimations for 100 Monte Carlo experiments with parameters $n=5000$, $p=500$, $d=1$, $\mathbf{k=5}$, and a piecewise constant treatment response. The shaded regions depict the mean and the $5\%$-$95\%$ interval of the 100 experiments.
		} 
	\end{minipage}
\end{figure}

\pagebreak
% support 10
\begin{figure}[H]
	\begin{minipage}[c]{0.75\textwidth}
		\includegraphics[width=\textwidth]{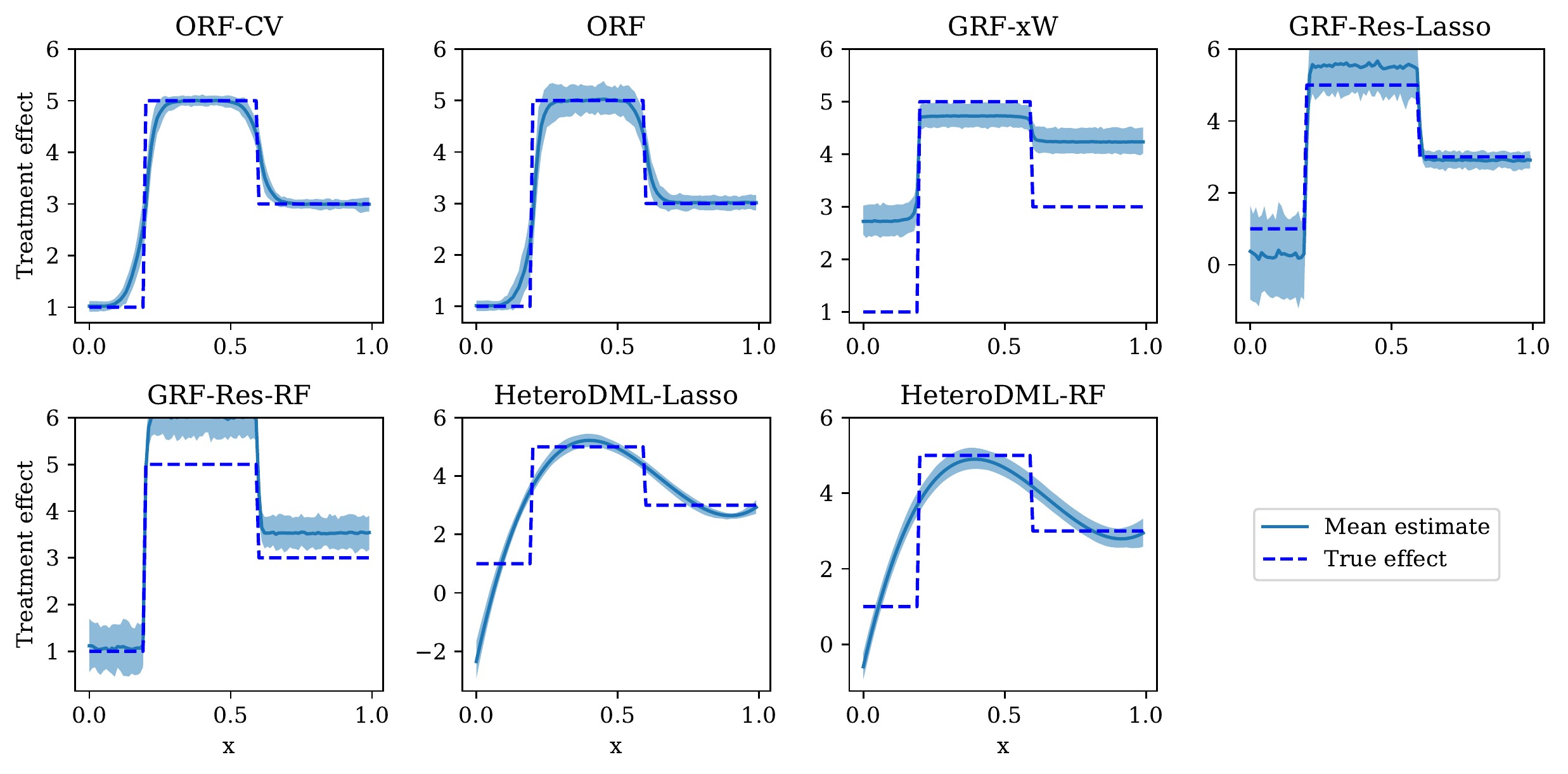}
	\end{minipage}\hfill
	\begin{minipage}[c]{0.23\textwidth}
		\caption{
			Treatment effect estimations for 100 Monte Carlo experiments with parameters $n=5000$, $p=500$, $d=1$, $\mathbf{k=10}$, and a piecewise constant treatment response. The shaded regions depict the mean and the $5\%$-$95\%$ interval of the 100 experiments.
		} 
	\end{minipage}
\end{figure}
% support 15
\begin{figure}[H]
	\begin{minipage}[c]{0.75\textwidth}
		\includegraphics[width=\textwidth]{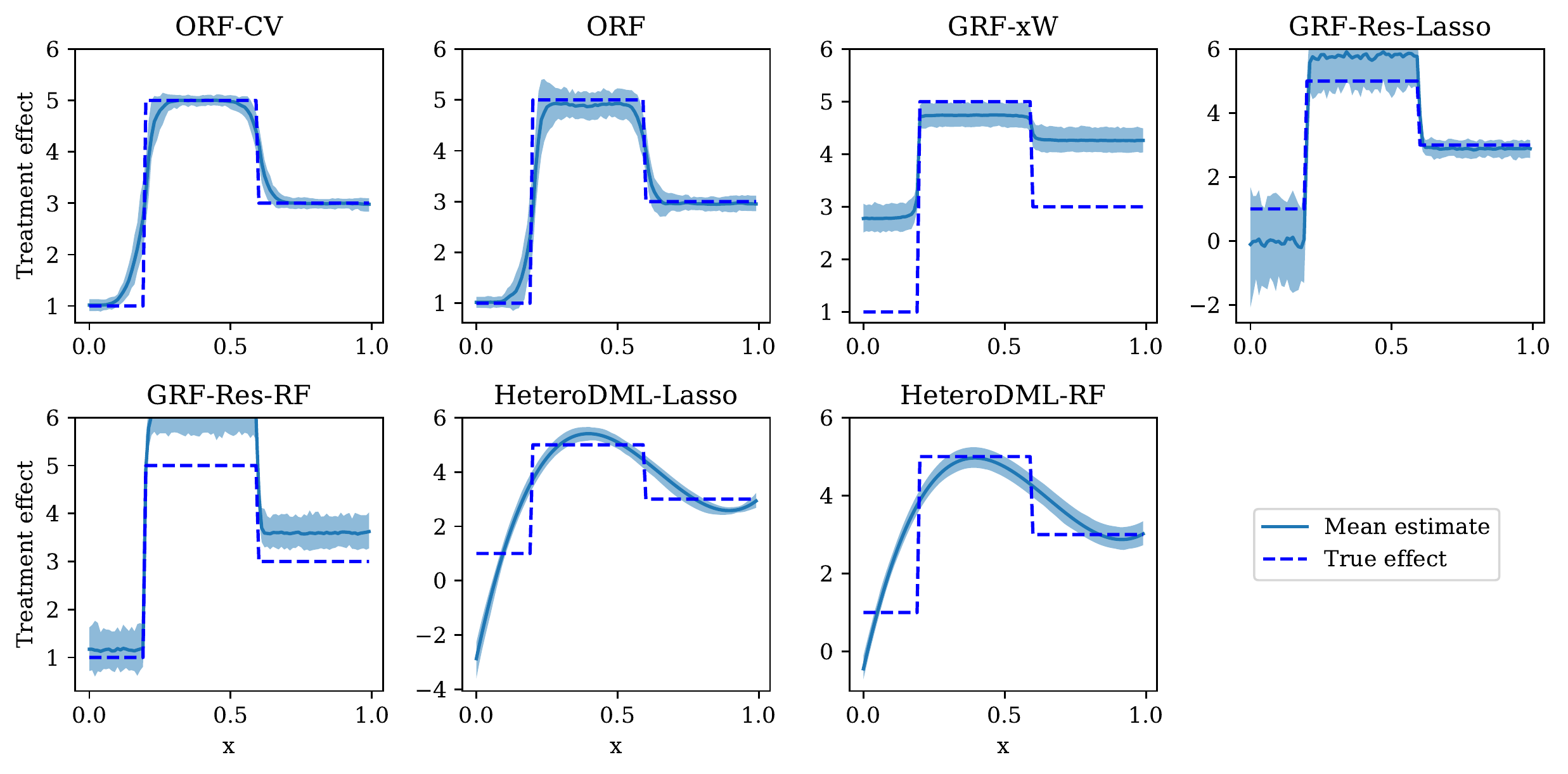}
	\end{minipage}\hfill
	\begin{minipage}[c]{0.23\textwidth}
		\caption{
			Treatment effect estimations for 100 Monte Carlo experiments with parameters $n=5000$, $p=500$, $d=1$, $\mathbf{k=15}$, and a piecewise constant treatment response. The shaded regions depict the mean and the $5\%$-$95\%$ interval of the 100 experiments.
		} 
	\end{minipage}
\end{figure}
% support 20
\begin{figure}[H]
	\begin{minipage}[c]{0.75\textwidth}
		\includegraphics[width=\textwidth]{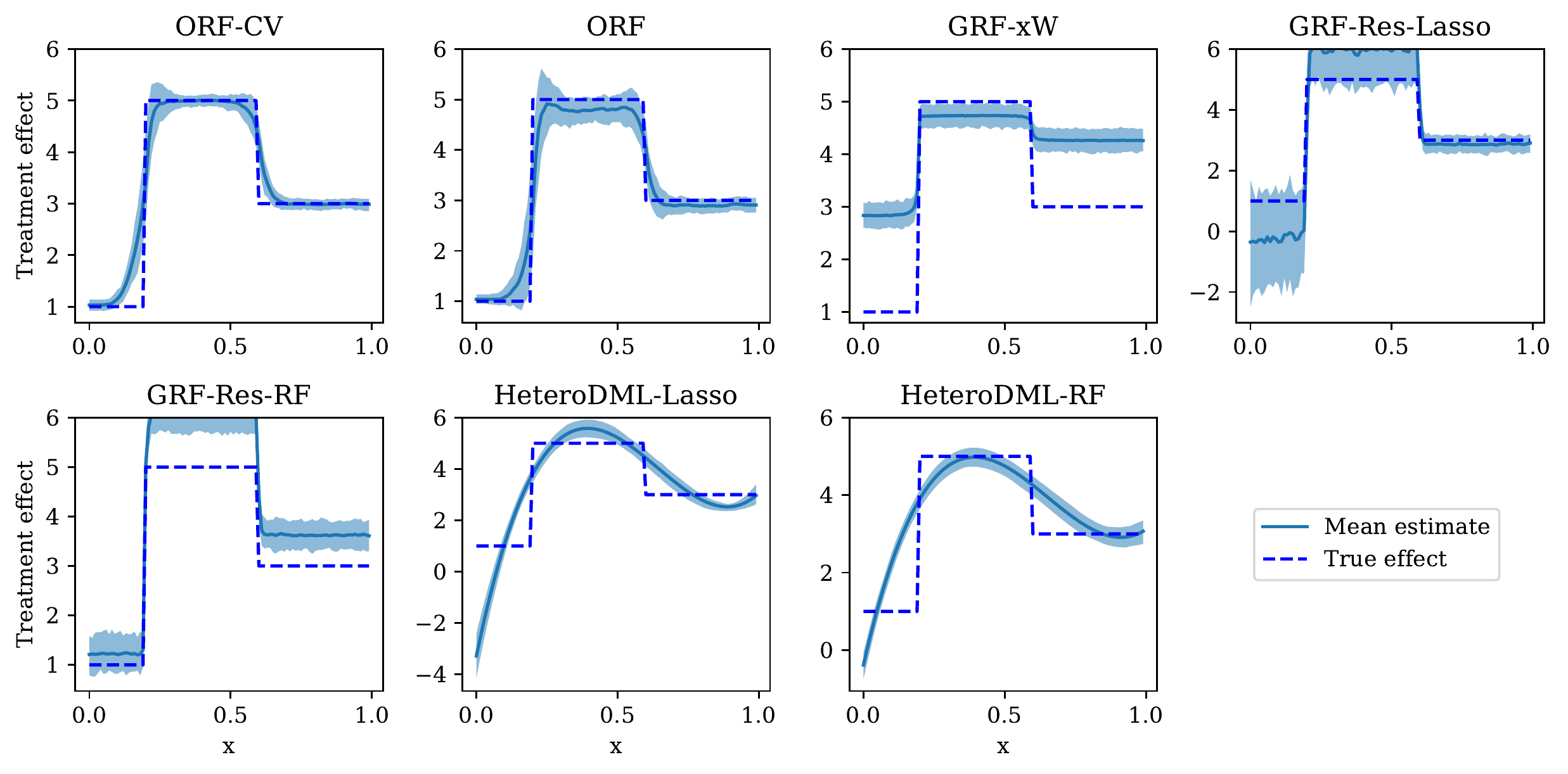}
	\end{minipage}\hfill
	\begin{minipage}[c]{0.23\textwidth}
		\caption{
			Treatment effect estimations for 100 Monte Carlo experiments with parameters $n=5000$, $p=500$, $d=1$, $\mathbf{k=20}$, and a piecewise constant treatment response. The shaded regions depict the mean and the $5\%$-$95\%$ interval of the 100 experiments.
		} 
	\end{minipage}
\end{figure}

\pagebreak
% support 25
\begin{figure}[H]
	\begin{minipage}[c]{0.75\textwidth}
		\includegraphics[width=\textwidth]{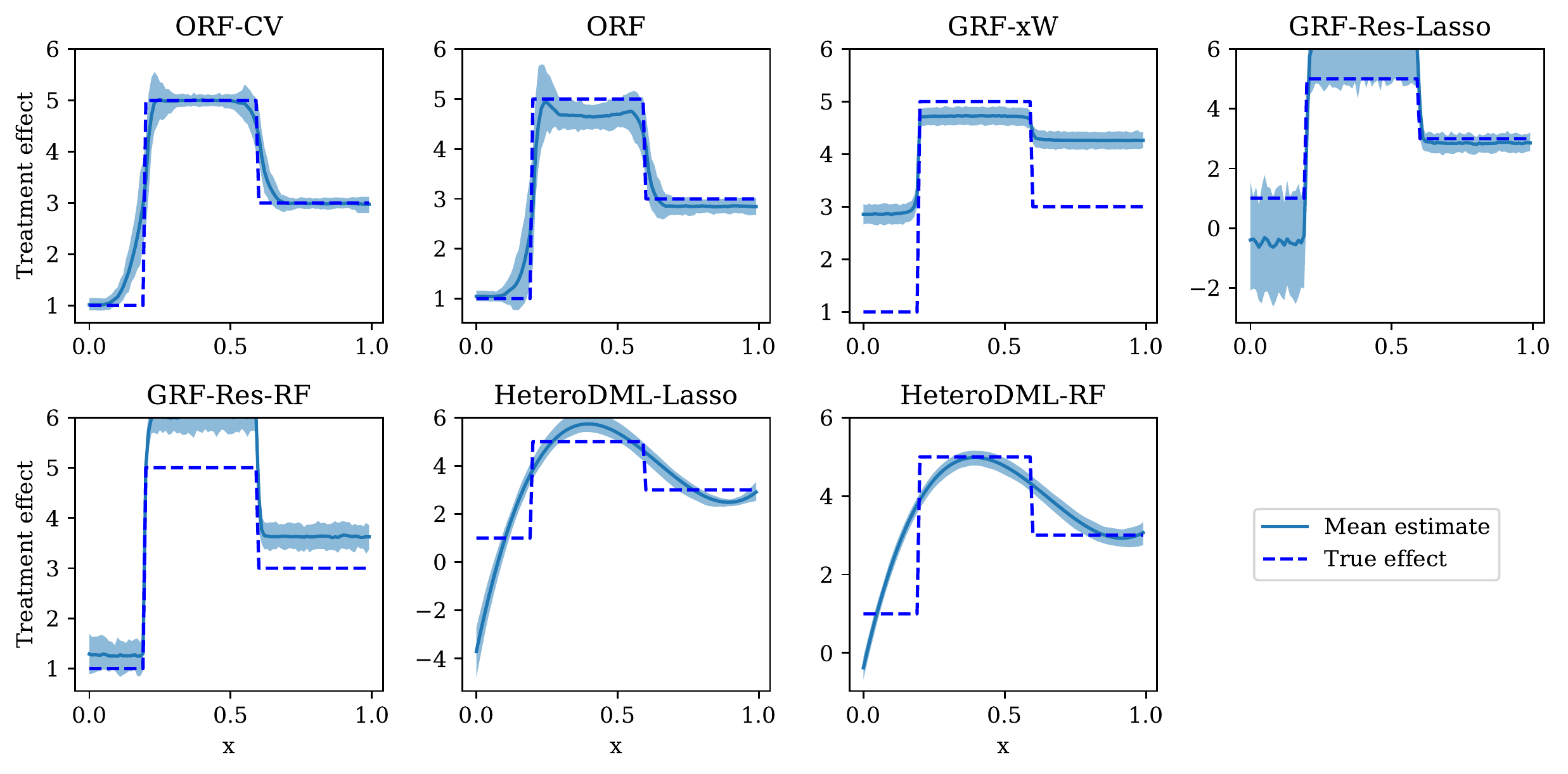}
	\end{minipage}\hfill
	\begin{minipage}[c]{0.23\textwidth}
		\caption{
			Treatment effect estimations for 100 Monte Carlo experiments with parameters $n=5000$, $p=500$, $d=1$, $\mathbf{k=25}$, and a piecewise constant treatment response. The shaded regions depict the mean and the $5\%$-$95\%$ interval of the 100 experiments.
		} 
	\end{minipage}
\end{figure}
% support 30
\begin{figure}[H]
		\begin{minipage}[c]{0.75\textwidth}
			\includegraphics[width=\textwidth]{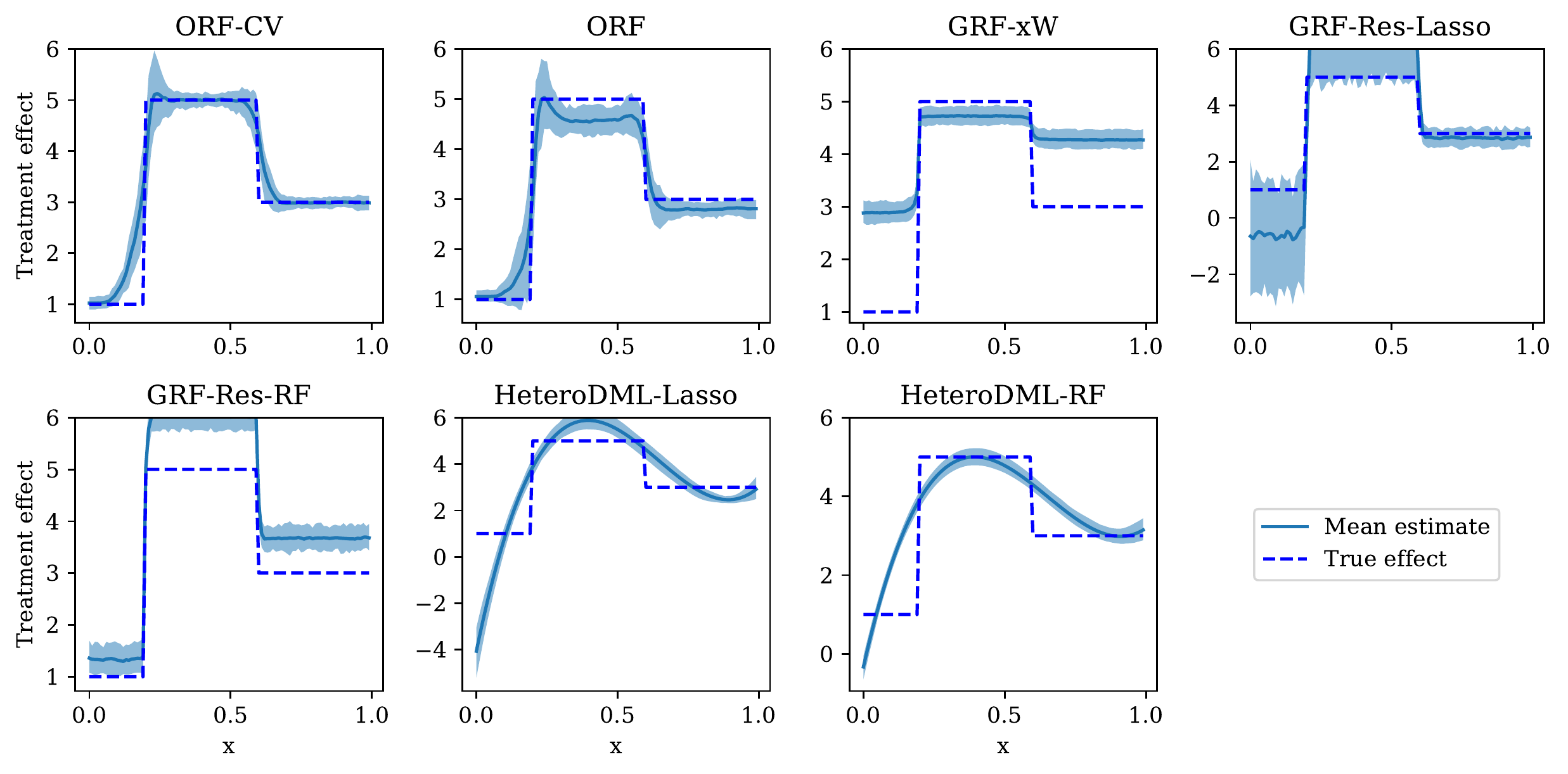}
		\end{minipage}\hfill
		\begin{minipage}[c]{0.23\textwidth}
			\caption{
				Treatment effect estimations for 100 Monte Carlo experiments with parameters $n=5000$, $p=500$, $d=1$, $\mathbf{k=30}$, and a piecewise constant treatment response. The shaded regions depict the mean and the $5\%$-$95\%$ interval of the 100 experiments.
			} 
		\end{minipage}
	\end{figure}

\pagebreak
\subsection{Experimental results for one-dimensional, piecewise polynomial $\theta_0$}

We present the results for a piecewise polynomial function given by:
\begin{align*}
\theta_0(x) &= 3x^2\mathbb{I}_{x\leq 0.2} + (3x^2+1)\mathbb{I}_{x>0.2 \text{ and } x \leq 0.6} + (6x+2)\mathbb{I}_{x> 0.6}
\end{align*}

% summary
\begin{figure}[H]
	\centering
	\includegraphics[scale=.50]{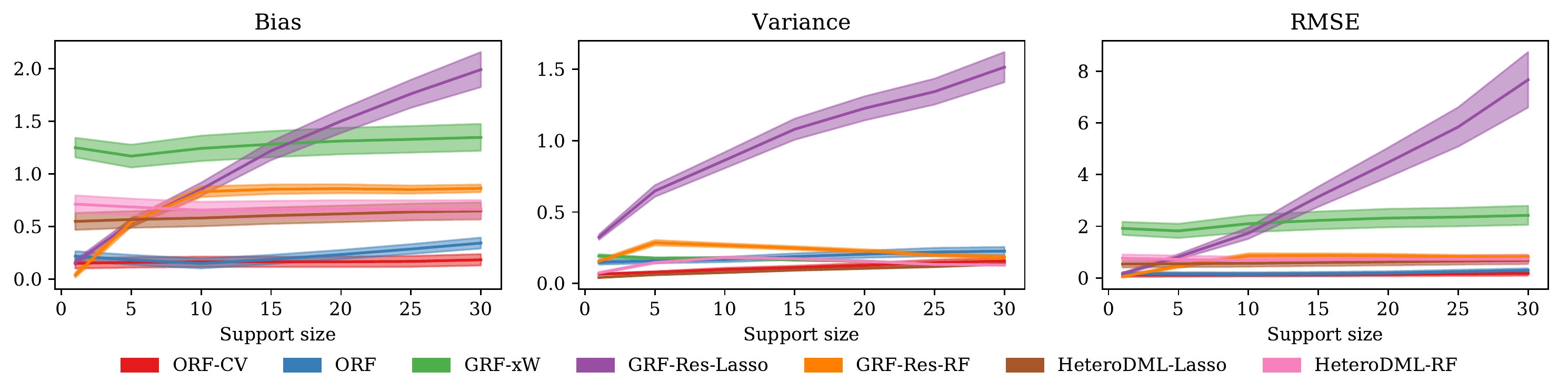}
	\caption{Bias, variance and RMSE as a function of support size for $n=5000$, $p=500$, $d=1$ and a piecewise polynomial treatment function. The solid lines represent the mean of the metrics across test points, averaged over the Monte Carlo experiments, and the filled regions depict the standard deviation, scaled down by a factor of 3 for clarity.}
\end{figure}
% support 1
\begin{figure}[H]
	\begin{minipage}[c]{0.75\textwidth}
		\includegraphics[width=\textwidth]{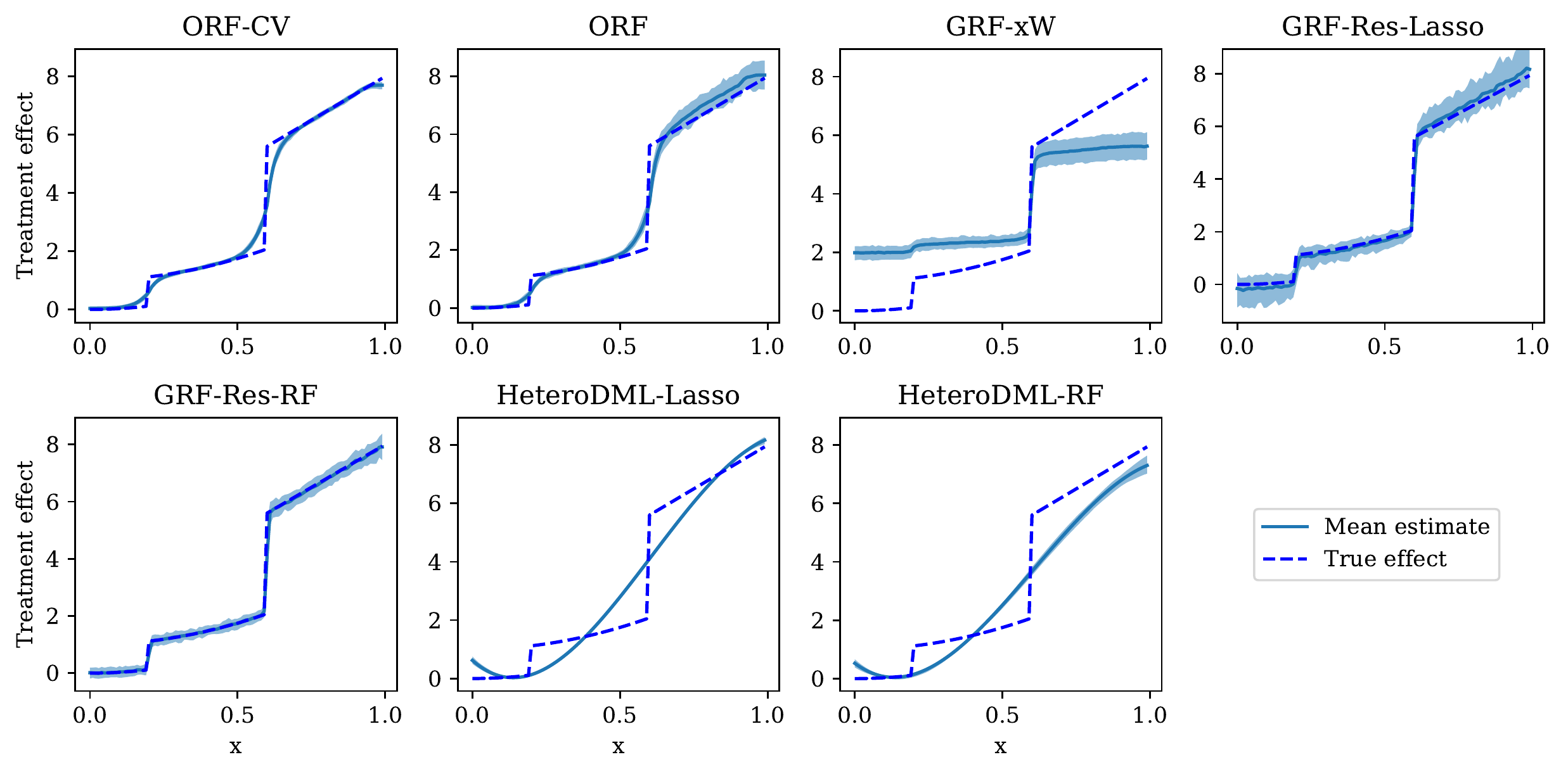}
	\end{minipage}\hfill
	\begin{minipage}[c]{0.23\textwidth}
		\caption{
			Treatment effect estimations for 100 Monte Carlo experiments with parameters $n=5000$, $p=500$, $d=1$, $\mathbf{k=1}$, and a piecewise polynomial treatment response. The shaded regions depict the mean and the $5\%$-$95\%$ interval of the 100 experiments.
		} 
	\end{minipage}
\end{figure}
% support 5
\begin{figure}[H]
	\begin{minipage}[c]{0.75\textwidth}
		\includegraphics[width=\textwidth]{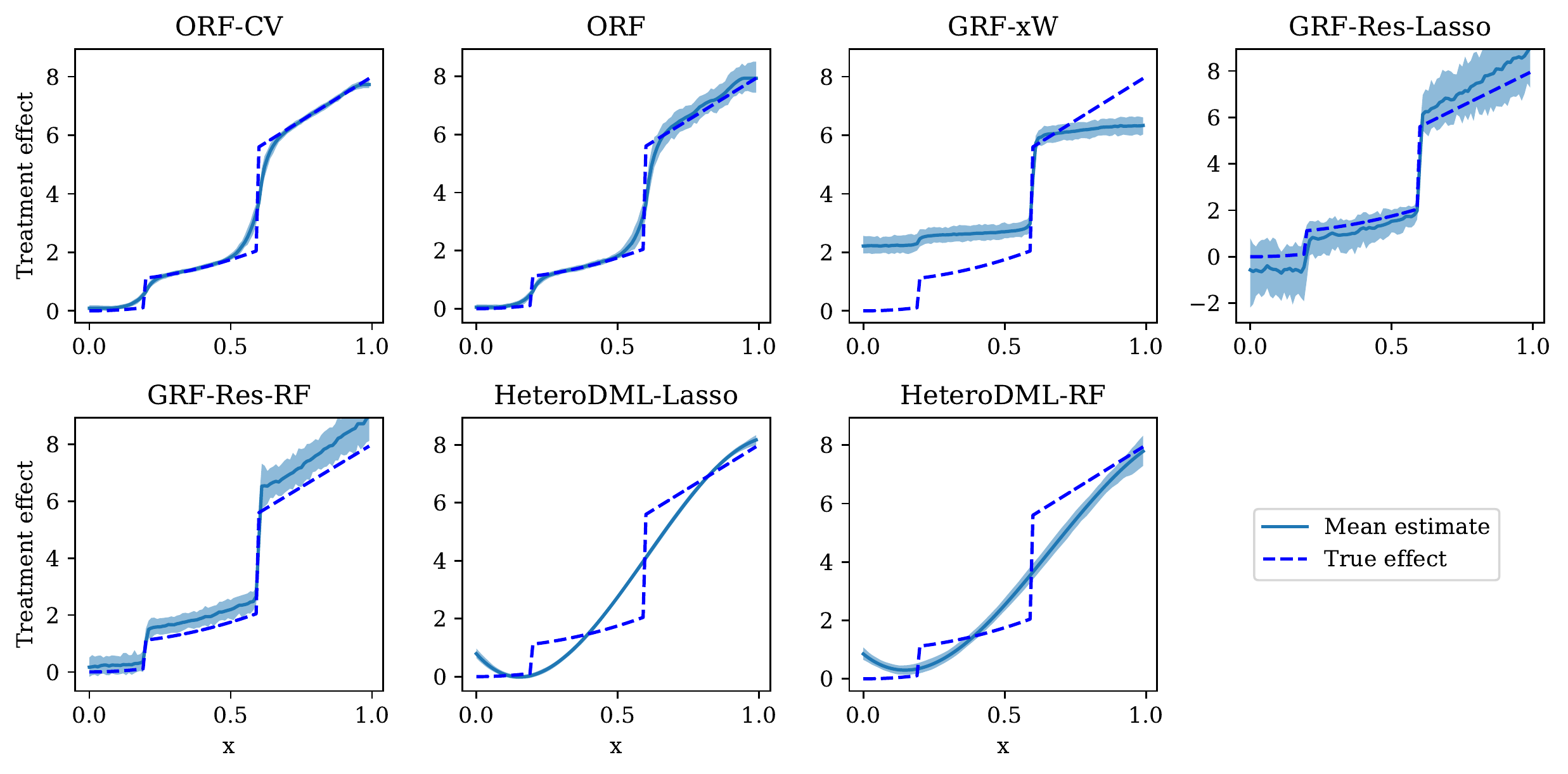}
	\end{minipage}\hfill
	\begin{minipage}[c]{0.23\textwidth}
		\caption{
			Treatment effect estimations for 100 Monte Carlo experiments with parameters $n=5000$, $p=500$, $d=1$, $\mathbf{k=5}$, and a piecewise polynomial treatment response. The shaded regions depict the mean and the $5\%$-$95\%$ interval of the 100 experiments.
		} 
	\end{minipage}
\end{figure}

\pagebreak
% support 10
\begin{figure}[H]
	\begin{minipage}[c]{0.75\textwidth}
		\includegraphics[width=\textwidth]{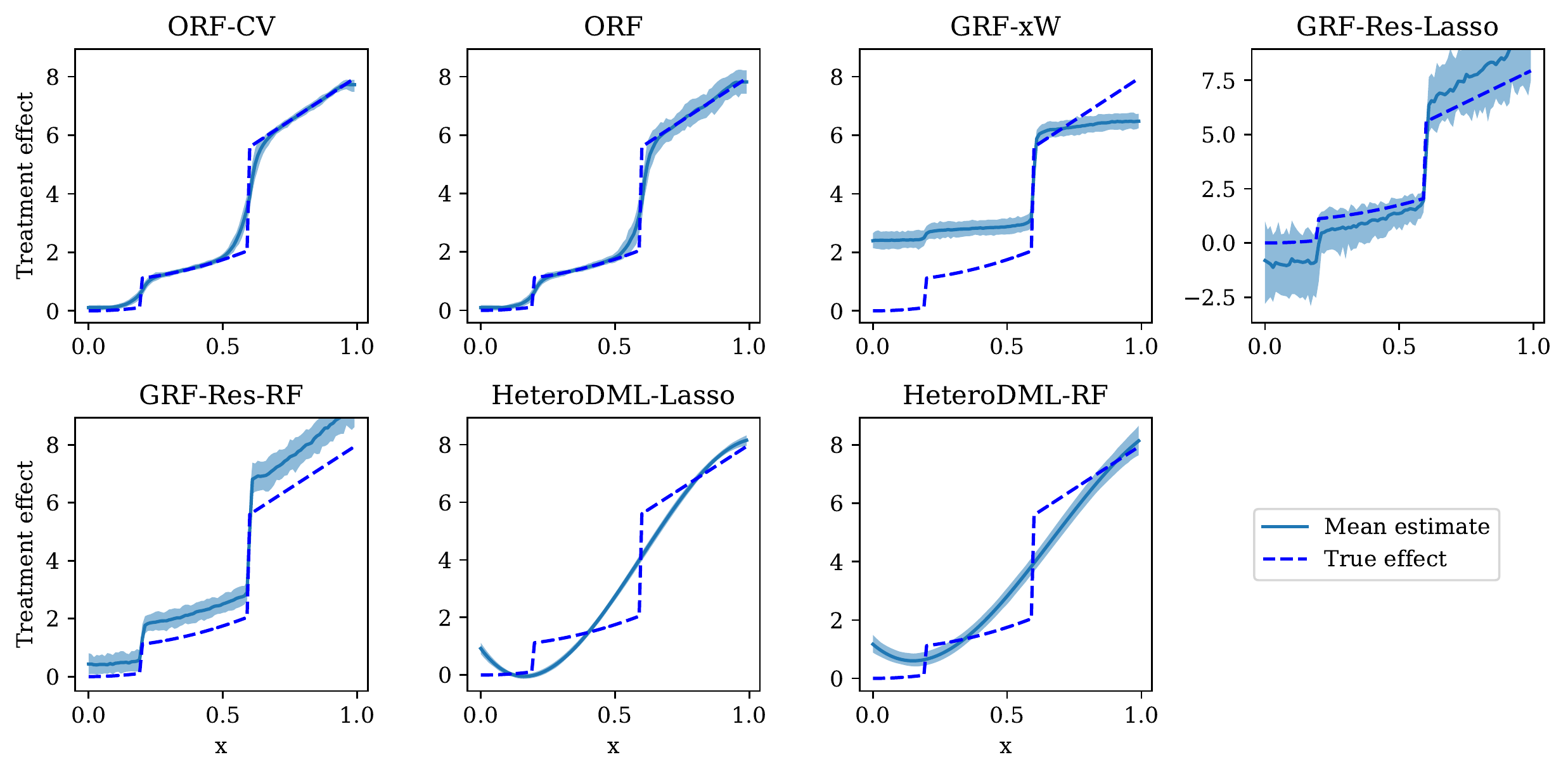}
	\end{minipage}\hfill
	\begin{minipage}[c]{0.23\textwidth}
		\caption{
			Treatment effect estimations for 100 Monte Carlo experiments with parameters $n=5000$, $p=500$, $d=1$, $\mathbf{k=10}$, and a piecewise polynomial treatment response. The shaded regions depict the mean and the $5\%$-$95\%$ interval of the 100 experiments.
		} 
	\end{minipage}
\end{figure}
% support 15
\begin{figure}[H]
	\begin{minipage}[c]{0.75\textwidth}
		\includegraphics[width=\textwidth]{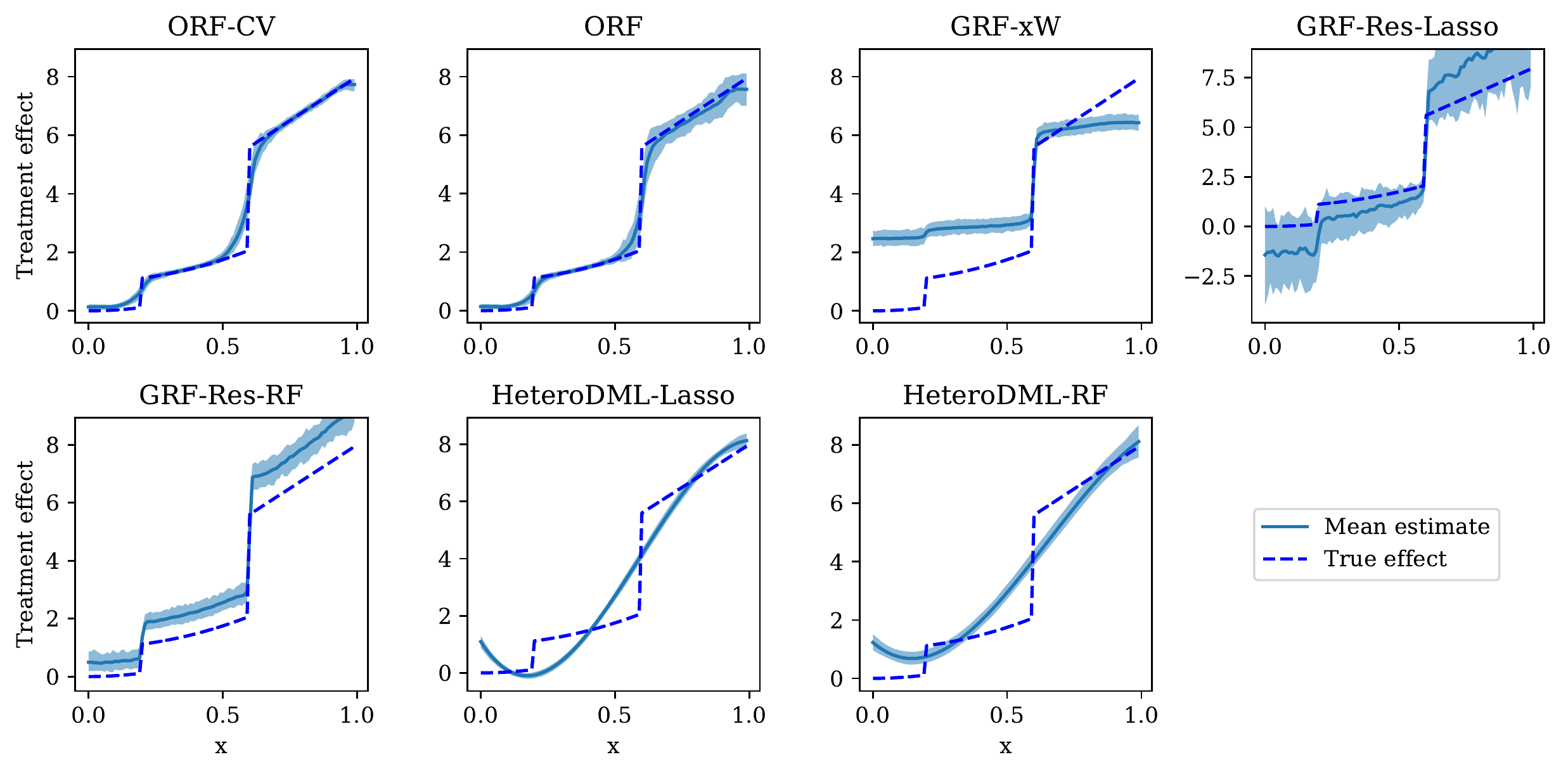}
	\end{minipage}\hfill
	\begin{minipage}[c]{0.23\textwidth}
		\caption{
			Treatment effect estimations for 100 Monte Carlo experiments with parameters $n=5000$, $p=500$, $d=1$, $\mathbf{k=15}$, and a piecewise polynomial treatment response. The shaded regions depict the mean and the $5\%$-$95\%$ interval of the 100 experiments.
		} 
	\end{minipage}
\end{figure}
% support 20
\begin{figure}[H]
	\begin{minipage}[c]{0.75\textwidth}
		\includegraphics[width=\textwidth]{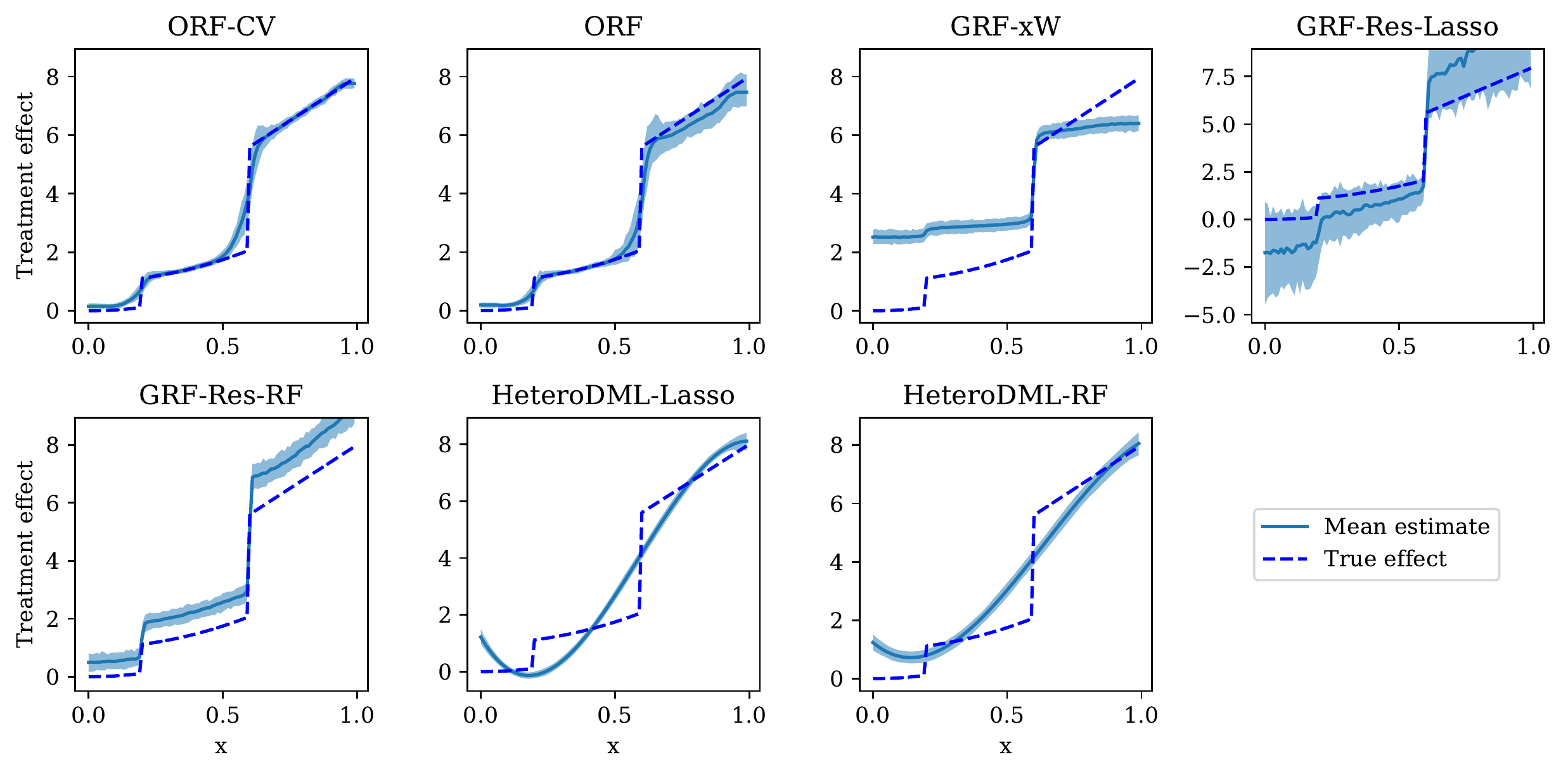}
	\end{minipage}\hfill
	\begin{minipage}[c]{0.23\textwidth}
		\caption{
			Treatment effect estimations for 100 Monte Carlo experiments with parameters $n=5000$, $p=500$, $d=1$, $\mathbf{k=20}$, and a piecewise polynomial treatment response. The shaded regions depict the mean and the $5\%$-$95\%$ interval of the 100 experiments.
		} 
	\end{minipage}
\end{figure}

\pagebreak
% support 25
\begin{figure}[H]
	\begin{minipage}[c]{0.75\textwidth}
		\includegraphics[width=\textwidth]{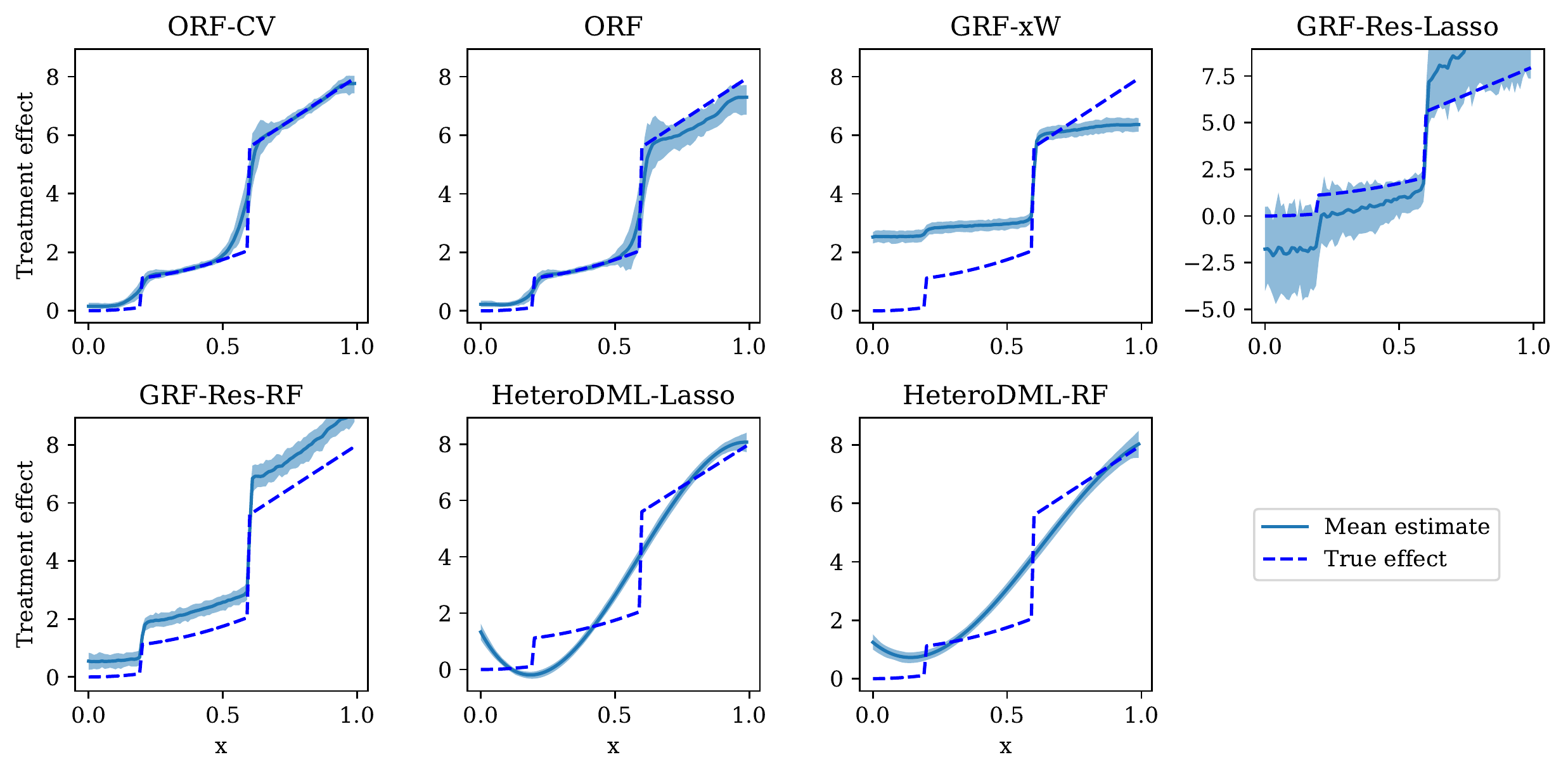}
	\end{minipage}\hfill
	\begin{minipage}[c]{0.23\textwidth}
		\caption{
			Treatment effect estimations for 100 Monte Carlo experiments with parameters $n=5000$, $p=500$, $d=1$, $\mathbf{k=25}$, and a piecewise polynomial treatment response. The shaded regions depict the mean and the $5\%$-$95\%$ interval of the 100 experiments.
		} 
	\end{minipage}
\end{figure}
% support 30
\begin{figure}[H]
		\begin{minipage}[c]{0.75\textwidth}
			\includegraphics[width=\textwidth]{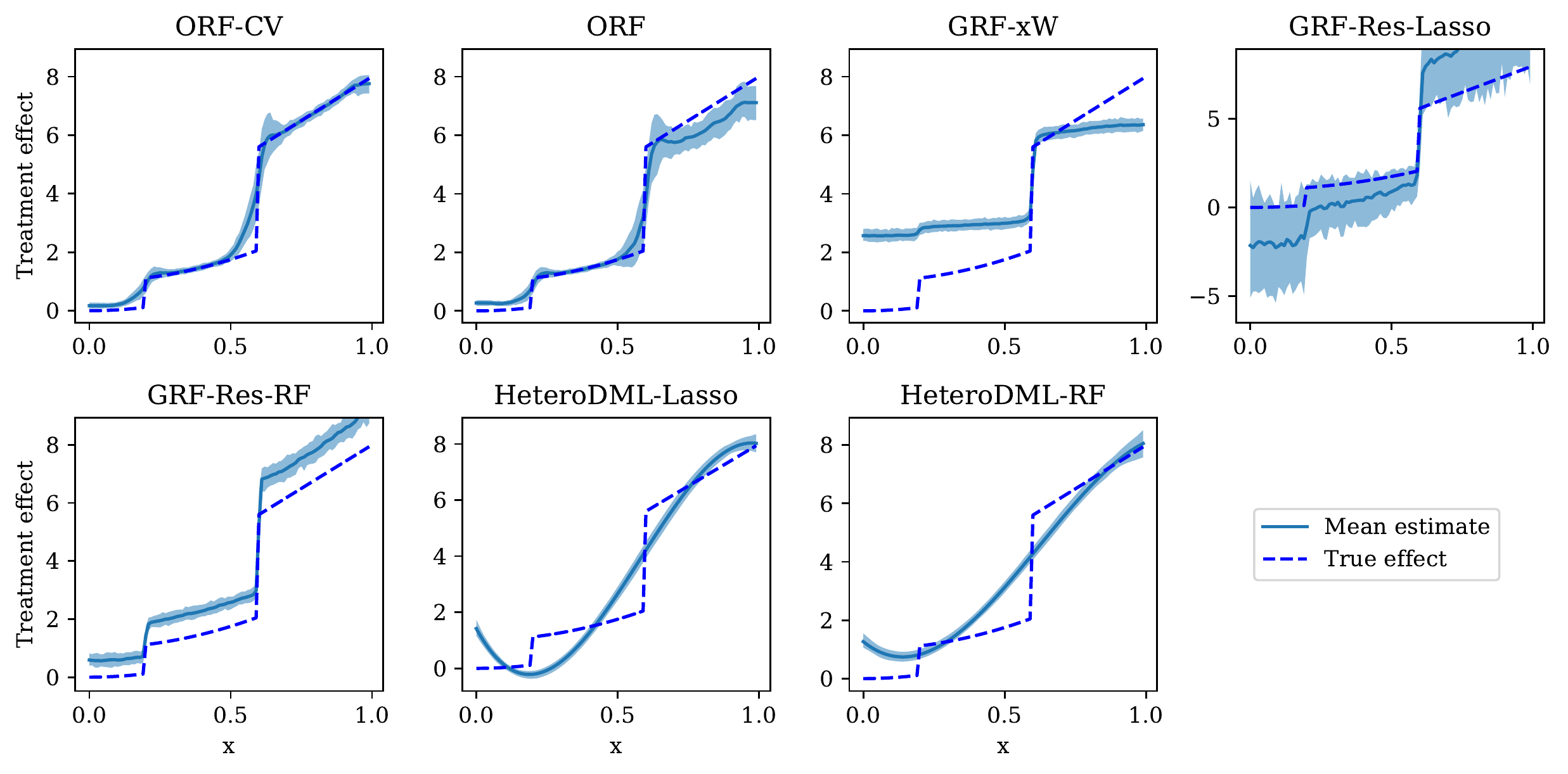}
		\end{minipage}\hfill
		\begin{minipage}[c]{0.23\textwidth}
			\caption{
				Treatment effect estimations for 100 Monte Carlo experiments with parameters $n=5000$, $p=500$, $d=1$, $\mathbf{k=30}$, and a piecewise polynomial treatment response. The shaded regions depict the mean and the $5\%$-$95\%$ interval of the 100 experiments.
			} 
		\end{minipage}
	\end{figure}

\pagebreak
\subsection{Experimental results for larger control support}

We present experimental results for a piecewise linear treatment response $\theta_0$, with $n=5000$ samples and large support $k\in \{50,75,100,150,200\}$. Figures \ref{fig:large1}-\ref{fig:large2} illustrate that the behavior of the ORF-CV algorithm, with parameters set in accordance our theoretical results, is consistent up until fairly large support sizes. Our method performs well with respect to the chosen evaluation metrics and outperform other estimators for larger support sizes. 
% summary
\begin{figure}[H]
	\centering
	\includegraphics[scale=.45]{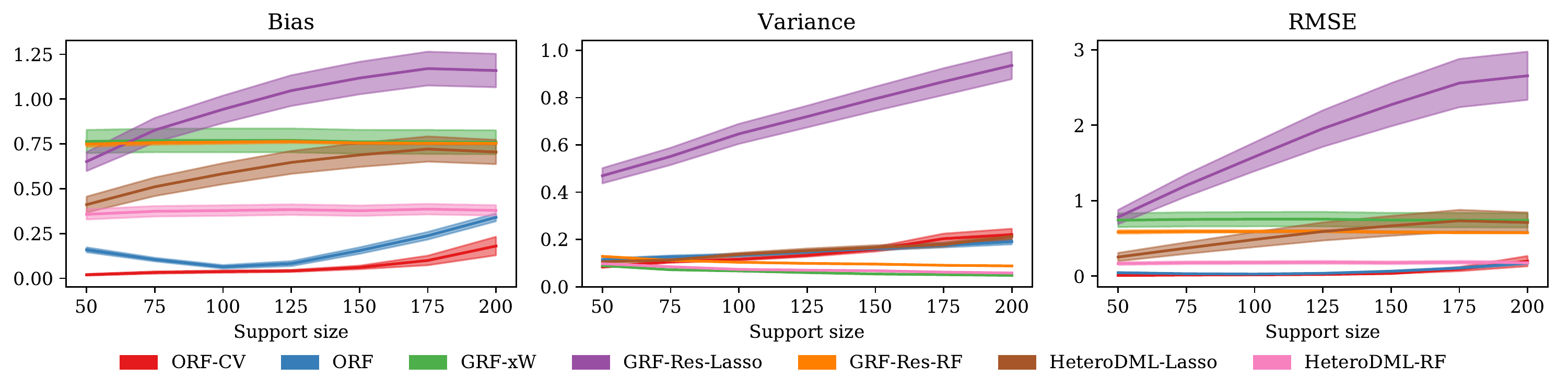}
	\caption{Bias, variance and RMSE as a function of support size. The solid lines represent the mean of the metrics across test points, averaged over the Monte Carlo experiments, and the filled regions depict the standard deviation, scaled down by a factor of 3 for clarity.}
	\label{fig:large1}
\end{figure}
% support 50
\begin{figure}[H]
	\begin{minipage}[c]{0.70\textwidth}
		\includegraphics[width=\textwidth]{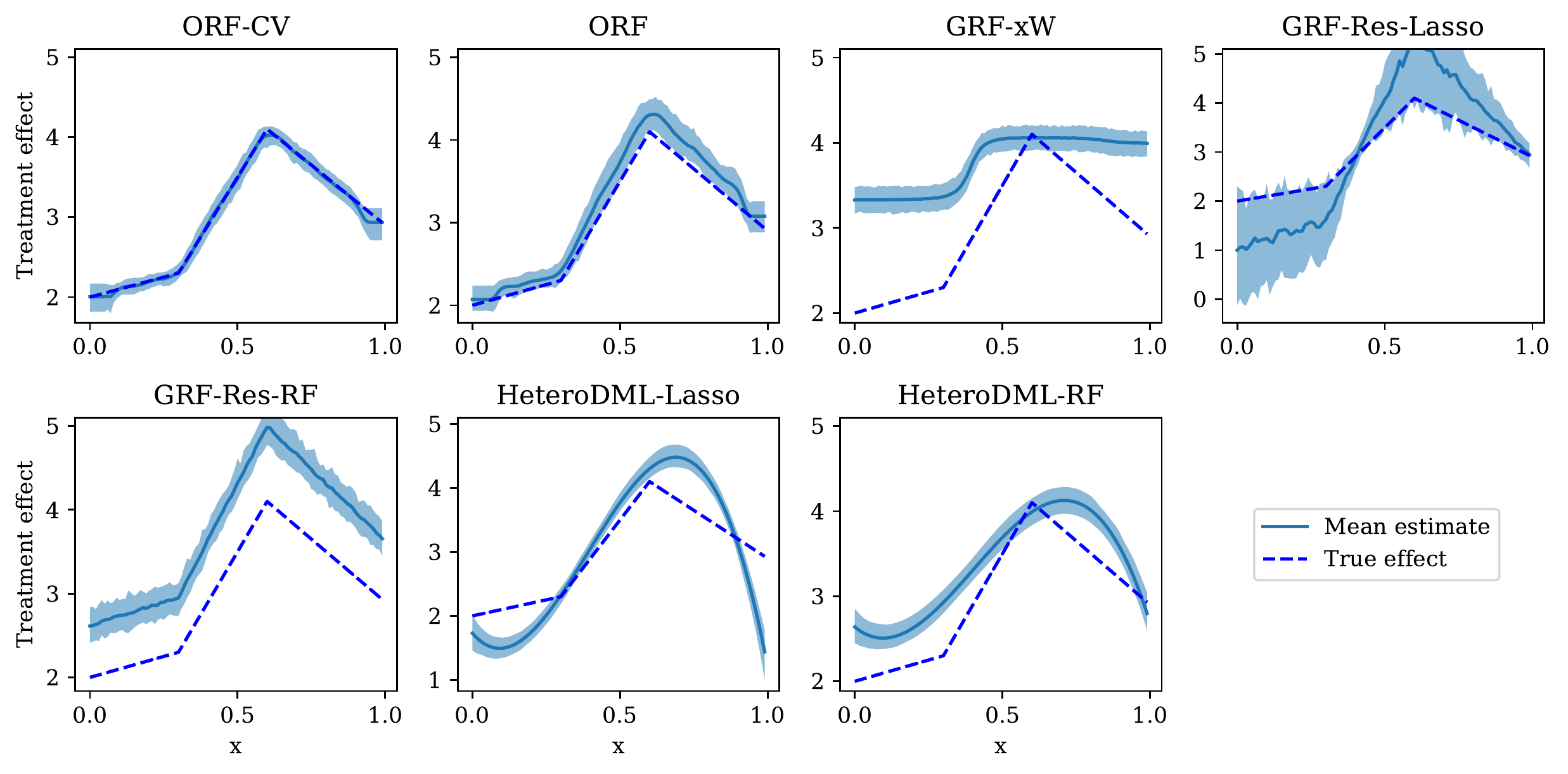}
	\end{minipage}\hfill
	\begin{minipage}[c]{0.25\textwidth}
		\caption{
			Treatment effect estimations for 100 Monte Carlo experiments with parameters $n=5000$, $p=500$, $d=1$, $\mathbf{k=50}$, and a piecewise linear treatment response. The shaded regions depict the mean and the $5\%$-$95\%$ interval of the 100 experiments.
		} 
	\end{minipage}
\end{figure}
% support 75
\begin{figure}[H]
	\begin{minipage}[c]{0.70\textwidth}
		\includegraphics[width=\textwidth]{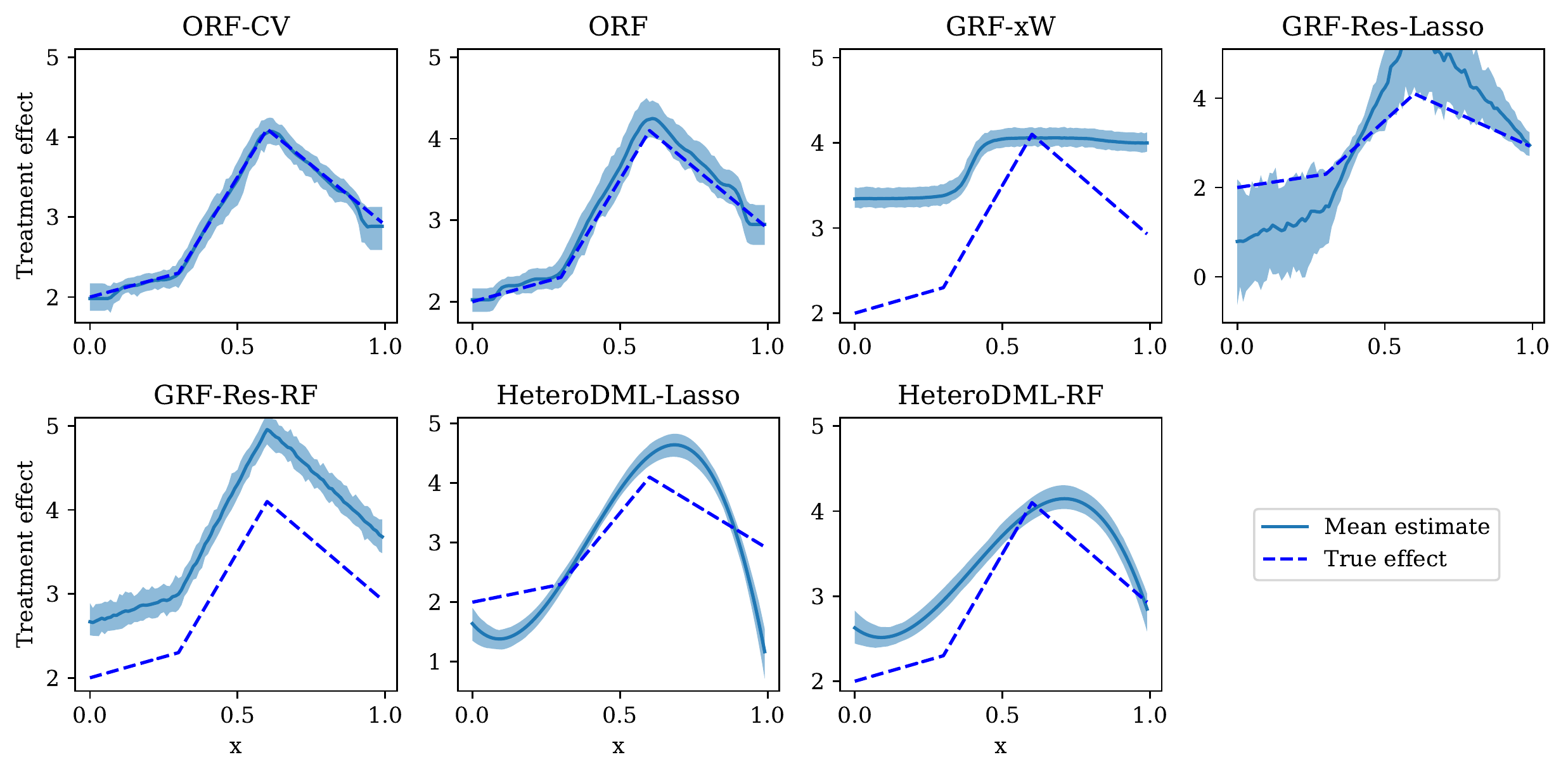}
	\end{minipage}\hfill
	\begin{minipage}[c]{0.25\textwidth}
		\caption{
			Treatment effect estimations for 100 Monte Carlo experiments with parameters $n=5000$, $p=500$, $d=1$, $\mathbf{k=75}$, and a piecewise linear treatment response. The shaded regions depict the mean and the $5\%$-$95\%$ interval of the 100 experiments.
		} 
	\end{minipage}
\end{figure}

\pagebreak
% support 100
\begin{figure}[H]
	\begin{minipage}[c]{0.75\textwidth}
		\includegraphics[width=\textwidth]{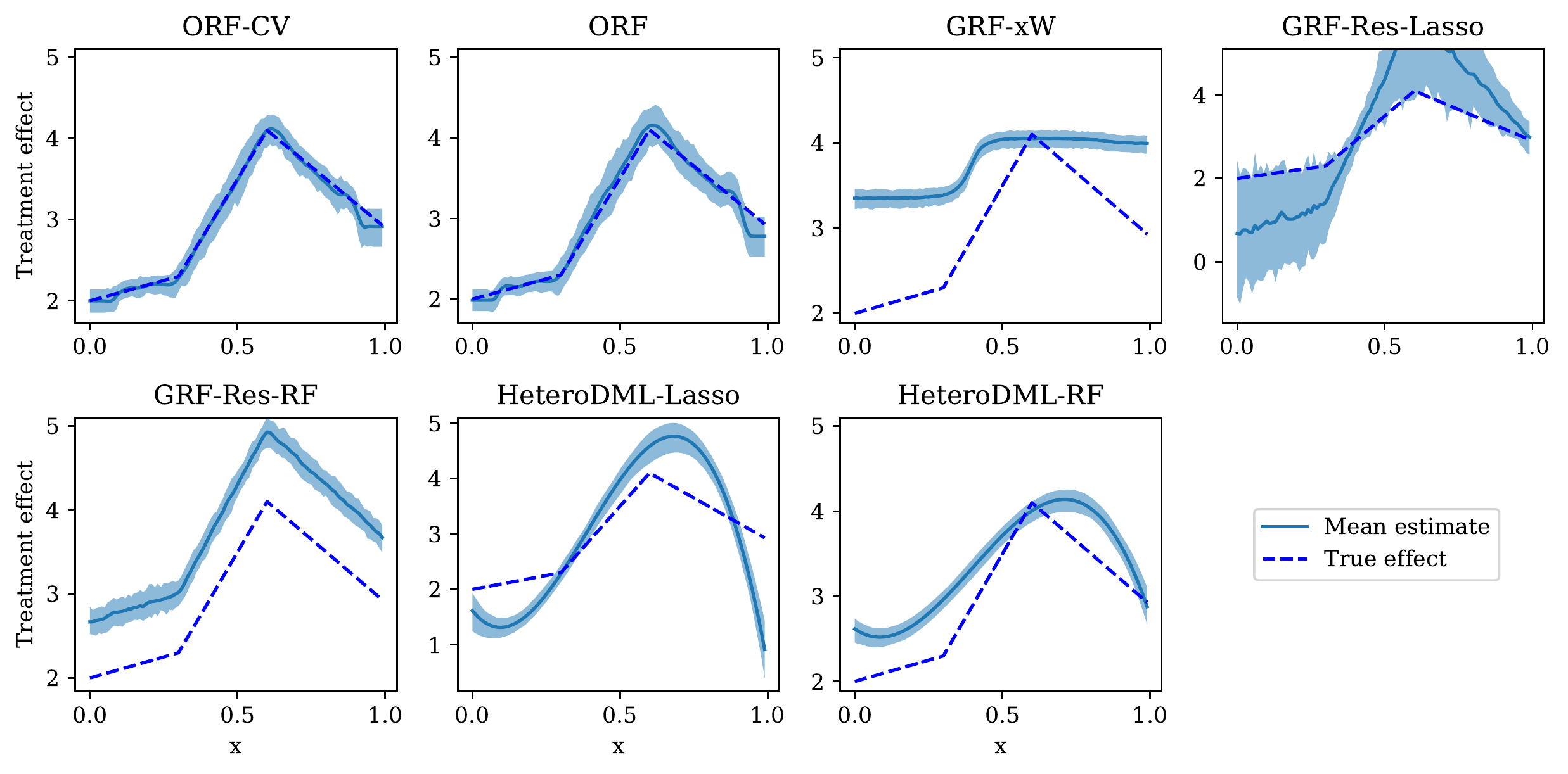}
	\end{minipage}\hfill
	\begin{minipage}[c]{0.23\textwidth}
		\caption{
			Treatment effect estimations for 100 Monte Carlo experiments with parameters $n=5000$, $p=500$, $d=1$, $\mathbf{k=100}$, and a piecewise linear treatment response. The shaded regions depict the mean and the $5\%$-$95\%$ interval of the 100 experiments.
		} 
	\end{minipage}
\end{figure}
% support 150
\begin{figure}[H]
	\begin{minipage}[c]{0.75\textwidth}
		\includegraphics[width=\textwidth]{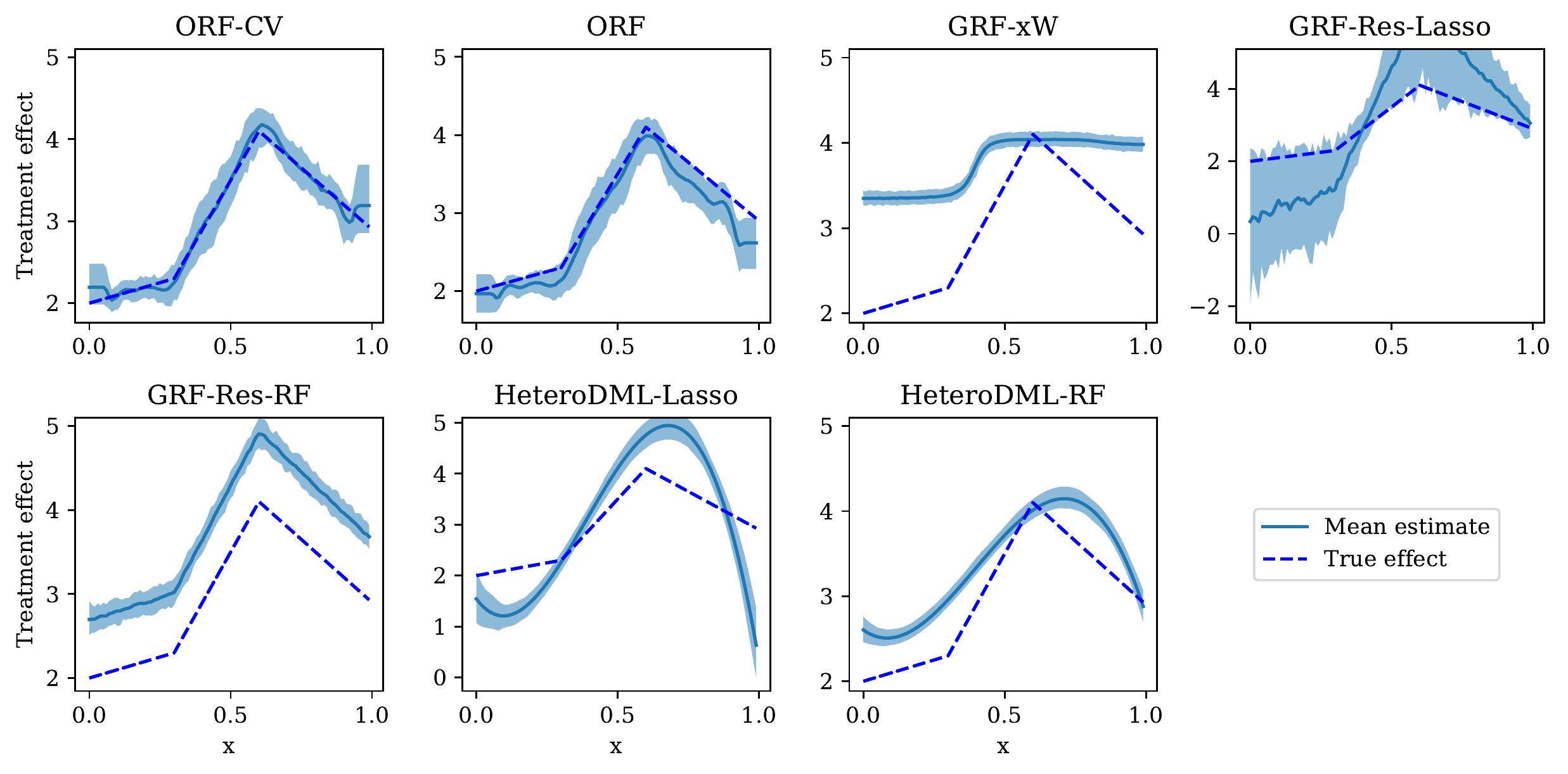}
	\end{minipage}\hfill
	\begin{minipage}[c]{0.23\textwidth}
		\caption{
			Treatment effect estimations for 100 Monte Carlo experiments with parameters $n=5000$, $p=500$, $d=1$, $\mathbf{k=150}$, and a piecewise linear treatment response. The shaded regions depict the mean and the $5\%$-$95\%$ interval of the 100 experiments.
		} 
	\end{minipage}
\end{figure}
% support 200
\begin{figure}[H]
	\begin{minipage}[c]{0.75\textwidth}
		\includegraphics[width=\textwidth]{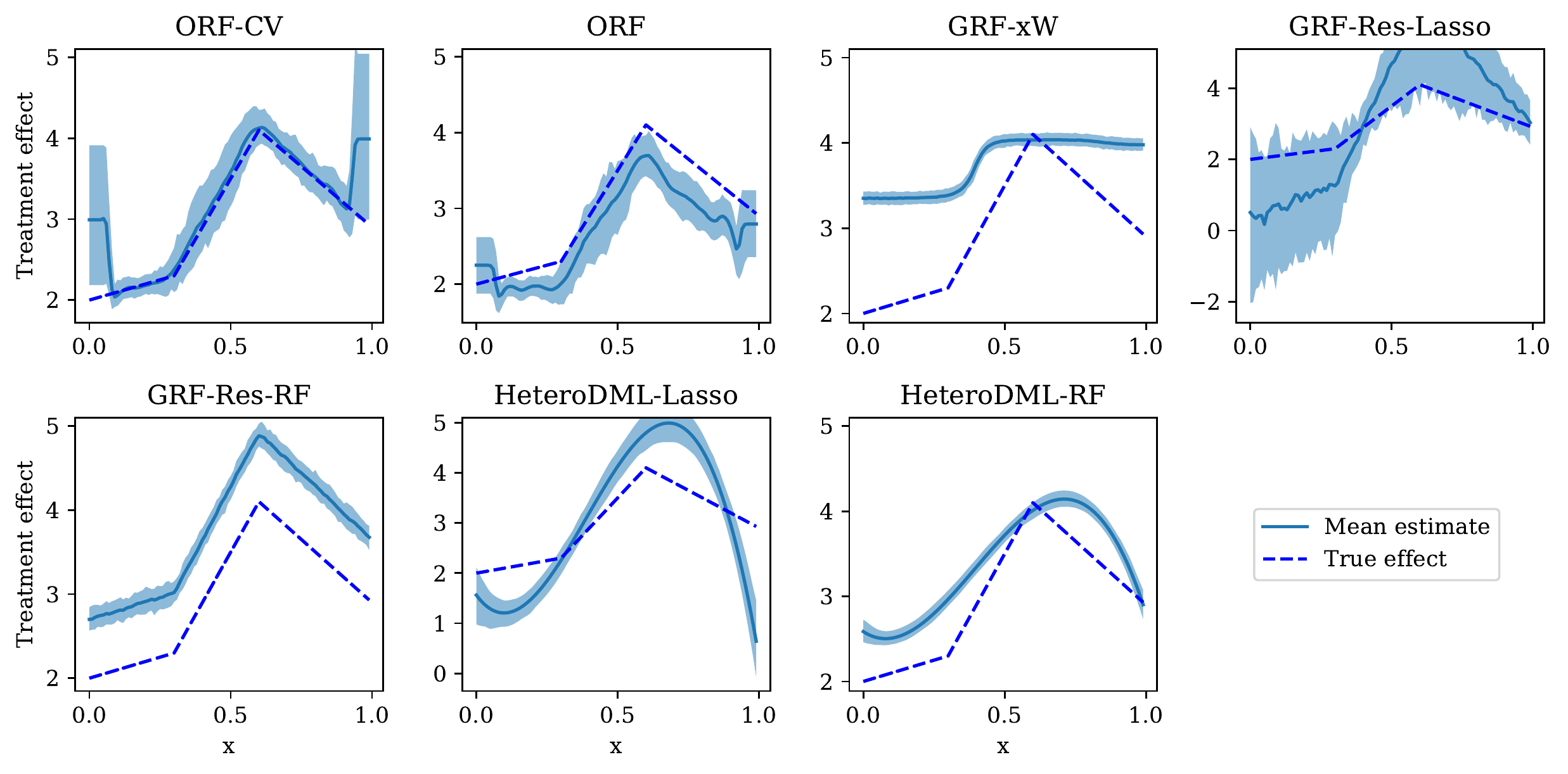}
	\end{minipage}\hfill
	\begin{minipage}[c]{0.23\textwidth}
		\caption{
			Treatment effect estimations for 100 Monte Carlo experiments with parameters $n=5000$, $p=500$, $d=1$, $\mathbf{k=200}$, and a piecewise linear treatment response. The shaded regions depict the mean and the $5\%$-$95\%$ interval of the 100 experiments.
		}
		\label{fig:large2} 
	\end{minipage}
\end{figure}

\newpage
\subsection{Experimental results for two-dimensional heterogeneity}

We introduce experimental results for a two-dimensional $x$ and corresponding $\theta_0$ given by:
\begin{align*}
\theta_0(x_1,x_2) &= \theta_{\text{piecewise linear}}(x_1)\mathbb{I}_{x_2=0} + 
\theta_{\text{piecewise constant}}(x_1)\mathbb{I}_{x_2=1}
\end{align*}
where $x_1\sim U[0, 1]$ and $x_2\sim Bern(0.5)$. In Figures \ref{fig:2D_1}-\ref{fig:2D_2}, we examine the overall behavior of the ORF-CV and ORF estimators, as well as the behavior across the slices $x_2=0$ and $x_2=1$. We compare the performance of the ORF-CV and ORF estimators with alternative methods for $n=5000$ and $k\in\{1,5,10,15,20,25,30\}$. We conclude that the ORF-CV algorithm yields a better performance for all support sizes and evaluation metrics. 

% summary all
\begin{figure}[H]
	\centering
	\includegraphics[scale=.50]{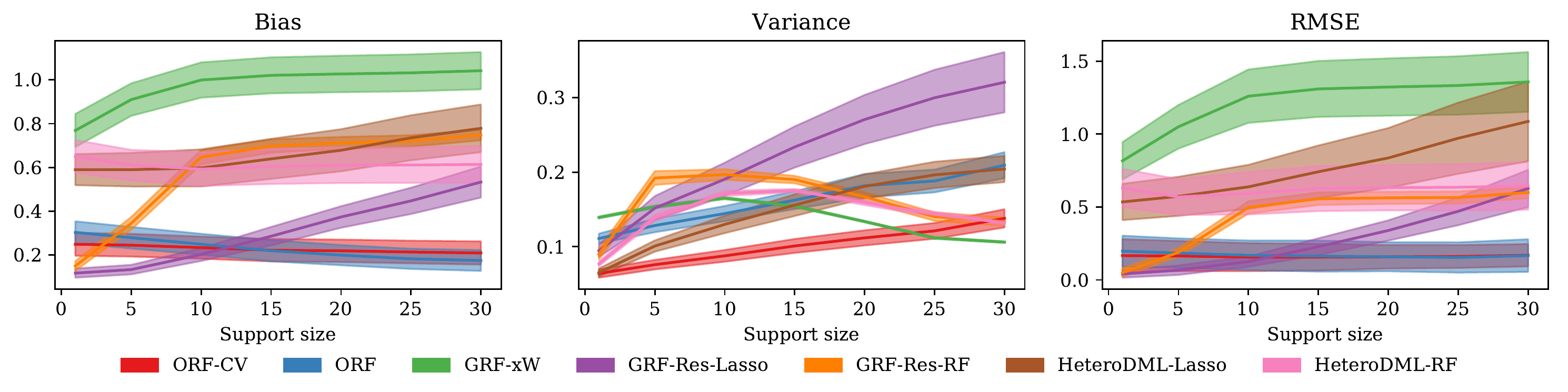}
	\caption{Overall bias, variance and RMSE as a function of support size for $n=5000$, $p=500$, $d=2$. The solid lines represent the mean of the metrics across test points, averaged over the Monte Carlo experiments, and the filled regions depict the standard deviation, scaled down by a factor of 3 for clarity.}
	\label{fig:2D_1}
\end{figure}
% summary x2=0, x2=1
\begin{figure}[H]
	\begin{minipage}[c]{0.75\textwidth}
		\includegraphics[width=\textwidth]{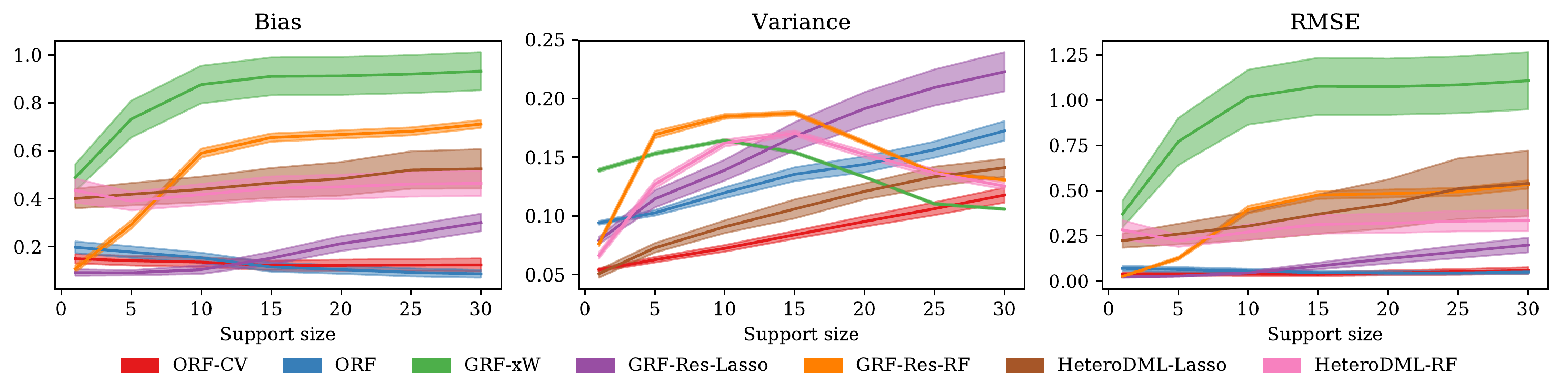}
		\includegraphics[width=\textwidth]{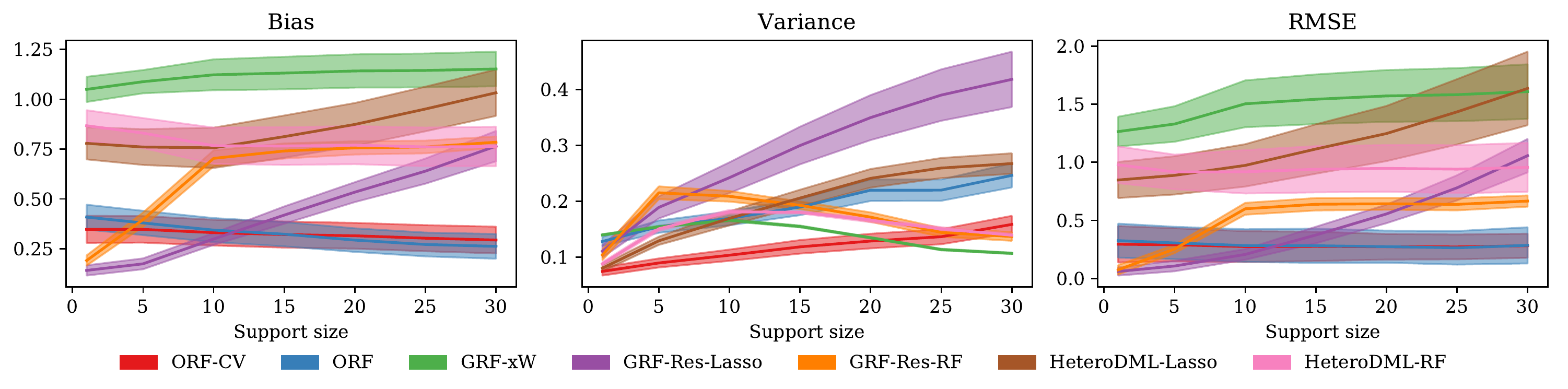}
	\end{minipage}\hfill
	\begin{minipage}[c]{0.23\textwidth}
		\caption{
			Bias, variance and RMSE as a function of support for $n=5000$, $p=500$, $d=2$ and slices $\mathbf{x_2=0}$ and $\mathbf{x_2=1}$, respectively. The solid lines represent the mean of the metrics across test points, averaged over the Monte Carlo experiments, and the filled regions depict the standard deviation, scaled down by a factor of 3 for clarity.
		} 
	\end{minipage}
\end{figure}

% support 1
\begin{figure}[H]
	\begin{minipage}[c]{0.63\textwidth}
		\includegraphics[width=\textwidth]{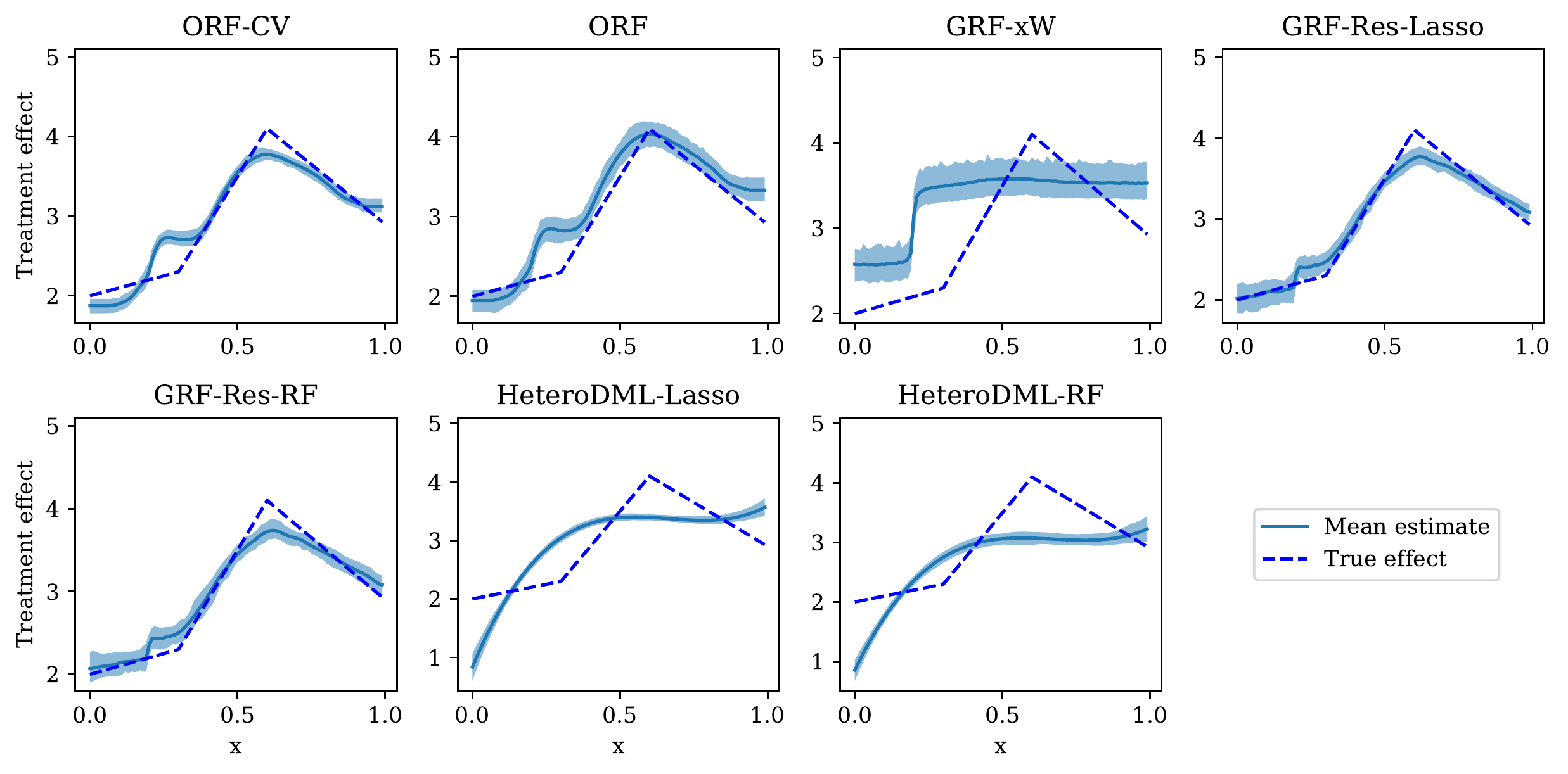}
		\includegraphics[width=\textwidth]{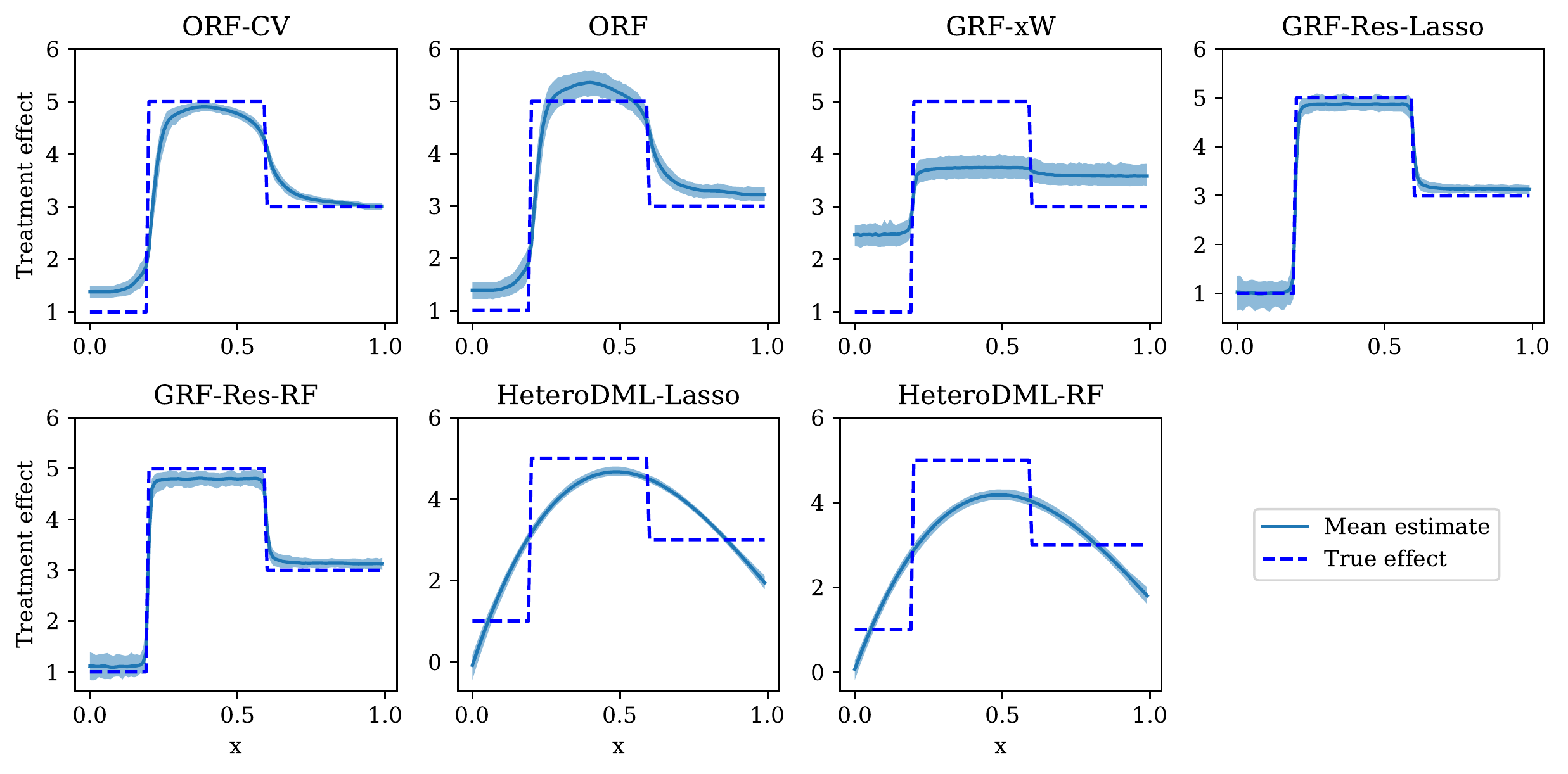}
	\end{minipage}\hfill
	\begin{minipage}[c]{0.34\textwidth}
		\caption{
			Treatment effect estimations for 100 Monte Carlo experiments with parameters $n=5000$, $p=500$, $d=2$, $\mathbf{k=1}$, and slices $\mathbf{x_2=0}$ and $\mathbf{x_2=1}$, respectively. The shaded regions depict the mean and the $5\%$-$95\%$ interval of the 100 experiments.
		} 
	\end{minipage}
\end{figure}
% support 5
\begin{figure}[H]
	\begin{minipage}[c]{0.63\textwidth}
		\includegraphics[width=\textwidth]{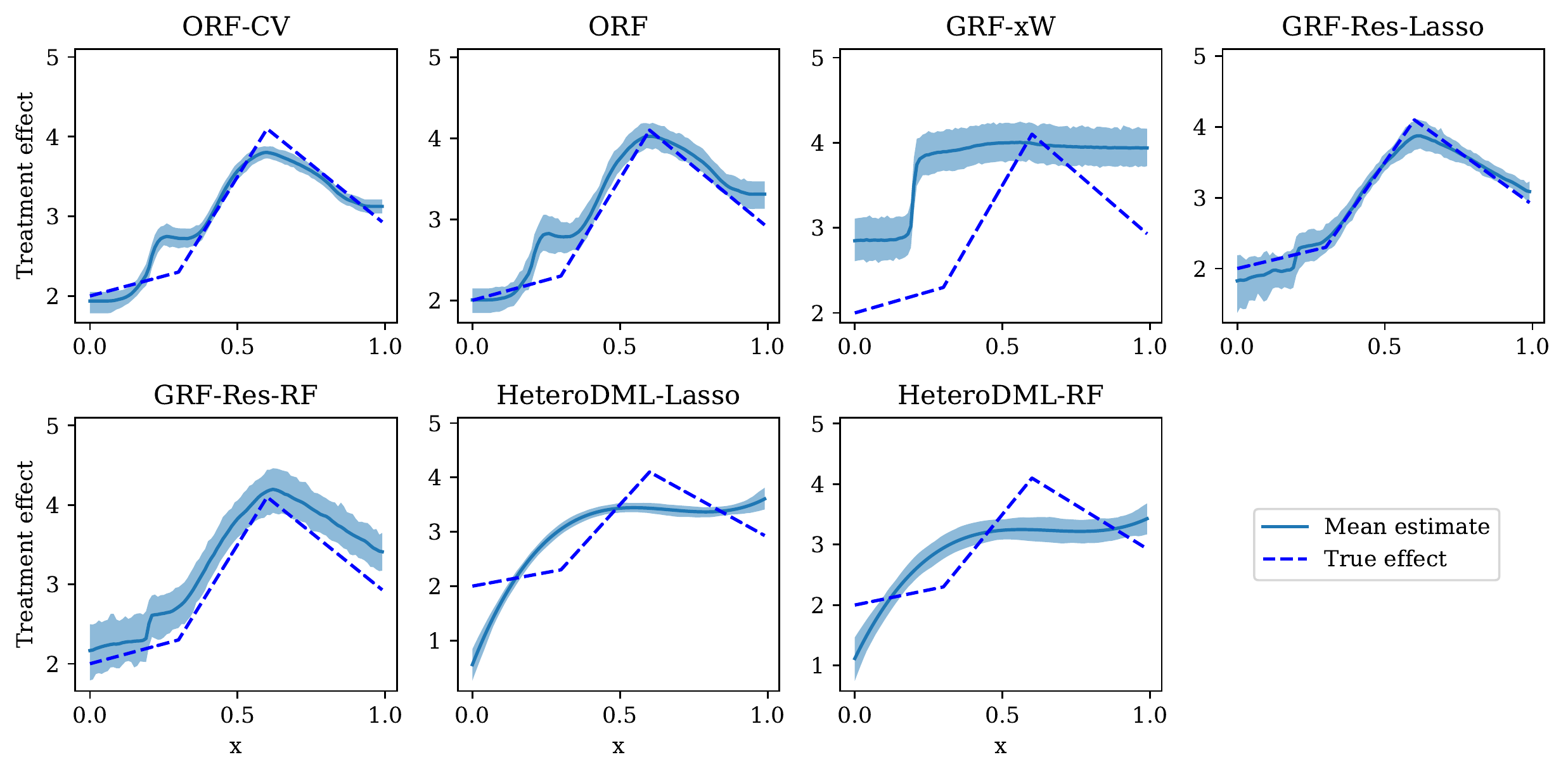}
		\includegraphics[width=\textwidth]{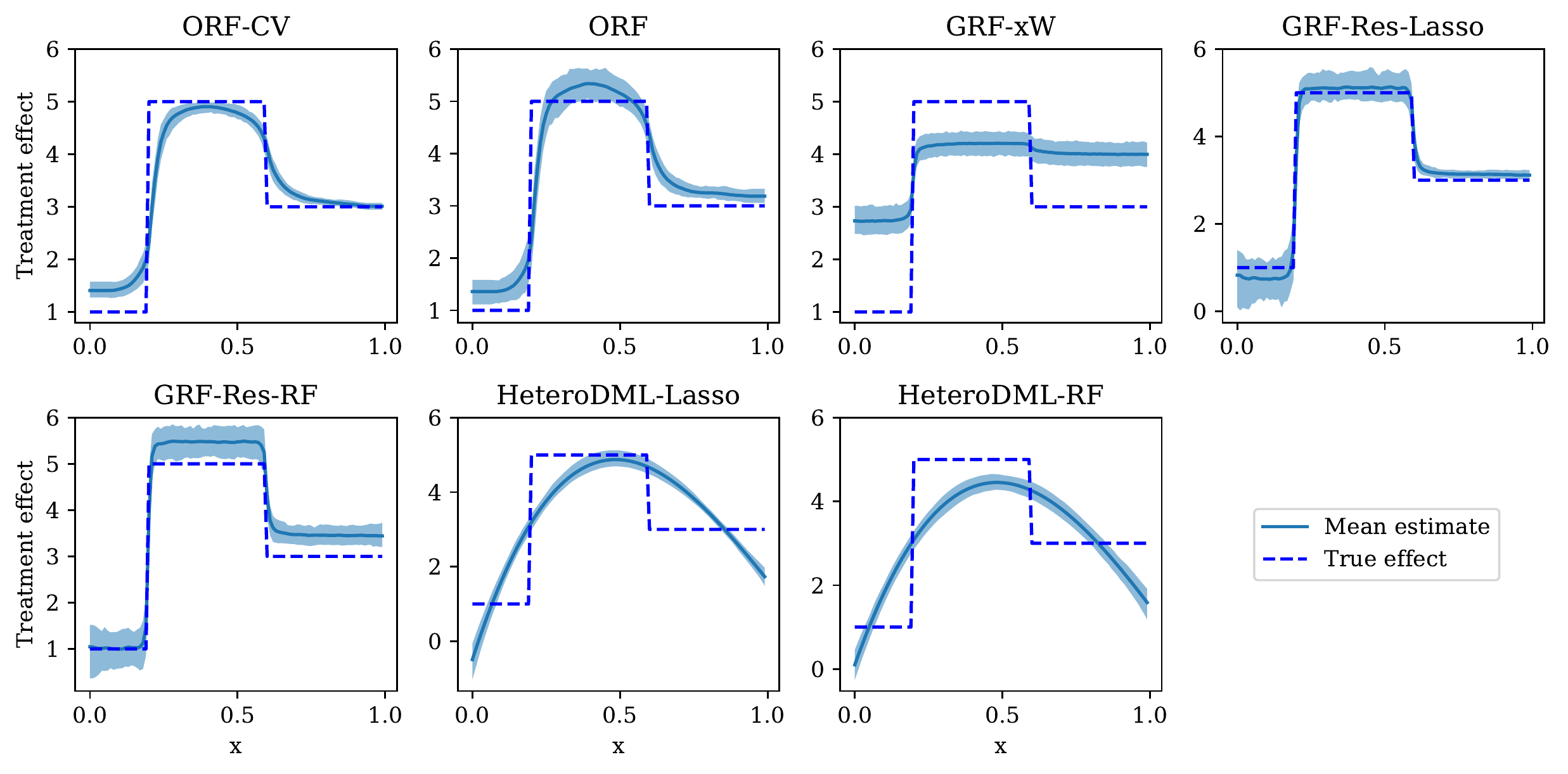}
	\end{minipage}\hfill
	\begin{minipage}[c]{0.34\textwidth}
		\caption{
			Treatment effect estimations for 100 Monte Carlo experiments with parameters $n=5000$, $p=500$, $d=2$, $\mathbf{k=5}$, and slices $\mathbf{x_2=0}$ and $\mathbf{x_2=1}$, respectively. The shaded regions depict the mean and the $5\%$-$95\%$ interval of the 100 experiments.
		} 
	\end{minipage}
\end{figure}

% support 10
\begin{figure}[H]
	\begin{minipage}[c]{0.63\textwidth}
		\includegraphics[width=\textwidth]{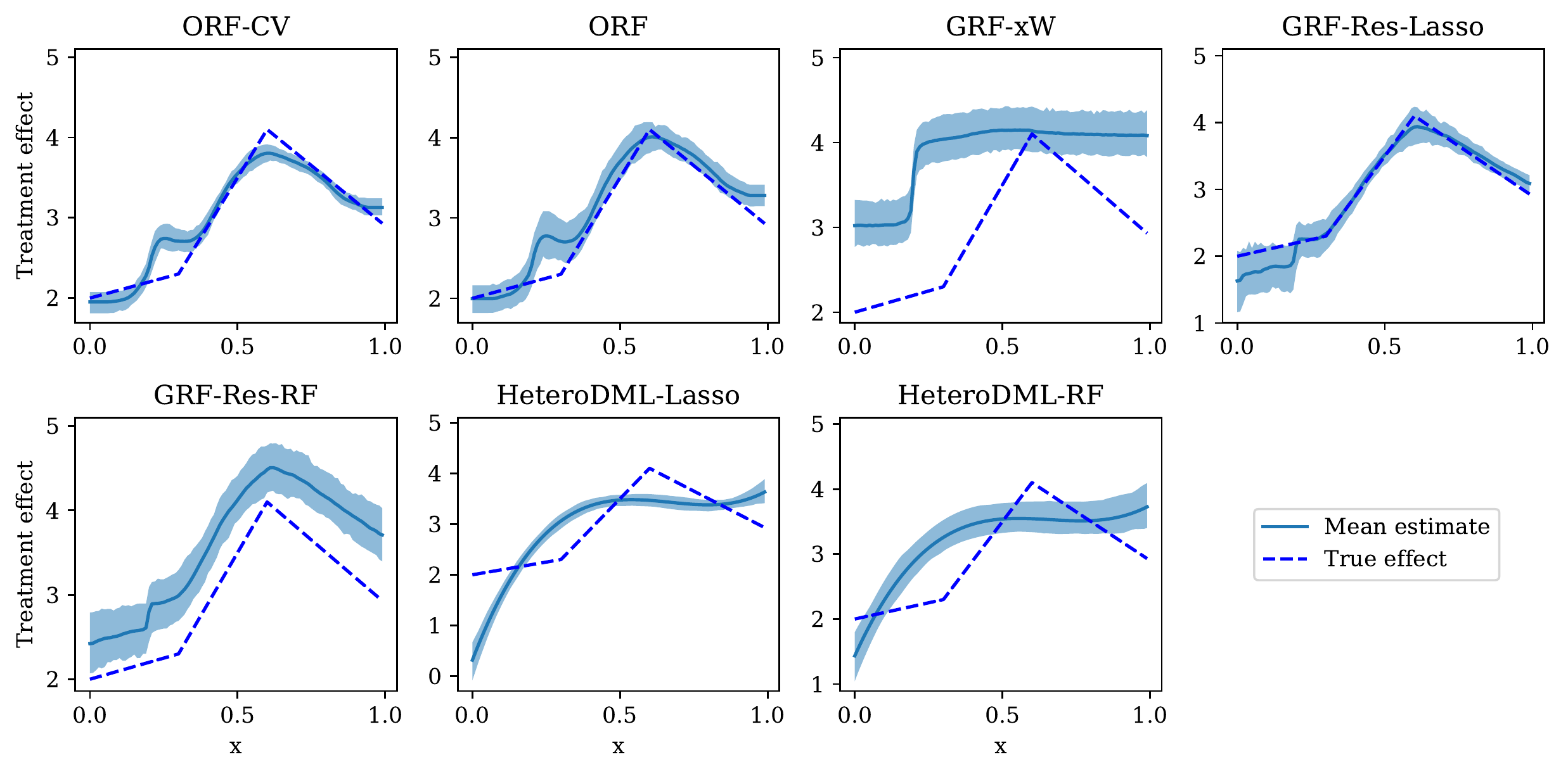}
		\includegraphics[width=\textwidth]{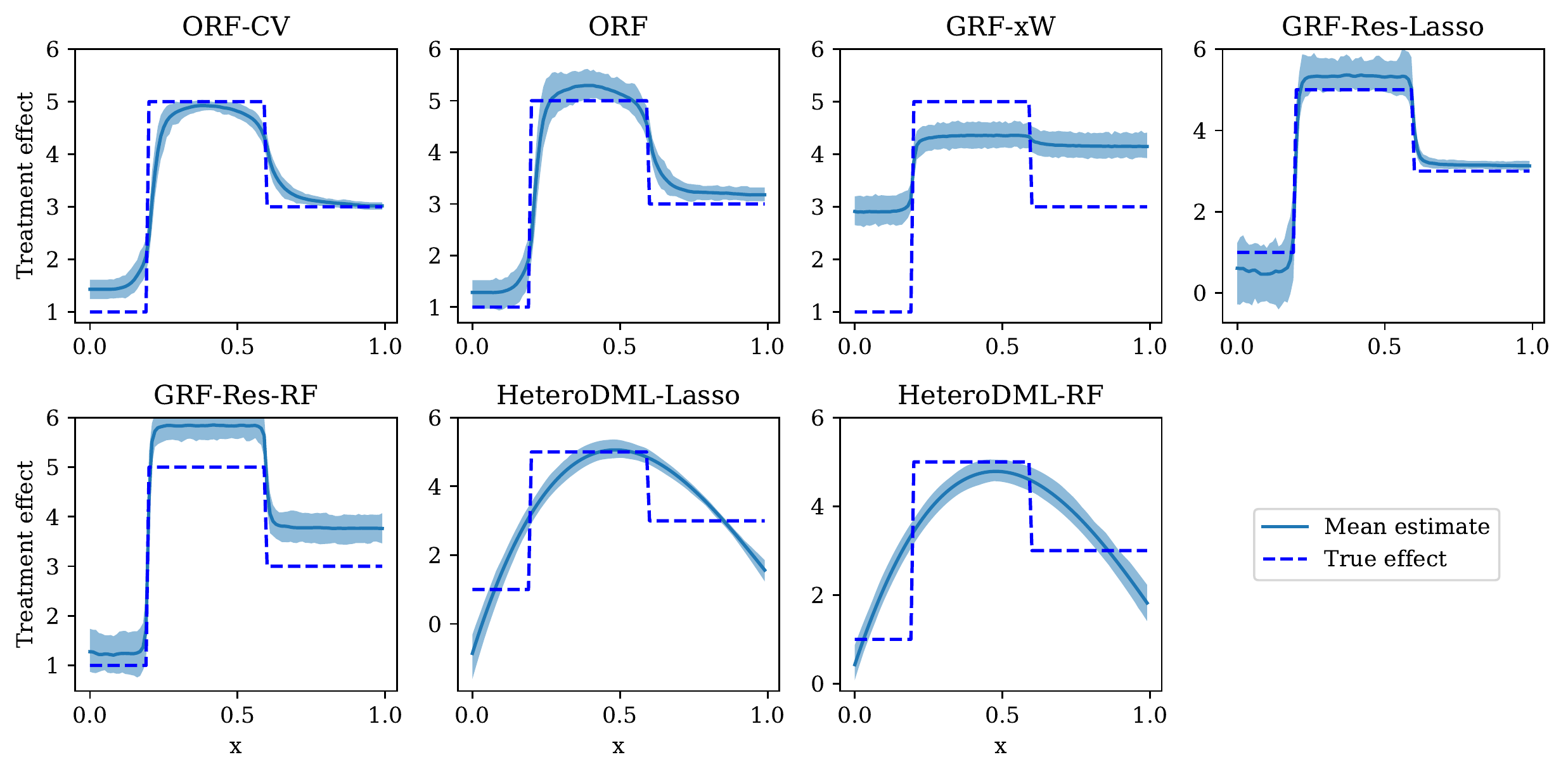}
	\end{minipage}\hfill
	\begin{minipage}[c]{0.34\textwidth}
		\caption{
			Treatment effect estimations for 100 Monte Carlo experiments with parameters $n=5000$, $p=500$, $d=2$, $\mathbf{k=10}$, and slices $\mathbf{x_2=0}$ and $\mathbf{x_2=1}$, respectively. The shaded regions depict the mean and the $5\%$-$95\%$ interval of the 100 experiments.
		} 
	\end{minipage}
\end{figure}

% support 15
\begin{figure}[H]
	\begin{minipage}[c]{0.63\textwidth}
		\includegraphics[width=\textwidth]{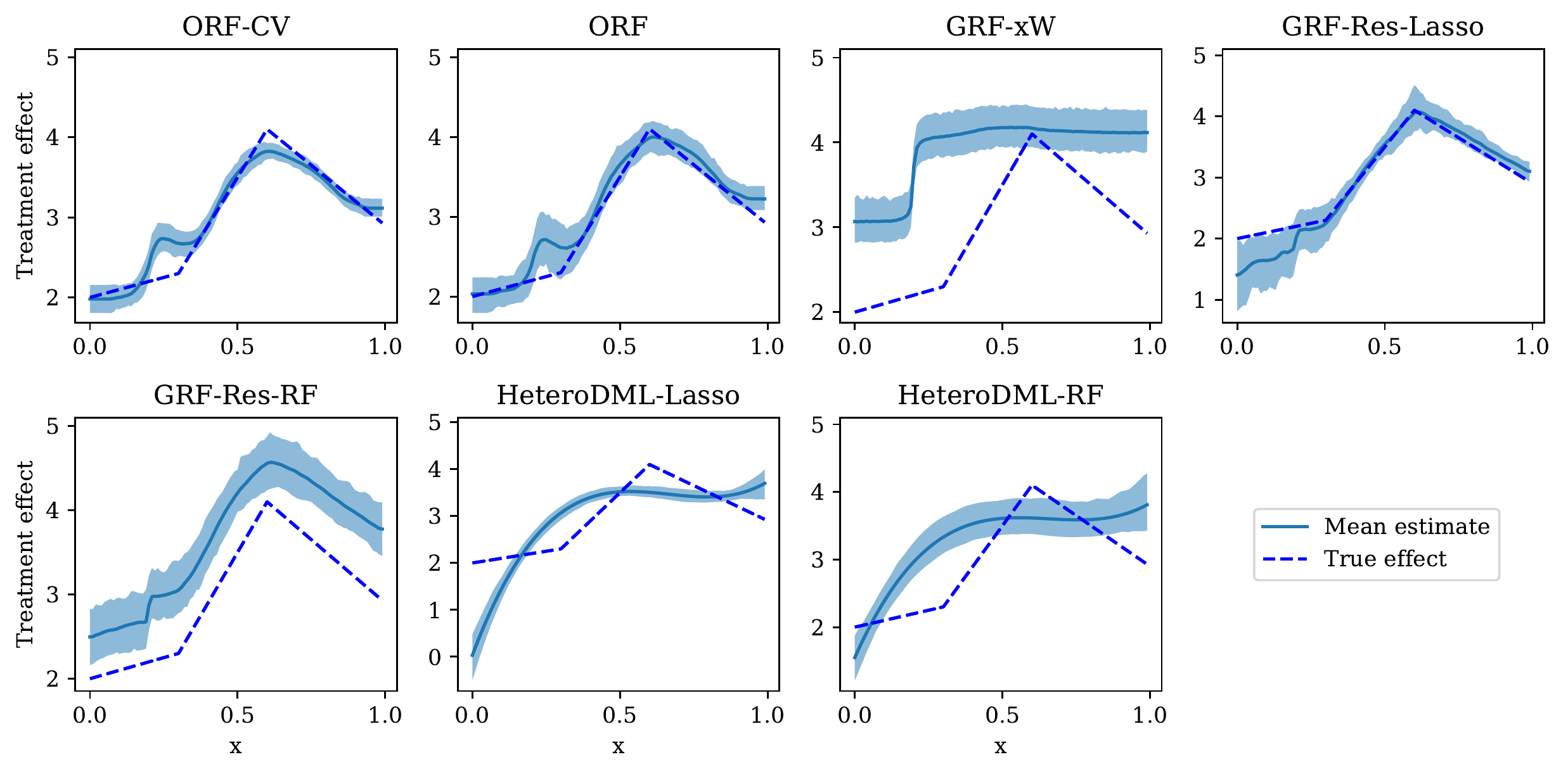}
		\includegraphics[width=\textwidth]{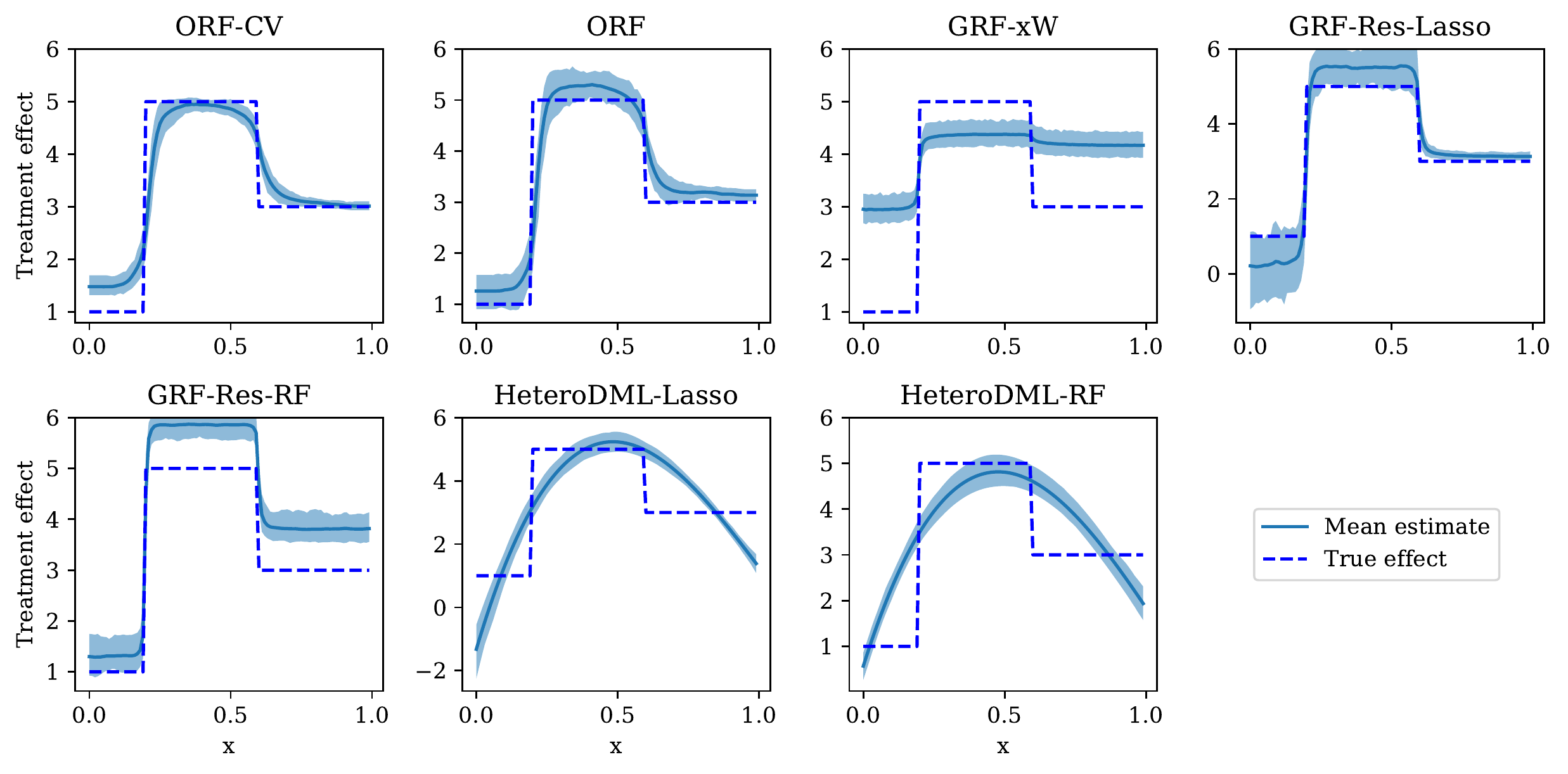}
	\end{minipage}\hfill
	\begin{minipage}[c]{0.34\textwidth}
		\caption{
			Treatment effect estimations for 100 Monte Carlo experiments with parameters $n=5000$, $p=500$, $d=2$, $\mathbf{k=15}$, and slices $\mathbf{x_2=0}$ and $\mathbf{x_2=1}$, respectively. The shaded regions depict the mean and the $5\%$-$95\%$ interval of the 100 experiments.
		} 
	\end{minipage}
\end{figure}

% support 20
\begin{figure}[H]
	\begin{minipage}[c]{0.63\textwidth}
		\includegraphics[width=\textwidth]{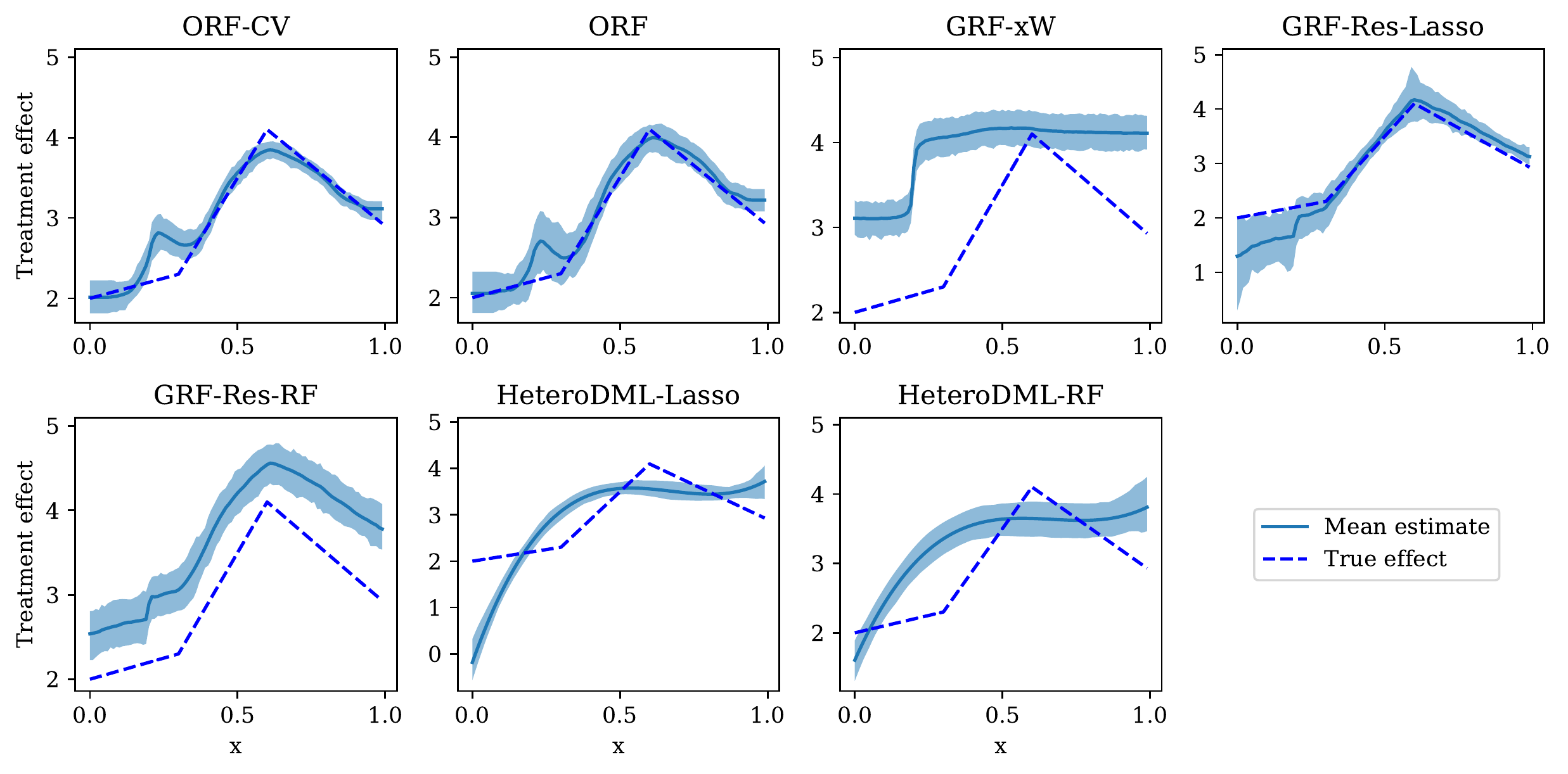}
		\includegraphics[width=\textwidth]{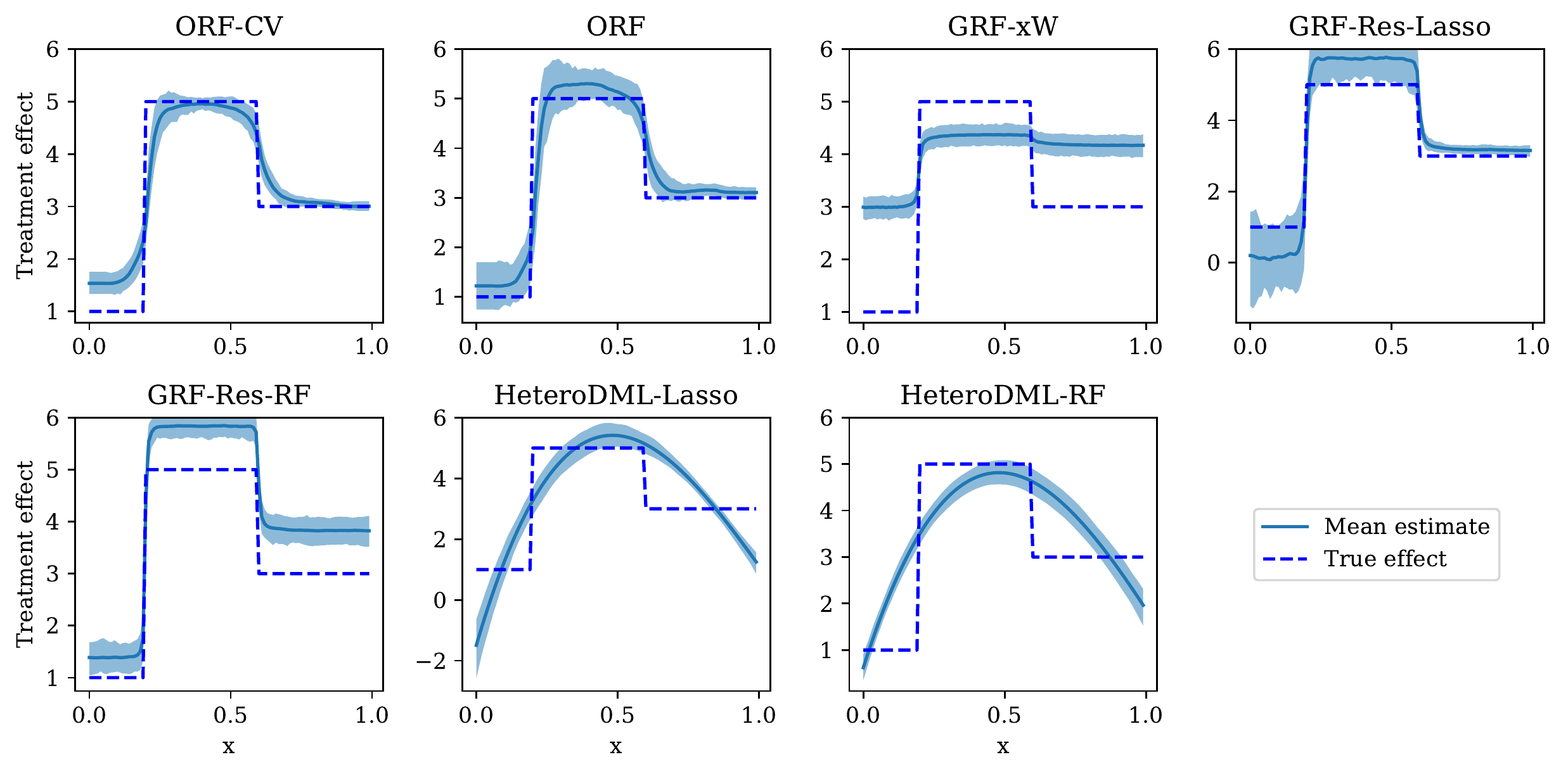}
	\end{minipage}\hfill
	\begin{minipage}[c]{0.34\textwidth}
		\caption{
			Treatment effect estimations for 100 Monte Carlo experiments with parameters $n=5000$, $p=500$, $d=2$, $\mathbf{k=20}$, and slices $\mathbf{x_2=0}$ and $\mathbf{x_2=1}$, respectively. The shaded regions depict the mean and the $5\%$-$95\%$ interval of the 100 experiments.
		} 
	\end{minipage}
\end{figure}
% support 25
\begin{figure}[H]
	\begin{minipage}[c]{0.63\textwidth}
		\includegraphics[width=\textwidth]{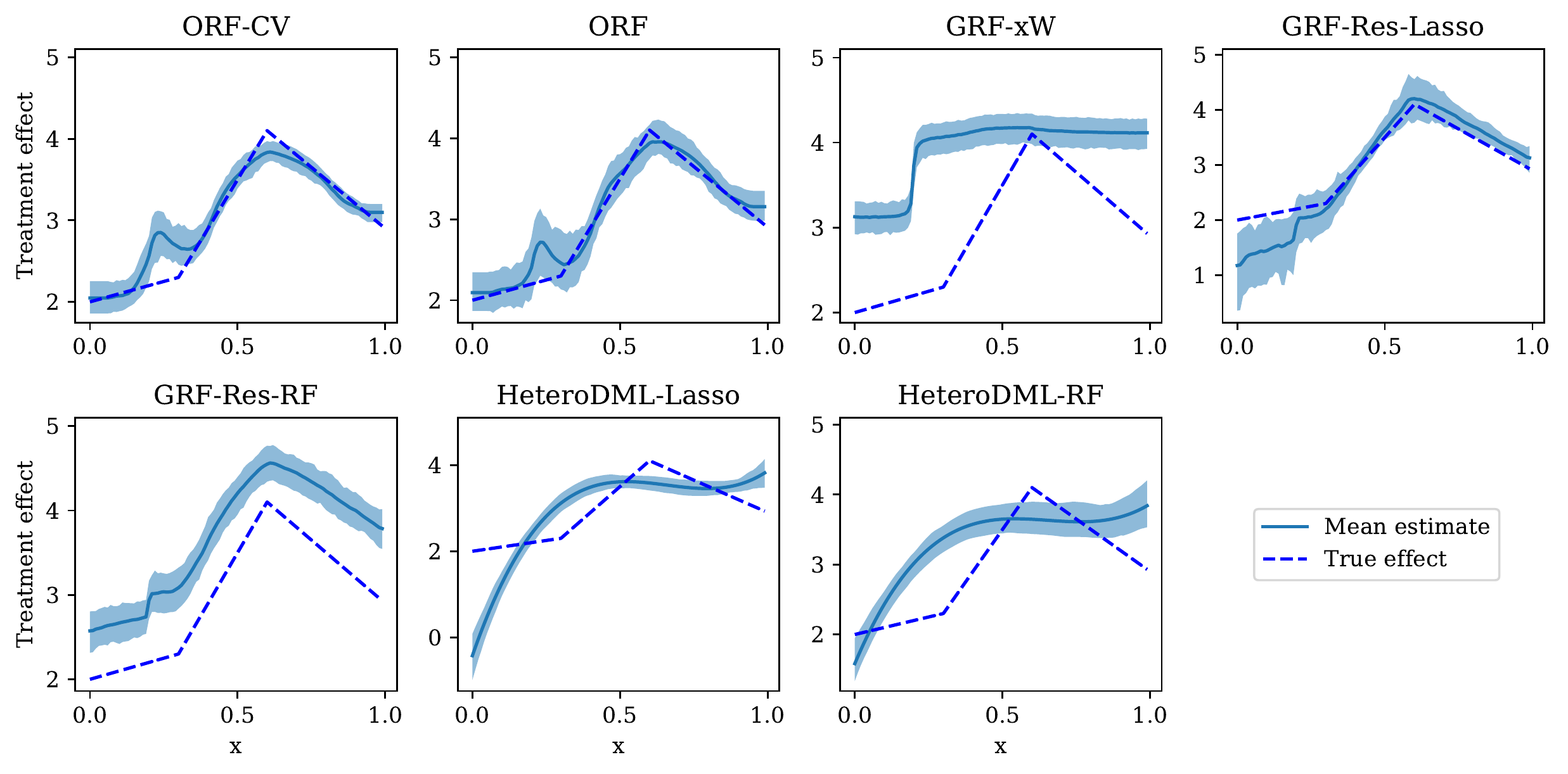}
		\includegraphics[width=\textwidth]{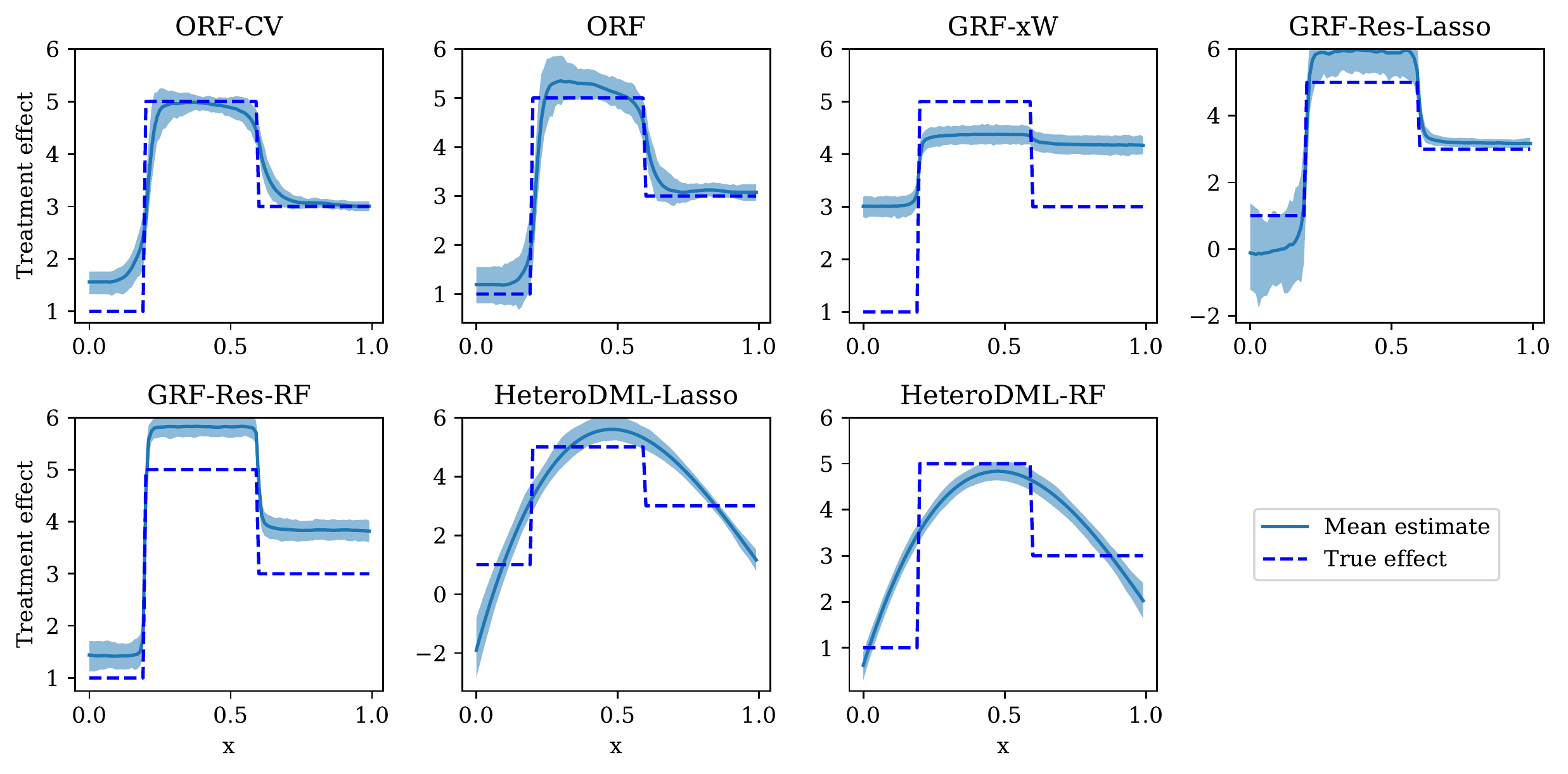}
	\end{minipage}\hfill
	\begin{minipage}[c]{0.34\textwidth}
		\caption{
			Treatment effect estimations for 100 Monte Carlo experiments with parameters $n=5000$, $p=500$, $d=2$, $\mathbf{k=25}$, and slices $\mathbf{x_2=0}$ and $\mathbf{x_2=1}$, respectively. The shaded regions depict the mean and the $5\%$-$95\%$ interval of the 100 experiments.
		} 
	\end{minipage}
\end{figure}

% support 30
\begin{figure}[H]
	\begin{minipage}[c]{0.63\textwidth}
		\includegraphics[width=\textwidth]{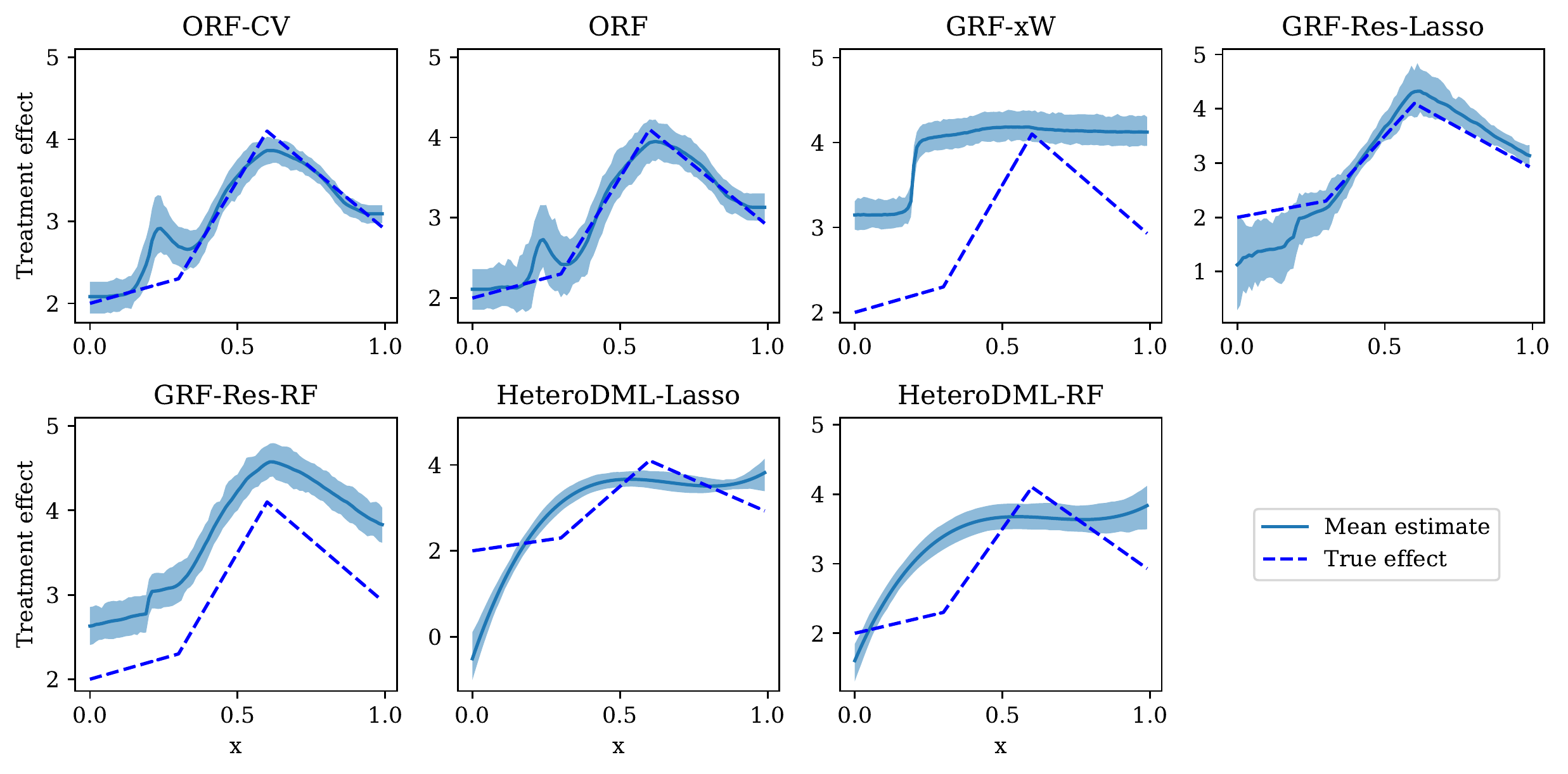}
		\includegraphics[width=\textwidth]{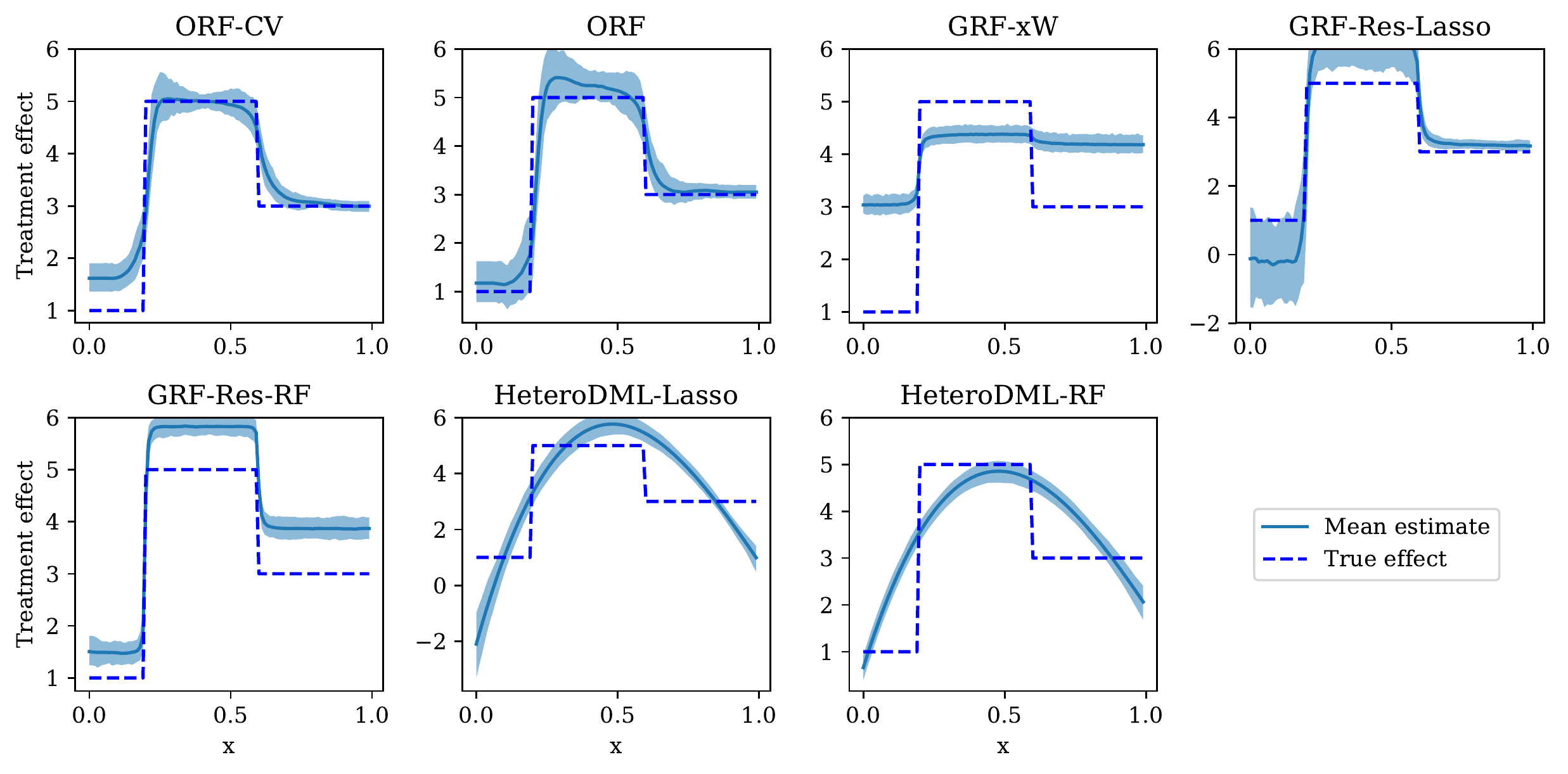}
	\end{minipage}\hfill
	\begin{minipage}[c]{0.34\textwidth}
		\caption{
			Treatment effect estimations for 100 Monte Carlo experiments with parameters $n=5000$, $p=500$, $d=2$, $\mathbf{k=30}$, and slices $\mathbf{x_2=0}$ and $\mathbf{x_2=1}$, respectively. The shaded regions depict the mean and the $5\%$-$95\%$ interval of the 100 experiments.
		}
		\label{fig:2D_2} 
	\end{minipage}
\end{figure}

\end{document}